\documentclass[12pt]{article} 
\usepackage[utf8]{inputenc}
\usepackage[margin=1in]{geometry}
\usepackage{mathtools}
\mathtoolsset{showonlyrefs} 
\usepackage{amsmath}
\usepackage{amsfonts}
\usepackage{amssymb}
\usepackage{textcomp}
\usepackage{pgfplots}
\usepackage{graphicx}
\usepackage{wrapfig}
\usepackage{amsthm}
\usepackage{tikz}
\usepackage{tikz-cd}
\usepackage{enumerate}
\usepackage{bbm}

\usepackage{eucal}
\usepackage{mathrsfs}
\usepackage{bbm}
\pgfplotsset{width=10cm,compat=1.9}
\usepackage{natbib}
\usepackage{booktabs}
\usepackage{tablefootnote}
\usepackage{makecell}
\usepackage{hyperref}
\hypersetup{
    colorlinks=true,
    linkcolor=blue,
    citecolor=blue
}

\newtheorem{innercustomthm}{Theorem}
\newenvironment{customthm}[1]
  {\renewcommand\theinnercustomthm{#1}\innercustomthm}
  {\endinnercustomthm}

  \newtheorem{innercustomcor}{Corollary}
\newenvironment{customcor}[1]
{\renewcommand\theinnercustomcor{#1}\innercustomcor}
{\endinnercustomcor}

\newtheorem{theorem}{Theorem}
\newtheorem{lemma}[theorem]{Lemma}
\newtheorem{proposition}[theorem]{Proposition}

\newtheorem{corollary}[theorem]{Corollary}
\newtheorem{definition}[theorem]{Definition}

\DeclareMathOperator*{\argmax}{arg\,max}

\newcommand{\nc}{\newcommand}
\nc{\ra}{\rightarrow}
\nc{\llin}{\ell^{{\rm lin}}}
\nc{\MN}{\mathcal{N}}
\nc{\MH}{\mathcal{H}}
\nc{\lng}{\langle}
\nc{\rng}{\rangle}
\nc{\MV}{\mathcal{V}}
\DeclareMathOperator{\thr}{thr}
\nc{\erma}{\zeta}
\nc{\ep}{\epsilon}
\nc{\MX}{\mathcal{X}}
\nc{\BZ}{\mathbb{Z}}
\nc{\st}{\star}
\nc{\BR}{\mathbb{R}}
\nc{\MC}{\mathcal{C}}
\nc{\MZ}{\mathcal{Z}}
\nc{\close}[3]{\MC_{{#1},{#2}}({#3})}
\nc{\bv}{\mathbf{v}}
\nc{\bz}{\mathbf{z}}
\nc{\bw}{\mathbf{w}}
\nc{\bb}{\mathbf{b}}
\nc{\MA}{\mathcal{A}}
\nc{\ME}{\mathcal{E}}
\nc{\ol}{\overline}
\nc{\YY}{[-1,1]}
\nc{\CLS}{\textbf{Cls}}
\nc{\REG}{\textbf{Reg}}
\nc{\condt}{\ | \ }
\nc{\MI}{\mathcal{I}}
\nc{\MB}{\mathcal{B}}
\renewcommand{\t}{\top}
\DeclareMathOperator{\E}{\mathbb{E}}
\newcommand{\One}{\mathbbm{1}}
\newcommand{\bx}{{\mathbf{x}}}
\nc{\x}{{\mathbf{x}}}
\newcommand{\by}{{\mathbf{y}}}

\newcommand{\cX}{\mathcal{X}}

\newcommand{\rr}{\mathbb{R}}
\newcommand{\pp}{\mathbb{P}}
\newcommand{\ee}{\mathbb{E}}

\newcommand{\R}{\mathcal{R}}
\newcommand{\F}{\mathcal{F}}
\newcommand{\MF}{\mathcal{F}}

\newcommand{\BN}{\mathbb{N}}

\newcommand{\fat}{\mathsf{fat}}

\newcommand{\z}{{\mathbf{z}}}

\newcommand{\norm}[1]{\left|\left| #1 \right|\right|}
\newcommand{\abs}[1]{\left| #1 \right|}
\DeclareMathOperator{\supp}{supp}

\newcommand{\V}{\mathcal{V}}
\newcommand{\prop}{\textrm{prop}}
\newcommand{\impr}{\textrm{improp}}
\newcommand{\p}{\mathfrak{P}}
\newcommand{\ps}[3]{\p_{#1}({#2},{#3})}
\newcommand{\tps}[3]{\widetilde{\p}_{#1}({#2},{#3})}
\newcommand{\ac}[1]{\textcolor{red}{[Adam: #1 ]}}

\newcommand{\noah}[1]{\textcolor{purple}{[Noah: #1]}}

\newcommand{\rel}{\mathbf{Rel}}
\DeclareMathOperator*{\argmin}{argmin}
\newcommand{\yhat}{\widehat{y}}
\newcommand{\fhat}{\widehat{f}}
\DeclareMathOperator{\reg}{Reg}
\newcommand{\vc}{\mathsf{vc}}

\renewcommand{\epsilon}{\varepsilon}
\newcommand{\sD}{\mathscr{D}}
\renewcommand{\fat}{\mathsf{fat}}
\DeclareMathOperator{\esssup}{ess\,sup}
\DeclareMathOperator{\sign}{sign}
\DeclareMathOperator{\unif}{Unif}

\nc{\rc}[2]{\mathcal{R}_{#1}({#2})}
\nc{\gc}[2]{\mathcal{G}_{#1}({#2})}

\title{Smoothed Online Learning is as Easy as Statistical Learning}
\author{Adam Block \\ MIT \and Yuval Dagan \\ MIT \and Noah Golowich \\MIT \and Alexander Rakhlin \\ MIT}
\usepackage{times}

\begin{document}

\maketitle

\begin{abstract}%
  Much of modern learning theory has been split between two regimes: the classical \emph{offline} setting, where data arrive independently, and the \emph{online} setting, where data arrive adversarially. While the former model is often both computationally and statistically tractable, the latter requires no distributional assumptions. In an attempt to achieve the best of both worlds, previous work proposed the smooth online setting where each sample is drawn from an adversarially chosen distribution, which is \emph{smooth}, i.e., it has a bounded density with respect to a fixed dominating measure. Existing results for the smooth setting were known only for binary-valued function classes and were computationally expensive in general; in this paper, we fill these lacunae.  In particular, we provide tight bounds on the minimax regret of learning a nonparametric function class, with nearly optimal dependence on both the horizon and smoothness parameters.  Furthermore, we provide the first oracle-efficient, no-regret algorithms in this setting.
  In particular, we propose an oracle-efficient improper algorithm whose regret achieves optimal dependence on the horizon and a proper algorithm requiring only a \emph{single} oracle call per round whose regret has the optimal horizon dependence in the classification setting and is sublinear in general. Both algorithms have exponentially worse dependence on the smoothness parameter of the adversary than the minimax rate. We then prove a lower bound on the oracle complexity of any proper learning algorithm, which matches the oracle-efficient upper bounds up to a polynomial factor, thus demonstrating the existence of a statistical-computational gap in smooth online learning.  Finally, we apply our results to the contextual bandit setting to show that if a function class is learnable in the classical setting, then there is an oracle-efficient, no-regret algorithm for contextual bandits in the case that contexts arrive in a smooth manner.
\end{abstract}

\section{Introduction}

Modern learning theory has primarily focused on two regimes: batch and sequential settings.  In the former, data are independent and learning is easy while in the latter, Nature has the power to adversarially choose data to make learning as difficult as possible.  Much of the empirical success in machine learning has derived from assuming that independence is satisfied and applying the Empirical Risk Minimization (ERM) principle in the batch (offline) setting, for which algorithms such as gradient descent have been highly successful even with complex function classes.  However, in many settings independence is likely to fail, and the sequential regime is attractive due to the minimal assumptions made on the data generating process. Unfortunately, it suffers from poor algorithmic efficiency and an inability to learn some of the most basic function classes.  

A typical example of the gap in difficulty of learning in these two settings is furnished by threshold functions on the unit interval.  Classical theory tells us that in the batch setting, due to the combinatorial simplicity of the function class, a simple ERM approach efficiently and optimally learns the class of thresholds; in contradistinction, an adversarial data generation process precludes sequential (online) learning entirely \citep{littlestone1988learning}.  One way to escape the difficulty of adversarial learning is to apply the technique of \emph{smoothed analysis} introduced in a now-famous paper by \cite{spielman2004nearly}, which focused on solving linear programs with the simplex algorithm.  In this regime, 
one analyzes worst case inputs that are perturbed by a small amount of stochastic noise.  Smoothed analysis in the setting of online learning was first introduced in \cite{rakhlin2011online}, where the authors showed non-constructively that thresholds again become learnable in this setting.  More recently, a series of papers \citep{haghtalab2020smoothed,haghtalab2021smoothed} has demonstrated that the stochastic perturbation has beneficial effects in far greater generality than the class of thresholds; in fact, any classification task that is possible in the batch setting is also statistically tractable in the smoothed online setting.  

In the more modern formulation of the smoothed paradigm studied in \cite{haghtalab2020smoothed,haghtalab2021smoothed}, instead of choosing an input that is then perturbed, the adversary chooses a distribution that is restricted to be sufficiently anti-concentrated so as to not put too much mass on the set of ``hard'' instances.  This anti-concentration (referred to as \emph{$\sigma$-smoothness}; Definition \ref{def:smooth}) is quantified by a parameter $\sigma \leq 1$ that governs how far from independent the adversary can be.  When $\sigma = 1$, we are in the batch setting where the data arrive i.i.d.; as $\sigma$ tends to $0$, the adversary is given more and more power to choose bad instances, with the limit of $\sigma = 0$ being entirely adversarial. 

While we provide a rigorous definition of the problem setting below, we outline the broad strokes here.  We consider learning over the course of $T$ rounds where, at each round, Nature reveals a context $x_t$ sampled in a $\sigma$-smooth manner, the learner reveals a prediction $\yhat_t$, and then Nature reveals $y_t$.  Given a loss function, the objective of the learner is to minimize regret to the best predictor in some function class $\F$.  When $\F$ is binary-valued and has finite VC dimension $d$ (i.e. is learnable in the batch setting), \cite{haghtalab2021smoothed} proved that $O\left(\sqrt{d T \log\left(T / \sigma\right)}\right)$ regret is achievable, albeit with an inefficient algorithm.  Two natural questions arise: can we extend these results to nonparametric, real-valued classes? and, more importantly, are there efficient algorithms that can achieve comparable regret?  In this paper, we answer both questions.   With regard to the first question, the natural extension of the covering-based argument in \cite{haghtalab2021smoothed} would yield suboptimal dependence on $\sigma$ in the nonparametric regime; instead, we obtain a nonconstructive proof through careful application of combinatorial inequalities and an adaptation of the coupling lemma of \cite{haghtalab2020smoothed}.



For the question of practical algorithms, we need to more carefully consider what we mean by efficiency.  Certainly the algorithm of \cite{haghtalab2021smoothed} is not efficient as it requires constructing an $\epsilon$-net of $\F$ as a first step, which is exponential in the VC dimension of $\F$.  A natural choice is to look to the batch setting and try to leverage the success of ERM-based approaches to achieve reasonable runtime, as is done in \cite{kalai2005efficient,hazan2016computational}; this approach is supported further by the empirical success of various heuristics for ERM \citep{goodfellow2016deep}.  As such, we suppose the learner has access to an \emph{ERM oracle} (Definition \ref{def:oracle}) that can efficiently optimize some loss over our function class $\F$ given as input a dataset; we analyze the time complexity of our algorithms in terms of the number of calls to this oracle.

In our algorithmic results, we distinguish between \emph{proper} and \emph{improper} learners. Proper learners are required, before seeing $x_t$, to output a hypothesis $\fhat_t \in\MF$ that is used to produce the prediction $\yhat_t := \fhat_t(x_t)$, whereas improper learners can make any prediction $\yhat_t$ based on knowledge of $x_t$. There are many settings in which proper (as opposed to improper) online learning may be desirable for downstream applications, such as learning in games \citep{daskalakis2021fast} and reinforcement learning with function approximation \citep{foster2021statistical}. For proper learners, our oracle time complexity results are optimal up to a polynomial factor in all parameters; the question of optimality for improper learning remains an interesting open question.

We now briefly describe our key contributions (which are summarized in Table \ref{tab:results}):

\begin{itemize}
  \item In Section \ref{sec:minimaxvalue} we give tight upper bounds on the statistical rates of learning a real-valued function class in the smoothed online setting without regard to computational efficiency, while extending and providing a new proof to the case of binary classification treated in \cite[Theorems 3.1 \& 3.2]{haghtalab2021smoothed}.  Our bounds are tight both in their dependence on $T$ and $\sigma$, up to logarithmic factors, showing that the regret in the smooth setting is only a factor $\log(T/\sigma)$ away from that in the i.i.d.~setting (Theorem \ref{thm:chainingupperbound}).  In the process of doing this, we provide in Lemma \ref{lem:coupling} a more general and much simpler proof of the key technical step of coupling from \cite{haghtalab2021smoothed}.
  \item In Section \ref{sec:relax} we present an \emph{improper} algorithm based on the relaxation method of \cite{rakhlin2012relax} with tight dependence on the horizon $T$ but suboptimal dependence on $\sigma$: in particular, the regret in the smooth setting scales as $\sigma^{-1/2}$ times that in the i.i.d.~setting (Theorem \ref{thm:fastrelaxation}).  Our algorithm is efficient in the sense that it requires only $O\left(\sqrt{T} \log T\right)$ oracle calls per round in general and only $2$ oracle calls per round in the classification setting.  We then show in Proposition \ref{prop:nonparametriclowerbound} that the polynomial dependence on $\sigma$ is not an artifact of our analysis, but rather inherent to the method.
  \item In Section \ref{sec:ftpl} we present a \emph{proper} algorithm based on Follow the Perturbed Leader (FTPL) that exhibits optimal dependence on $T$ for classification and suboptimal dependence on $T$ in general (Theorem \ref{thm:generalftpl}). Further, the algorithm requires only $1$ oracle call per round.  We use a Gaussian white noise with intensity approximated by $\mu$ as our perturbation, which allows for optimization without enumeration of experts; to establish correctness, we overcome a major technical hurdle introduced by the complicated dependence structure of this perturbation.
  \item In Section \ref{sec:lowerbound} we show that the suboptimality of the FTPL learner is inherent for oracle-efficient algorithms: in particular, we provide a lower bound based on the method of \cite{hazan2016computational} that demonstrates that any \emph{proper} algorithm with access to an ERM oracle requires $\widetilde\Omega\left(\sigma^{- \frac 12}\right)$ time.  Combined with the upper bound of $\log 1/\sigma$ of Theorem \ref{thm:chainingupperbound}, this implies that there exists an exponential statistical-computational gap in smoothed online learning.
  \item Finally, in Appendix \ref{app:bandits}, we apply our results to the problem of Contextual Bandits.  We show in Theorem \ref{thm:cbregret} that whenever a function class $\F$ is learnable offline, we can get vanishing regret in the smooth contextual bandit setting with an oracle-efficient algorithm.
  \end{itemize}
  \begin{table}[ht]
      \begin{center}
    \begin{tabular}{c|cccc}
      \toprule   Reference  & Algorithm & Iterations & Oracle calls & Total time  \\\toprule

      Theorem \ref{thm:chainingupperbound} & \thead{Non-constructive,\\\emph{proper}} & \thead{\CLS\footnotemark: $\ep^{-2} d \log(1/\sigma)$ \\ \REG: $\ep^{-2} \log(1/\sigma)$} & NA & NA \\\hline
      
      Theorem \ref{thm:fastrelaxation} & \thead{Relaxation-based,\\\emph{improper}} & \thead{\CLS: $\ep^{-2} d \sigma^{-1}$ \\ \REG: $\ep^{-2} \sigma^{-1}$}  & \thead{\CLS: $\ep^{-2} d \sigma^{-1}$ \\ \REG: $\ep^{-3} \sigma^{-3/2}$} &  \thead{\CLS: $\ep^{-4} d^2 \sigma^{-3}$ \\ \REG: $\ep^{-4} \sigma^{-5/2}$}  \\\hline
      
      Theorem \ref{thm:generalftpl} & \thead{FTPL,\\\emph{proper}} & \thead{\CLS: $\ep^{-2} d \sigma^{-1}$ \\ \REG: $\ep^{-3} \sigma^{-1}$ } & \thead{\CLS: $\ep^{-2} d \sigma^{-1}$ \\ \REG: $\ep^{-3} \sigma^{-1}$} &  \thead{\CLS: $\ep^{-4} d^2 \sigma^{-5/2}$ \\ \REG:  $\ep^{-7} \sigma^{-3}$}  \\\midrule\midrule
      
    \thead{Theorems \ref{thm:oracle-lb} \& \ref{thm:oracle-lb-apx}, \\    Corollaries \ref{cor:stat-comp-gap} \& \ref{cor:sc-gap-apx}} &\thead{Computational \\ \textbf{lower bound} \\ for any \emph{proper} alg.} & NA & \thead{\CLS: \\ $\max\left\{ \sigma^{-1/2} \erma^2, \ep^{-2} d \right\}$ \\ with $\erma$-approx.~oracle} &  \thead{\CLS: \\ $\max \left\{ \sigma^{-1/2}, \ep^{-2} d \right\}$ \\ with exact oracle} \\ \hline
      
      \thead{Theorem 3.2 from \\ \cite{haghtalab2021smoothed}} & \thead{Statistical \\ \textbf{lower bound}  \\ for any algorithm} & \thead{\CLS: $\epsilon^{-2} d \log(1/\sigma)$ \\  \REG: $\epsilon^{-2} \log(1/\sigma)$} & NA & NA \\ \bottomrule
    \end{tabular}
  \end{center}
  \caption{{\small Overview of our main results. For each algorithm/lower bound, the number of iterations $T$ after which the algorithm achieves regret $\leq \ep T$ is shown for two cases: (a) \textbf{Cls:} the case of binary classification for a class $\MF$ of VC dimension $d$; (b) \textbf{Reg:} the case of regression for a class $\MF$ with scale-sensitive VC dimension bounded as $\vc(\MF, \alpha) \lesssim \alpha^{-p}$, for some $0 < p < 2$ (our results extend to the case of $p \geq 2$, which may be found in the theorem statements).  We take $L = 1$ and suppress logarithmic factors where possible. 
  }}
    \label{tab:results}
  \end{table}
   \footnotetext{This bound was shown in \cite{haghtalab2021smoothed}, though with slightly different log factors.
   }

 In an independent and concurrent work, \cite{haghtalab2022oracle} established several results similar to our own. In particular, they provided an analysis of an \emph{improper} algorithm similar to our oracle-efficient improper learner from Section \ref{sec:relax}, as well as computational lower bounds similar to those presented in our Section \ref{sec:lowerbound}. Additionally, they presented an oracle-efficient \emph{proper} learner in the case of binary classification that, while based on the principle of FTPL, has a substantially different analysis than our own in Section \ref{sec:ftpl}; note that our algorithm also applies for general, real-valued function classes in the nonparametric regime.
  
We discuss further related work in Appendix \ref{app:relatedwork}.

\section{Problem Setup and Notation}
\label{sec:prelim}
In this section we formally define the smoothed online learning setting.  We then introduce some concepts and notation we use throughout the paper.

\paragraph{Miscellaneous notation.}
For distributions $p,q$ on a measure space $\MX$, we write $p \ll q$ if $p$ is absolutely continuous with respect to $q$. For a positive integer $m$, let $[m] = \{1, 2, \ldots, m\}$.  For expressions $f,g$ we say $f \lesssim g$ if there is some universal constant $C$ such that $f \leq C g$.  We also use $f = O(g)$ to signify the same thing.

\subsection{Smoothed Online Learning}
We consider the setting of smoothed online learning, following \cite{haghtalab2021smoothed}. Consider a space $\MX$ of covariates equipped with some sigma-algebra. 
Let $\F \subset [-1, 1]^{\cX}$ be a function class and $\ell: [-1,1] \times [-1,1] \to [0,1]$ be an $L$-Lipschitz, convex loss function for some $L > 0$. Fix some $T \in \BN$ denoting the number of rounds of learning.  
We consider the following learning setting, making a distinction between \emph{improper learning} and \emph{proper learning}:
\begin{enumerate}
    \item For each $t \in [T]$, nature samples $x_t \sim p_t$ for some distribution $p_t$ that may depend in any way on the past samples $x_s$ for $s < t$ and the algorithm's past predictions. Nature also chooses $y_t \in [-1,1]$ adversarially (perhaps depending on $x_t$) in a similar manner. 
    \item The learner makes a (possibly random) prediction as follows:
      \begin{itemize}
      \item \textbf{Improper learner:}       The learner observes $x_t$ and makes a prediction $\yhat_t \in [-1,1]$.
      \item \textbf{Proper learner:} The learner chooses a hypothesis $\fhat_t \in \MF$, and its prediction is defined as $\yhat_t := \fhat_t(x_t)$; the learner then observes $x_t$. 
      \end{itemize}
    \item Nature reveals $y_t \in [-1,1]$ to the learner, and the learner incurs loss $\ell(\yhat_t, y_t)$.
\end{enumerate}
With no restrictions on $p_t$, the above online learning problem has been extensively studied, and essentially tight rates are known \citep{rakhlin2015sequential,pmlr-v134-block21a}; furthermore, exponential lower bounds are known for oracle-efficient proper learning algorithms \citep{hazan2016computational}. To circumvent these lower bounds, we consider the smooth setting. Our fundamental assumption (Definition \ref{def:smooth} below) is that there is a distinguished distribution $\mu$ on $\MX$, accessable to the learner through efficient sampling, so that the adversary is constrained to choose covariates $x_t$ according to some distribution with bounded Radon-Nikodym derivative with respect to $\mu$. This assumption, which follows that of \cite{haghtalab2020smoothed,haghtalab2021smoothed}, has been used extensively as well in the smoothed analysis of local search algorithms \citep{manthey2020smoothed} and discrete optimization problems \citep{beier2004typical}.  We emphasize the assumption that in all cases the learner has access to $\mu$ through efficient sampling; note that a typical example to keep in mind is that $\mu$ is uniform on some set, so the efficient sampling assumption is not very restrictive.
\begin{definition}[Adaptive smooth distributions]
  \label{def:smooth}
    Let $p, \mu$ be probability measures on a set $\cX$.  We say that $p$ is \emph{$\sigma$-smooth} with respect to $\mu$ if $p \ll \mu$ and
    \begin{equation}
        \esssup \frac{d p}{d \mu} \leq \frac 1\sigma.
    \end{equation}
    Let $\p(\sigma, \mu)$ denote the class of all distributions $p$ that are $\sigma$-smooth with respect to $\mu$.  We denote this class simply by $\p$ when $\sigma, \mu$ are clear from the context.  For any $T \in \BN$, we let $\ps{T}{\sigma}{\mu}$ denote the space of joint distributions $\sD$ on $x_1, \ldots, x_T \in \MX$ satisfying the following property: letting $p_t$ denote the law of $x_t$ conditional on $x_s$ for all $s <t$, for all $t \in [T]$, $p_t \in \p(\sigma, \mu)$ almost surely.  Similarly, we let $\tps{T}{\sigma}{\mu}$ denote the space of joint distributions $\sD$ on $(x_1, y_1), \ldots, (x_T, y_T) \in \MX \times \YY$ such that if $p_t$ is the law of $x_t$ conditional on $(x_s, y_s)$ for $s < t$, then $p_t \in \p(\sigma, \mu)$ almost surely.  Note that no constraints are placed on the distribution of $y_t$.
  \end{definition}
For $\sigma = 1$ we recover the notion of the data $x_1, \ldots, x_T$ being sampled iid from $\mu$; as $\sigma$ tends to zero, the notion of $\sigma$-smoothness becomes weaker and thus we consider $\sigma$-smoothness as interpolating between the favorable situation of i.i.d. data and the unfavorable adversarial situation.

\subsection{Minimax value}
The goal of the learner is to minimize expected regret to the best function in $\F$, defined as
\begin{equation}
    \ee\left[\reg_T(\F)\right] = \ee\left[\sum_{t = 1}^T \ell(\yhat_t, y_t) - \inf_{f \in \F} \sum_{t = 1}^T \ell(f(x_t), y_t)\right]
\end{equation}
with the expectation taken over both the sampling of $x_t \sim p_t$ and the learner's possibly randomized predictions.  For any function class, we consider the minimax regret for proper and improper learners, respectively, to be the values $\V_T^\prop(\F, \ps{}{\sigma}{\mu})$ and $\V_T^\impr(\F, \ps{}{\sigma}{\mu})$, respectively, where
\begin{align}\label{eq:mimimaxregret}
  \V_T^\prop (\F, \ps{}{\sigma}{\mu}) &= \left\langle \inf_{q_t \in \Delta(\MF)} \sup_{p_t \in \p(\sigma, \mu)}  \ee_{x_t \sim p_t} \sup_{y_t \in [-1,1]} \ee_{\fhat_t \sim q_t}\right\rangle_{t = 1}^T \left[\sum_{t = 1}^n \ell(\fhat_t(x_t), y_t) - \inf_{f \in \F} \sum_{t = 1}^T \ell(f(x_t), y_t)\right] \\
  \V_T^\impr (\F, \ps{}{\sigma}{\mu}) &= \left\langle \sup_{p_t \in \p(\sigma, \mu)} \ee_{x_t \sim p_t} \inf_{q_t \in \Delta([-1,1])} \sup_{y_t \in [-1,1]} \ee_{\yhat_t \sim q_t} \right\rangle_{t = 1}^T \left[\sum_{t = 1}^n \ell(\yhat_t ,y_t)- \inf_{f \in \F} \sum_{t = 1}^T \ell(f(x_t), y_t)\right].  \nonumber
\end{align}
where $\left\langle \cdot \right\rangle_{t = 1}^T$ denotes iterated application of the enclosed operators.  It is straightforward to see that $\V_T^\prop(\F, \p(\sigma,\mu)) \geq \V_T^\impr(\F, \p(\sigma,\mu))$.

\subsection{ERM oracle model}
\label{sec:erm-oracle}
To capture the notion of computational efficiency, we consider the following \emph{ERM oracle model}: 
\begin{definition}[ERM oracle]
  \label{def:oracle}
  For $\erma > 0$, a \emph{$\erma$-approximate (weighted) empirical risk minimization (ERM) oracle} takes as input a sequence $(x_1, y_1), \ldots, (x_m, y_m) \in \MX \times [-1,1]$ of data, a sequence $w_1, \dots, w_m \in \mathbb{R}$ of weights, and a sequence $\ell_1, \dots, \ell_m$ of $[-1,1]$-valued loss functions and outputs some $\fhat \in \MF$ satisfying
  \begin{align}
\sum_{i=1}^m w_i \ell_i(\fhat(x_i), y_i) \leq \inf_{f \in \MF} \sum_{i=1}^m w_i \ell_i(f(x_i), y_i) + {\erma} \cdot {\sum_{i=1}^m |w_i|}.\label{eq:erm-def}
  \end{align}
\end{definition}
We remark that while Definition \ref{def:oracle} allows for an arbitrary sequence $\ell_1, \ldots, \ell_m$ of loss functions, all of our algorithms will set each $\ell_i$ equal to either $\ell$ (the given loss function of the learning problem) or $\ell^\textrm{Id}$, where $\ell^{\textrm{Id}}(\yhat, y) := \yhat$. 

We will measure the computation cost of algorithms via the following two metrics: first, the \emph{number of calls to the ERM oracle}, and second, the \emph{total computation time}. To define the latter, we must specify the manner in which our algorithm interacts with the ERM oracle: in particular, we assume that there is a certain region of memory on which the algorithm lists tuples $(x_i, y_i, w_i)$. Listing (or modifying) any such tuple takes unit time, as does calling the ERM oracle, which performs the optimization in \eqref{eq:erm-def} in unit time. This convention mirrors that of oracle machines in complexity theory \citep{arora2009computational}. 
We refer to the case that $\zeta = 0$ the \emph{exact ERM oracle model}, and the case that $\zeta > 0$ the \emph{approximate ERM oracle model}. 

\subsection{Statistical complexities}
For a space $\MX$, an $n \in \BN$, and a function class $\MF \subset [-1,1]^\MX$, the Rademacher complexity, $\rc{n}{\MF}$, and Gaussian complexity, $\gc{n}{\MF}$, conditional on $x_1, \dots, x_n$, are defined as follows:
\begin{align}
\rc{n}{\MF} := \E_{\epsilon} \left[ \sup_{f \in \MF} \sum_{i=1}^n \ep_i f(x_i) \right], \qquad \gc{n}{\MF} := \E_{\gamma} \left[ \sup_{f \in \MF} \sum_{i=1}^n \gamma_i f(x_i) \right]\nonumber,
\end{align}
where the $\ep_i$ are i.i.d.~Rademacher random variables, and the $\gamma_i$ are i.i.d.~standard normal random variables. It is a well-known fact that there is a universal constant $C$ so that $\frac{1}{C} \cdot \rc{n}{\MF} \leq \gc{n}{\MF} \leq C \log n \cdot \rc{n}{\MF}$ for all $\MF$ and $n$.

Furthermore, we consider the notion of scale-sensitive VC dimension from \cite{kearns1994efficient,bartlett1996fat} that characterizes learnability of $\F$ in a batch setting.  For any $\alpha > 0$ and points $x_1, \dots, x_m \in \cX$, we say that $\F$ is shattered by the $x_i$ at scale $\alpha$ with witness $s_1, \dots, s_m \in \mathbb{R}$ if for all $(\epsilon_1, \dots, \epsilon_m) \in \{\pm 1\}^m$ there is an $f_\epsilon \in \F$ such that
\begin{equation}
  \epsilon_i (f_\epsilon(x_i) - s_i) \geq \frac \alpha 2 \qquad \text{for all } i
\end{equation}
We define the VC dimension of $\F$ at scale $\alpha$, denoted by $\vc(\F, \alpha)$ to be the largest $m$ such that there exists a shattering set of size $m$.  We let $\vc(\F) = \lim_{\alpha \downarrow 0} \vc(\F, \alpha)$ denote the VC dimension of $\F$.

\section{Minimax Value}\label{sec:minimaxvalue}

In this section, we provide tight bounds on $\V_T^\prop(\F, \p(\sigma, \mu))$ without regard to oracle efficiency.   
While our proof is nonconstructive, we emphasize that our results in Section \ref{sec:lowerbound} demonstrate that any proper algorithm based on ERM oracle calls which achieves the optimal dependence on $\sigma$ cannot be computationally efficient.  Our results show that $\V_T^\prop(\F, \ps{}{\sigma}{\mu})$ is always a $\mathrm{poly}(\log(T/ \sigma))$ factor away from the optimal statistical rates achievable in the batch setting.  We now present our bound:
\begin{theorem}\label{thm:chainingupperbound}
    Let $\F: \cX \to [-1, 1]$ be a real-valued function class and denote by $\vc(\F, \delta)$ the scale-sensitive VC dimension of $\F$ at scale $\delta > 0$.  Then, for some $c > 0$
    \begin{equation}
        \V_T^\prop(\F, \p(\sigma, \mu)) \lesssim L \log^{3/2}(T) \cdot \log\left(\frac{T}{\sigma}\right) \inf_{\alpha > 0} \left\{ T\alpha + \sqrt{T} \int_{\alpha}^2 \sqrt{\vc\left(\F, c \delta \right)} d \delta  \right\}
    \end{equation}
    In particular, if $\vc(\F, \delta) \leq d \log\left(\frac 1\delta\right)$ for all $\delta > 0$, then
    \begin{equation}
        \V_T^\prop(\F, \p(\sigma, \mu)) \lesssim L \log^{\frac 32}(T) \log\left(\frac{T}{\sigma}\right)\sqrt{d T}\label{eq:vprop-parametric}
    \end{equation}
    and if $\vc(\F, \delta) \lesssim \delta^{-p}$ for some $p < \infty$, then
    \begin{equation}
        \V_T^\prop(\F, \p(\sigma, \mu)) \lesssim L \log^{\frac 32}(T) \log\left(\frac{ T}{\sigma}\right) T^{\max\left(\frac 12, 1 - \frac 1p\right)} 
    \end{equation}
\end{theorem}
Our result extends \cite[Theorem 3.1]{haghtalab2021smoothed} to the cases of real-valued and nonparametric function classes.  In that paper, in order to prove their regret bounds for smoothed online classification, the authors introduced the clever approach of coupling, showing that if a distribution $p$ is smooth with respect to the uniform distribution on a discrete set, then in expectation we may pretend the data comes independently from the uniform distribution.
In Appendix \ref{app:minimaxvalue}, we generalize their result in Lemmas \ref{lem:coupling} and \ref{lem:coupling2} with a dramatically simpler proof\footnote{Already in the case of discrete distributions, the proof of \cite[Theorem 2.1]{haghtalab2021smoothed} first demonstrates their claim for uniform measures and then proceeds to apply convex analysis for the general case.  They then claim that Choquet's Theorem \cite[Corollaire 8]{choquet1963existence} implies the general case when $\cX$ is not discrete; however, it is not \emph{a priori} obvious that $\p(\sigma, \mu)$ is compact in the relevant topology.}.  Namely, for any $k \in \BN$ we construct a coupling between $(x_1, \dots, x_T) \sim \sD \in \p_T(\sigma, \mu)$ and $\{Z_t^j\}_{\substack{1 \leq t \leq T \\ 1 \leq j \leq k}} \sim \mu^{\otimes kT}$ such that $\{x_1, \dots, x_T\} \subset \{Z_t^j\}_{\substack{1 \leq t \leq T \\ 1 \leq j \leq k}}$ with probability at least $1 - e^{- \sigma k}$.

Curiously, our proof of Theorem \ref{thm:chainingupperbound} is quite different from that of \cite[Theorem 3.1]{haghtalab2021smoothed}, although we still use the coupling.  In that paper, the authors apply the coupling to analyze covering numbers with respect to $L^2(\mu)$ (i.e. in the i.i.d. sense); 
the natural extension of this technique would be to apply chaining \citep{dudley1967sizes} and bound $\V_T^{prop}(\F, \sD)$ by $\ee_\mu\left[\R_T(\F)\right]$. Unfortunately, Proposition \ref{prop:nonparametriclowerbound} in the sequel shows that such a bound is not possible without a suboptimal, polynomial dependence on $\sigma$.  
We instead go by a different approach, which is based on the observation that the sequential fat-shattering dimension (see Definition \ref{def:seqfat}) is bounded above by the scale-sensitive VC dimension times a logarithmic factor in the domain size (Lemma \ref{lem:finitedomain}). Though the true domain $\MX$ may be infinite, we show that it is possible adapt the coupling lemma of \cite{haghtalab2021smoothed} to bound an ``effective'' domain size. In combination with the non-constructive bounds using distribution-dependent sequential Rademacher complexity from \cite{rakhlin2011online,pmlr-v134-block21a}, this provides a tight characterization of the minimax regret's dependence on the horizon $T$, the complexity of the function class, and the smoothness parameter $\sigma$. 
A full proof is in Appendix \ref{app:minimaxvalue}.  Combining Theorem \ref{thm:chainingupperbound} with \cite[Theorem 3.2]{haghtalab2021smoothed}, we have a complete characterization of the statistical rates of smooth online learning.  We further note that our proof applied to the results of \cite{rakhlin2015square} immediately extends to nonconstructively show that the dependence of the minimax value on $T$ for squared loss are the expected ``fast rates'' from \cite{rakhlin2017empirical}.  Unfortunately, our efficient algorithms below do not provably achieve these rates; resolving this disparity is an interesting future direction with applications to the study of contextual bandits, as described in Appendix \ref{app:bandits}.

Finally, before we proceed to consider oracle-efficient algorithms, we note that the requirement that the learner has access to $\mu$ cannot be dropped without a substantial loss in the regret.  In particular, we have the following result:
\begin{proposition}\label{prop:unknownmu}
    There exists a function class $\F: [0,1] \to \{\pm 1\}$ with $\vc(\F) = 1$ and an adversary that is $\sigma$-smooth with respect to some unknown $\mu$ such that, no matter how the learner chooses $\yhat_t$, it holds for $T \leq \frac 1\sigma$ that
    \begin{equation}
        \ee\left[\reg_T\right] \geq \frac{T}{2}.
    \end{equation}
\end{proposition}
Proposition \ref{prop:unknownmu} is proved by letting $\F$ be thresholds on the unit interval and allowing the $\mu$ with respect to which the adversary is $\sigma$-smooth adapt to the data sequence.  The details can be found in Appendix \ref{app:minimaxvalue}.  In particular, the result shows that for the learner to be able to ensure regret scaling as in Theorem \ref{thm:chainingupperbound}, he needs to have access to $\mu$ in some way.  Critically, the size of the set of possible $\mu$ in the lower bound of Proposition \ref{prop:unknownmu} is growing exponentially with $T$; indeed, if we know that our adversary is $\sigma$-smooth with respect to some $\mu \in \mathscr P$, a finite class of distributions, then we can use Hedge \citep{freund1997decision} to aggregate predictions assuming smoothness with respect to each $\mu \in \mathscr P$ and add an additive term of size $O\left(\sqrt{\log(\abs{\mathscr P}) T}\right)$ to our regret.



\section{Relaxations and Oracle-Efficient Algorithms}\label{sec:relax}

In\footnote{In an earlier version of the paper, we used a slightly different relaxation with a worse rate.  With a minor modification, we get a quadratic improvement in the dependence of the regret on $\sigma$.  While our improvement was independent of other work, we note that \cite{haghtalab2022oracle} present the same final relaxation and analysis.} the previous section, we derived sharp bounds for the minimax regret in the smoothed online setting, with sharp dependence on the key parameter $\sigma$.  A natural next step is to design an algorithm that achieves these bounds.  One possibility, suggested in \cite{haghtalab2021smoothed}, constructs a $\frac 1{\sqrt T}$-net on $\ell \circ \F$ with respect to $L^2(\mu)$ and runs Hedge \citep{freund1997decision} on the resulting covering.  Unfortunately, in the nonparametric case, after optimizing $\delta$ this approach yields suboptimal rates, corresponding to one-step discretization.  Ideally, an algorithm achieving optimal regret would construct nets at multiple scales and aggregate the resulting predictions in some way.  While there has been some progress on how to do this \citep{cesa1999prediction,gaillard2015chaining,daskalakis2021fast}, optimal rates are not yet achievable; in any case, relying on the construction of $\delta$-nets is inefficient as these can be exponentially large.

Thus in order to bring the smoothed online learning paradigm from the world of theory into that of practice, we need more efficient algorithms.  Presently, we describe an oracle-efficient improper learning procedure.  As we shall see, the algorithm has regret with optimal dependence on the horizon, $T$, but suboptimal dependence on $\sigma$.  We will improve the dependence on $\sigma$ with a proper learning algorithm in the following section, at the cost of worse dependence on $T$ in general.  Here, we leverage the relaxation approach, studied in \cite{rakhlin2012relax}.
\begin{definition}
Fix $T \in \BN$.    A sequence of real valued functions $\rel_T(\F | x_1, \dots, x_t) : \cX^{t} \to \rr$, $t \leq T$, is a relaxation if for any $x_{1:T} \in \cX$, we have the following two properties:
    \begin{align}
        \sup_{p_t \in \p(\sigma, \mu)} &\ee_{x_t' \sim p_t}\inf_{q_t \in \Delta([-1,1])} \sup_{y_t' \in [-1,1]} \left[\ee_{\yhat_t \sim q_t}[\ell(\yhat, y_t')] + \rel_T(\F| x_1, y_1, \dots, x_t', y_t')\right]  \label{eq:relaxation} \\
        &\leq \rel_T(\F | x_1, y_1, \dots, x_{t-1}, y_{t-1}) \\
        - \inf_{f \in \F} \sum_{t = 1}^T \ell(f(x_t), y_t) &\leq \rel_T(\F| x_1, y_1, \dots, x_T, y_T)
    \end{align}
\end{definition}
As established in \cite[Proposition 1]{rakhlin2012relax}, a relaxation gives rise to both an algorithm and an associated regret bound; indeed, any $q_t$ guaranteeing \eqref{eq:relaxation} at each time $t$ yields a regret at most $\rel_T(\F)$; the challenge, of course, is to define the relaxation.  We have the following result:
\begin{proposition}\label{prop:slowrelaxation}
    Suppose that $\sD \in \p(\sigma, \mu)$.  Then, for any function class $\F$ and $L$-Lipschitz, convex loss function $\ell$, and any $k \in \mathbb{N}$,
    \begin{equation}\label{eq:slowrelaxation}
        \rel_T(\F| x_1, y_1, \dots, x_t, y_t) = \ee_{\mu, \epsilon}\left[\sup_{f \in \F} 2 L \sum_{j = 1}^k \sum_{s = t + 1}^T \epsilon_{s,j} f(x_{s,j}) - \sum_{s = 1}^t \ell(f(x_s), y_s)\right] + (T - t)^3 e^{- \sigma k}
    \end{equation}
    is a relaxation, where the expectation is over independent $x_{s,j} \sim \mu$ and Rademacher random variables $\epsilon_{s,j}$ for $s > t$.
\end{proposition}
We provide a proof in Appendix \ref{app:relax} that uses the minimax theorem, symmetrization and Lemma \ref{lem:coupling}.  Applying \cite[Lemma 5.1]{rakhlin2014statistical} to reduce to deterministic predictions, Proposition \ref{prop:slowrelaxation} gives rise to an algorithm that plays
\begin{equation}
    \yhat_t = \argmin_{\yhat \in [-1,1] } \sup_{y_t \in [-1, 1]}  \ell(\yhat, y_t) + \ee_{\mu, \epsilon}\left[\sup_{f \in  \F} 2 L \sum_{j = 1}^k \sum_{s = t + 1}^T \epsilon_{s,j} f(x_{s,j}) - L_t(f)\right]
\end{equation}
where we drop the additive constant from \eqref{eq:slowrelaxation} because it does not depend on $f$ and we let $L_t(f) = \sum_{s = 1}^t \ell(f(x_s), y_s)$.  After optimizing the resulting regret bound with respect to $k$, we get that the regret of this algorithm is $O\left( \R_{T\frac{\log T}{\sigma}}(\F)\right)$, which has an optimal dependence on $T$ up to logarithmic factors, but is suboptimal with respect to $\sigma$.  Note that while the supremum inside of the expectation in \eqref{eq:slowrelaxation} can be solved with an ERM oracle by letting $\ell_s(f(x_{s,j}), y_s) = f(x_{s,j})$ for $x > t$, we still require a costly integration in order to find $\yhat_t$.  We can compute this expectation by sampling from $\mu$ and applying concentration but this approach requires many calls to the ERM oracle.  Motivated by the random playout idea in \cite{rakhlin2012relax}, we propose a much more efficient algorithm:
\begin{theorem}\label{thm:fastrelaxation}
    Suppose that $\sD \in \p(\sigma, \mu)$, $\F: \cX \to [-1,1]$ is a function class, and $\ell$ is an $L$-Lipschitz, convex loss function.  At each time $t$, for $1 \leq j \leq k$, sample $x_{t+1,j}, \dots, x_{T,j} \sim \mu$ and $\epsilon_{t+1,j}, \dots, \epsilon_{T,j}$ independently.  After observing $x_t$, predict
    \begin{equation}\label{eq:fastrelaxation}
        \yhat_t = \argmin_{\yhat \in [-1,1]} \sup_{y_t \in [-1,1]}\left\{\ell(\yhat, y_t) + \sup_{f \in \F}\left[6L\sum_{j = 1}^k \sum_{s = t+1}^T \epsilon_{s,j} f(x_{s,j}) - L_t(f)\right] \right\}
    \end{equation} 
    Then the expected regret against any smooth adversary is
    \begin{equation}
        \ee\left[\reg_T\right] \leq 6L \ee_\mu\left[\R_{kT}(\F)\right] + \sqrt{T} + T^3 e^{- \sigma k}
    \end{equation}
    Moreover, this regret can be achieved with $O\left(\sqrt{T} \log T \right)$ calls to the ERM oracle per round in general and only $2$ calls per round in the special case that $\ell(\yhat, y) = \frac{1 - \yhat y}2$.  In particular, when $\vc(\F, \delta) \lesssim \delta^{- p}$ for some $p > 0$, we have
    \begin{equation}
        \ee\left[\reg_T\right] \lesssim  L \left(\frac{T\log \left(T\right)}{\sigma}\right)^{\max\left(\frac 12, 1 - \frac 1p\right)}
    \end{equation}
\end{theorem}
Theorem \ref{thm:fastrelaxation} is proved in Appendix \ref{app:relax}.  To understand why \eqref{eq:fastrelaxation} is oracle-efficient, note that we can discretize the interval $[-1,1]$ at scale $\frac{1}{L \sqrt{T}}$ to produce a set $S$ of size $2 L \sqrt{T}$.  For each $\yhat \in S$, we can exhaustively search $S$ for the optimal $y_t$ with $O(\sqrt{T})$ calls to the (value of the) ERM oracle.  Because the problem is convex in $\yhat$ (due to the convexity of $\ell$), we can run zeroth order optimization as in \cite{agarwal2011stochastic} to find $\yhat_t$ up to $\frac 1{\sqrt T}$ error in $O\left(\sqrt T \log T\right)$ calls to the oracle.  If the losses are linear, then the problem is convex in $y_t$ and is thus extremized on the boundary; further leveraging the linear loss allows the problem to be solved in 2 oracle calls per round.  Note that in the case of $y$ being binary-valued, we can think of $\ell(\yhat, y) = \frac{1 - \yhat y}2$ as the indicator loss when guessing $\yhat \in \{\pm 1\}$ and thus, for classification, we can get optimal regret with 2 oracle calls per round.

While Theorem \ref{thm:chainingupperbound} demonstrates that regret can depend on $\sigma$ only through a logarithmic factor, our relaxation-based algorithm has a polynomial dependence on $\sigma$.  Unfortunately, this polynomial dependence cannot in general be eliminated for any relaxation relying on the classical Rademacher complexity.  To see this, note that the regret of any algorithm is bounded below by $\V_T^{improp}(\F, \p(\sigma, \mu))$.  The following proposition shows that the value is in turn bounded below by a polynomial factor of $\sigma$.
\begin{proposition}\label{prop:nonparametriclowerbound}
    For any $\sigma \leq 1$ and $0 < p < 2$, there exists a measure $\mu$ and a function class $\F$ satisfying $\vc(\F, \delta) \lesssim \delta^{-p}$ such that for all $T \gtrsim \frac 1\sigma \log\left(\frac 1\sigma\right)$ with $\ell$ the absolute loss,
    \begin{equation}
        \sigma^{- \frac p4} \ee_\mu\left[\R_T(\F)\right] \lesssim \V_T^{improp}(\F, \p(\sigma, \mu))
    \end{equation}
\end{proposition}
We construct a measure such that the learning problem is easy when the population distribution is $\mu$ by having $\mu$ concentrate a lot of mass on a distinguished point $x^\ast$ on which all the functions in $\F$ agree; thus, a sample from $\mu$ will include many copies of $x^\ast$, which incur no regret.  We then consider an i.i.d. adversary and let $p_t$ be uniform over a set of points that shatters $\F$ at scale $\sqrt{\sigma}$; in this way, we make it so a sample from $p_t$ will incur high regret and the gap between the performance on samples from $p_t$ and $\mu$ is relatively large.  A complete proof can be found in Appendix \ref{app:minimaxvalue}.  Note that Proposition \ref{prop:nonparametriclowerbound} is not in conflict with Theorem \ref{thm:chainingupperbound} because the example described above makes $\R_T(\F)$ polynomially small in $\sigma$ for a carefully designed $\mu$; in essence, the separation is created by making the Rademacher complexity much smaller than expected based on the complexity of the function class $\F$.

In the improper procedure \eqref{eq:fastrelaxation}, we have our first efficient algorithm for the smoothed online learning setting that works for a generic function class and achieves an optimal regret dependence on the horizon $T$.  There are three drawbacks to Theorem \ref{thm:fastrelaxation}.  First, $\yhat_t$ is \emph{improper}.  Second, our dependence on $\sigma$ is significantly worse than the optimal statistical dependence explored in Section \ref{sec:minimaxvalue}.  Third, while the algorithm is efficient, we may hope to have an algorithm that makes only $1$ oracle call per round in general.  We address these issues in the following section.


\section{Follow the Perturbed Leader and Oracle-Efficient Proper Learning}\label{sec:ftpl}

In the previous section, we provided an improper oracle-efficient algorithm that achieves optimal dependence on the horizon $T$, but is improper and requires more than one oracle call per round;  here we demonstrate that a proper learner can have similar regret in some situations with only 1 oracle call per time step.  In the following section, we will show that  
our algorithm's regret is optimal up to a polynomial factor for any oracle-efficient algorithm.  

In \cite{rakhlin2012relax}, the authors make use of the connection between relaxations, random playout, and the Follow the Perturbed Leader (FTPL) style algorithms pioneered in \cite{kalai2005efficient} to make the relaxation approach more efficient in some cases.  We expand upon this approach, using Theorem \ref{thm:fastrelaxation} as a starting point.  Indeed, the prediction $\yhat_t$ in \eqref{eq:fastrelaxation} is cosmetically very similar to that of FTPL, were we recall that the FTPL approach introduces a noise process $\omega(f)$ and, at each time step, sets
\begin{equation}\label{eq:ftpldef}
    f_t \in \argmin_{f \in \F} L_{t-1}(f) + \eta \omega(f)
\end{equation}
for some real-valued parameter $\eta$.  The perturbation $\eta \omega(f)$ acts to regularize the predictions; typically, the noise $\omega$ is independent across functions, with the classic example being exponential noise in \cite{kalai2005efficient}.  On the other hand, up to a sign, the supremum in \eqref{eq:fastrelaxation} returns the optimal value of $L_t(f) + \eta \omega(f)$ with appropriate values of $\eta$ and letting $\omega$ be the Rademacher process.  It is natural to wonder, then, if the min-max problem that is the source of the extra oracle calls is really necessary; we show below that it is not in the sense that FTPL provides an efficient, proper algorithm.

Were we to apply existing FTPL results, using independent perturbations for each $f$, we would require the enumeration of representative ``experts,'' which would preclude the desired oracle-efficiency.  Above, we saw that a Rademacher process perturbation is motivated by the relaxations of the previous section, but analysis is much easier with a Gaussian process.  We first treat the case of binary classification:
\begin{theorem}\label{thm:ftplclassification}
    Suppose that $\F: \cX \to [-1,1]$ is a function class and $\ell$ a loss function that is Lipschitz in both arguments.  Suppose further that we are in the smoothed online learning setting where $x_i$ are drawn from a distribution that is $\sigma$-smooth with respect to some distribution $\mu$ on $\cX$.  Let
    \begin{equation}
        \hat{\omega}_{t,n}(f) = \frac 1{\sqrt n}\sum_{ i =1}^n \gamma_{t,i} f(Z_{t,i})
    \end{equation}
    where the $Z_{t,i} \sim \mu$ are independent and the $\gamma_{t,i}$ are indpendent standard normal random variables.  Suppose that $\zeta \geq 0$ and consider the algorithm which uses the approximate ERM oracle to choose $f_t$ according to\footnote{Note that we have not included the total weight multiplying $\erma$ in \eqref{eq:ftpl-erm}, as in \eqref{eq:erm-def}; thus we are technically using a $\frac{\zeta}{T + \log(1/\delta) \cdot n}$-approximate ERM oracle with probability $1-O(\delta)$.}
    \begin{equation}
        L_{t-1}(f_t) + \eta \hat{\omega}_{t,n}(f_t) \leq \inf_{f \in \F} L_{t-1}(f) + \hat{\omega}_{t,n}(f) + \zeta \label{eq:ftpl-erm}
    \end{equation}
    and let $\yhat_t = f_t(x_t)$.  If $\F$ and $y_t$ are binary valued, $\vc(\F) \leq d$, then for appropriate choices of $n$ and $\eta$\footnote{Specified in Proposition \ref{prop:ftplclassificationbound} in Appendix \ref{subsec:conclude}}
    \begin{equation}
        \ee\left[\reg_T(f_t)\right] \lesssim \sqrt{\frac{d T \log T}{\sigma}}  + \zeta T
    \end{equation}
    More generally, if we let
    \begin{equation}
        \hat{\omega}_{t,n}(f) = \frac 1{\sqrt n} \sum_{i = 1}^n \gamma_{t,i} \ell(f(Z_{t,i}), y_{t,i})
    \end{equation}
    with $y_{t,i}$ drawn uniformly from $\epsilon \mathbb{Z} \cap [-1,1]$ and $\vc(\F, \delta) \lesssim \delta^{-p}$ for some $p < 2$, then for appropriate choices of the parameters\footnote{Given in Corollary \ref{cor:ftplsimplealgo} in Appendix \ref{subsec:conclude}.}, if $f_t$ is chosen according to \eqref{eq:ftpl-erm}, 
    \begin{equation}
        \ee\left[\reg_T(f_t)\right] \lesssim T^{\frac 34} \sigma^{- \frac 14} \log\left(\frac T \sigma\right) + \zeta T
    \end{equation}
\end{theorem}
In order to improve the regret for the case of real-valued labels, we introduce a second, stabilizing perturbation.  The following result bounds the regret of the resulting algorithm:
\begin{theorem}\label{thm:generalftpl}
    Suppose that $\F : \cX \to [-1,1]$ is a function class and $\ell: [-1,1] \times [-1,1] \to [0,1]$ is a loss function that is $L$-Lipschitz in \emph{both} arguments.  Suppose further that we are in the smooth online learning setting where $x_t$ is chosen from a distribution that is $\sigma$-smooth with respect to some $\mu$.  Fix $\epsilon > 0$ and consider the following two processes:
    \begin{align}
        \hat{\omega}_{t,m}(f) = \frac{1}{\sqrt{m}} \sum_{i = 1}^m \gamma_{t,i} f(Z_{t,i}) && \hat{\omega}_{t,n}'(f) = \sum_{j = 1}^n \gamma_{t,j}' \ell(f(Z_{t,j}'), y_{t,j}')
    \end{align}
    where $\gamma_{t,i}, \gamma_{t,j}'$ are independent standard normal random variables, $Z_{t,i}, Z_{t,j}' \sim \mu$, and $y_{t,j}'$ are independent and uniform on $\epsilon \mathbb{Z} \cap [-1,1]$.  Suppose that $f_t$ is chosen according to
    \begin{equation}
        L_{t-1}(f_t) + \eta \hat{\omega}_{t,m}(f_t) + \hat{\omega}_{t,n}(f_t) \leq \inf_{f \in \F} L_{t-1}(f) + \eta \hat{\omega}_{t,m}(f) + \hat{\omega}_{t,n}(f) + \zeta
    \end{equation}
    If there is some $p < 2$ such that $\vc(\F, \delta) \lesssim \delta^{-p}$ then for appropriate choices of the parameters\footnote{Outlined in Corollary \ref{cor:ftplparameters} in Appendix \ref{subsec:conclude}.}, we have:
    \begin{align}
        \ee\left[\reg_T(f_t)\right] \lesssim \frac{T^{\frac 23} \log T}{\sigma^{\frac 13}} + \zeta T
    \end{align}
    If $2 \leq p < \infty$, then there are appropriate choices of parameters such that $\ee\left[\reg_T(f_t)\right] = o(T)$.

  \end{theorem}

In Appendix \ref{subsec:conclude} we provide a slightly more general form of the above regret bounds, including in the case when the labels are smooth with respect to a known measure.  Further, we give the precise dependence of our regret bounds on $n$ and $\eta$; the optimal values of these parameters are polynomial in $L, T, 1/\sigma$, and the complexity of the function class.  Interestingly, we can still achieve regret with the same dependence on $T$ by setting the parameters independently of $\sigma$ and $L$; this is useful for applications where we can assume $\sigma$-smoothness but do not know what $\sigma$ is.

The proofs of Theorems \ref{thm:ftplclassification} and \ref{thm:generalftpl} proceed similarly.  First, we apply a variant of the classic ``Be-the-Leader'' approach \cite[Lemma 3.1]{cesa2006prediction}, which leads to a regret decomposition into a perturbation term and a stability term.  The perturbation term is easily controlled with classical empirical process theory.  For the stability term, we further decompose the regret into a term quantifying the difference in losses of $f_t$ and $f_{t+1}$ on a tangent sequence and another quantifying the dependence of $f_{t+1}$ on $x_t, y_t$.  For the former term, we prove a novel anti-concentration inequality for the infimum of a Gaussian process, which may be of independent interest, and apply this inequality to control the Wasserstein distance between $f_t$ and $f_{t+1}$.    
We bound the latter term in the case of linear loss in a similar way as \cite{haghtalab2022oracle} did for the corresponding term in their algorithm.  This suffices for the binary labels case, but to extend to the more general setting, we use a discretization scheme to reduce to the case that the labels are also chosen in a smooth manner with respect to some distribution; we then reduce this setting to the case of linear loss and apply our earlier bound.  The details can be found in Appendix \ref{app:ftpl}.

The stability estimate was a significant technical challenge due to the complex dependence structure of $\hat{\omega}_{t,n}$ accross $\F$; most regret bounds for FTPL-style algorithms are simplified by independent perturbations.  To the best of our knowledge, Theorem \ref{thm:generalftpl} constitutes the first proof of an FTPL regret bound where the algorithm uses a generic Gaussian process as the perturbation.  


\section{Computational Lower Bounds}\label{sec:lowerbound}
Comparing the results of Theorem \ref{thm:chainingupperbound} and Theorem \ref{thm:generalftpl}, we notice that the requirement of oracle efficiency incurs an exponential loss in the regret's dependence on $\sigma$. In this section, we show that this exponential gap is necessary for any oracle-efficient algorithm. 
\begin{theorem}
   \label{thm:oracle-lb}
  Fix any $T \in \BN$ and $\sigma \in (0,1]$. 
In the ERM oracle model, any randomized algorithm cannot guarantee expected regret  smaller than $\frac{T}{200}$ against a $\sigma$-smooth online adversary over $T$ rounds and any  binary $\MF$ with $|\MF| \leq 1/\sigma$ in total time smaller than $\widetilde{O}(1/\sqrt{\sigma})$. 
\end{theorem}
Theorem \ref{thm:oracle-lb} is proved in Appendix \ref{app:aldous} by constructing a family of function classes on a space $\cX$ of size $1/\sigma$ and noting that a worst-case adaptive adversary is $\sigma$-smooth in this setting.  The construction then mirrors that in \cite{hazan2016computational}, which reduces from Aldous' problem \citep{aldous1983minimization}; the main difference being that Definition \ref{def:oracle} allows for negative weights in the ERM oracle, which complicates the proof.


As an immediate corollary of Theorem \ref{thm:oracle-lb}, we obtain the following regret lower bound for computationally efficient algorithms in the ERM oracle model, i.e., those whose total time after $T$ time steps is $\textrm{poly}(T)$:
\begin{corollary}
  \label{cor:stat-comp-gap}
  Fix any $\alpha \geq 1$, $\ep < 1/200, \sigma \in (0,1]$, and $d \geq \log 1/\sigma$. 
Any algorithm whose total time in the ERM oracle model over $T$ rounds is bounded as $T^\alpha$ requires that $T \geq \widetilde{\Omega}\left(\max \left\{ \frac{d}{\ep^2}, \sigma^{-\frac{1}{2\alpha}} \right\} \right)$ to achieve regret $\ep T$ for classes $\MF$ of VC dimension $d$ against a $\sigma$-smooth adversary.
\end{corollary}
Corollary \ref{cor:stat-comp-gap} and Theorem \ref{thm:chainingupperbound} show that there is an exponential statistical-computational gap for smoothed online learning in the ERM oracle model: for general classes $\MF$, it is possible to achieve regret proportional to $\log(1/\sigma)$, but the regret must be polynomial in $1/\sigma$ if the algorithm is required to be oracle-efficient.  While Theorem \ref{thm:oracle-lb} and Corollary \ref{cor:stat-comp-gap} get a lower bound only on the total computation time, as opposed to the number of oracle calls, we provide analogous results obtaining lower bounds on the number of oracle calls with an \emph{approximate} ERM oracle in Theorem \ref{thm:oracle-lb-apx} and Corollary \ref{cor:sc-gap-apx} in the appendix.  In particular, we show that any algorithm with $T^{O(1)}$ oracle calls with a $1/T^{O(1)}$-approximate ERM oracle needs $T \geq \max\{ d, 1/\ep, 1/\sigma \}^{\Omega(1)}$ to obtain sublinear regret against binary classes of VC dimension $d$.


\section*{Acknowledgements}
AB acknowledges support from the National Science Foundation Graduate Research Fellowship
under Grant No. 1122374. 
NG is supported by a Fannie \& John Hertz Foundation Fellowship and an NSF Graduate
Fellowship. 
AR acknowledges support from the ONR through awards N00014-20-1-2336 and N00014-20-1-2394, from the NSF through awards DMS-2031883 and DMS-1953181, and from the ARO through award W911NF-21-1-0328.

\bibliographystyle{alpha}
\bibliography{smoothedonlinelearning}

\appendix

\section{Applications to Contextual Bandits}\label{app:bandits}
We apply our results to the study of contextual bandits.  A series of recent papers \citep{foster2020beyond,simchi2021bypassing,foster2020instance} has focused on reducing the contextual bandit framework to that of online learning, with \cite{foster2020beyond} introducing an efficient and optimal reduction, \textsf{SquareCB}, that turns an online regression oracle into a fast, no-regret contextual bandit algorithm.  One of the key advantages of this reduction is the fact that the learner ``only'' has to design algorithms that have small regret in the full-information setting, thought to be an easier task than one that requires a careful balance of exploration and exploitation.  Unfortunately, there is still a dearth of oracle-efficient online algorithms with provably good regret in general, limiting the broader application of these results.  In \cite{simchi2021bypassing}, the authors use a similar reduction, but with an offline regression oracle, for which there are many practical algorithms; unfortunately, the result requires the contexts to arrive in i.i.d. fashion, unlike the more general result of \cite{foster2020beyond}.  Here, we show that whenever a function class is learnable in the offline setting, we can still use the \textsf{SquareCB} reduction to get a no-regret algorithm that is efficient with respect to an ERM oracle in the smooth contextual bandit setting.

We consider the setting described in \cite{foster2020instance} with the modification that contexts arrive in a $\sigma$-smooth manner.  Formally, at each $1 \leq t \leq T$, Nature selects a $\sigma$-smooth distribution $p_t$, and samples $x_t \sim p_t$, then samples a loss function $\ell_t$ independently from some distribution depending on $x_t$.  The learner selects an action $a_t \in [K]$ and observes $\ell_t(a_t)$.  We are given a function class $\F: \cX \times [K] \to [0,1]$ and suppose that there is some unknown $f^\ast \in \F$ such that $\ee\left[\ell_t(a) | x_t = x\right] = f^\ast(x, a)$ for all $x \in \cX$ and $a \in [K]$.  The goal is to minimize regret to the best policy induced by $\F$, where for any $f \in \F$, we define $\pi_f(x) = \argmin_{a \in [K]} f(x, a)$, i.e., we wish to minimize
\begin{equation}
    \reg_{CB}(T) = \sum_{t = 1}^T \ell_t(a_t) - \ell_t(\pi_{f^\ast}(x_t))
\end{equation}
We have the following result:
\begin{theorem}\label{thm:cbregret}
    Suppose we are in the $\sigma$-smooth Contextual Bandit setting described above.  If we run \textsf{SquareCB} with the relaxation-induced online regressor from \eqref{eq:fastrelaxation}, we can achieve
    \begin{equation}
        \ee\left[\reg_{CB}(T)\right] \leq 12 \frac{K \log T}{\sqrt{\sigma}} \sqrt{T \R_T(\F)}
    \end{equation}
    with $O\left(\sqrt T \log T\right)$ calls to the ERM oracle per round.  If we instead instantiate \textsf{SquareCB} with the FTPL algorithm from Theorem \ref{thm:generalftpl} and $\R_T(\F) = o(T)$, then $\ee\left[\reg_{CB}(T)\right] = o(T)$ as well with only $1$ call to the ERM oracle per round.
\end{theorem}
To prove Theorem \ref{thm:cbregret}, we observe that if the contexts arrive in a $\sigma$-smooth manner with respect to $\mu$, then the context-action pairs can be taken to be $\frac \sigma K$-smooth with respect to $\mu \otimes \text{Unif}([K])$.  We then apply \cite[Theorem 1]{foster2020beyond}.  The details and precise rates in the case of the FTPL instantiation can be found in Appendix \ref{app:cbregret}.

Note that the regret bound does not have optimal rates with respect to either $T$ or $K$.  In order to recover optimal rates with respect to $T$, we would need to find an algorithm that exhibits fast rates with square loss in the smoothed-online setting.  This is an interesting further direction in its own right, in addition to the practical implications on better rates for efficient algorithms for contextual bandits.


\section{Related Work}\label{app:relatedwork}
Here we describe some recent of the recent literature and how it relates to our work.

\textbf{Smoothed Analysis.}  Smoothed analysis was first introduced in \cite{spielman2004nearly}, where it was proposed as an explanation for the gap between theoretical lower bounds and excellent empirical performance of the 
simplex algorithm \citep{klee1972good}. Since then, smoothed analysis has been applied to analyze the performance of algorithms for many other problems which are known to be hard in the worst-case, such as the $k$-means algorithm for clustering \citep{arthur2006worstcase}, the flip algorithm for finding a local max-cut \citep{etscheid2017smoothed}, and more generally better-response algorithms for finding Nash equilibria in network coordination games \citep{boodaghians2020smoothed} (see also \cite{roughgarden2021beyond} for a more comprehensive overview). 

In the context of learning theory, \cite{rakhlin2011online} gave a nonconstructive proof demonstrating its utility for the specific case of threshold functions, while \cite{haghtalab2020smoothed,haghtalab2021smoothed} proved that the minimax regret of binary classification in smoothed online learning is governed by VC dimension.

\textbf{Online Learning.}
The optimal statistical rates attainable by online learning algorithms was shown to be characterized by sequential complexity measures of the function class in \cite{rakhlin2015sequential,pmlr-v30-Rakhlin13,rakhlin2015online}. This characterization was extended to the case of constrained adversaries (including the special case of smoothed adversaries) in \cite{rakhlin2011online}. Several subsequent papers have established further refined regret bounds \cite{pmlr-v134-block21a,rakhlin2015square}.  The profusion of publications relating to algorithmic questions about online learning is too large to enumerate here, but notable relevant work includes \cite{hazan2016computational}, which provides lower bounds on oracle-efficiency and \cite{rakhlin2012relax} which introduces a general framework for constructing algorithms.

\textbf{Follow The Perturbed Leader.} Our proper learning algorithm is motivated by Follow the Perturbed Leader (FTPL) \cite{kalai2005efficient,hannan20164}.  FTPL has been successful for many problems, including learning from experts \citep{kalai2005efficient}, multi-armed bandits \citep{abernethy2015fighting}, and online structured learning \citep{cohen2015following}, which includes as special cases problems such as online shortest path \citep{takimoto2002path} and online learning of permutations \citep{helmbold2007learning}.  There is a similar diversity in methods of proving regret bounds for FTPL style algorithms, including potential-based analysis \citep{abernethy2014online,cohen2015following} and relaxation methods \citep{rakhlin2011online}.  A common approach, which we adopt, is to show that the algorithm is stable \citep{kalai2005efficient,agarwal2019learning,devroye2013prediction,agarwal2011stochastic}.  One of the primary advantages of our FTPL approach is the fact that we do not generate independent noise for each function in our class.  In \citep{dudik2017oracle}, the authors present an FTPL-style algorithm which aims to do something similar, mitigating the computational burden by sharing randomness between functions.  Their method, however, is very different from ours in that they rely on their new notions of \emph{admissability} and \emph{implementability} of a matrix to transform low-dimensional independent noise into a more structured form; in contradistinction, we directly use a Gaussian Process on the function class to ensure stability of our algorithm.

\textbf{Contextual Bandits.} There is a rich history of studying contextual bandits.  Most relevant to our work is the series of papers \cite{foster2020beyond,simchi2021bypassing,foster2020instance} which provides a reduction from contextual bandits to an online learning oracle.  See these papers for further references.



\section{Proofs from Section \ref{sec:minimaxvalue}}\label{app:minimaxvalue}
\subsection{Proofs Related to the Coupling}

We first extend \cite[Theorem 2.1]{haghtalab2021smoothed} by providing a simpler and more general proof of the coupling between $\sD \in \p_T(\sigma, \mu)$ and independent random variables drawn according to $\mu$.  While we use a slightly different version (Lemma \ref{lem:coupling2}) in the proof of Theorem \ref{thm:chainingupperbound}, the following lemma is both simpler for exposition and is used in the proofs of the results in Section \ref{sec:relax}.
\begin{lemma}\label{lem:coupling}
    Suppose that $\sD \in \ps{T}{\sigma}{\mu}$.  Then for any $T$ there exists a measure $\Pi$ with random variables $(x_t, Z_t^j)_{\substack{1 \leq t \leq T \\ 1 \leq j \leq k}}$ satisfying the following properties:
    \begin{enumerate}
        \item $x_t$ is distributed according to $p_t(\cdot | x_1, \dots, x_{t-1})$ induced by $\sD$.
        \item $\{Z_t^j\}_{\substack{1 \leq t \leq T \\ 1 \leq j \leq k}}$ are iid according to $\mu$
        \item With probability at least $1 - Te^{-\sigma k}$, we have $x_t \in \{Z_t^j\}_{1 \leq j \leq k}$ for all $t$
    \end{enumerate}
\end{lemma}
\begin{proof}
    We construct the coupling recursively.  For any $t$, suppose that $Z_s^j, x_s$ has been constructed for $s < t$.  If $t = 0$ then this is the empty set.  Now, sample $Z_t^j$ iid according to $\mu$.  Let $\pi_t^j = \sigma \frac{d p_t}{d \mu}(Z_t^j)$.  Note that $\pi_t^j \leq 1$ by the assuption of $\sigma$-smoothness.  Construct the random set $S_t \subset [k]$ by adding $j$ to $S_t$ iwth probability $\pi_t^j$ independently for each $1 \leq j \leq k$.  If $S_t$ is nonempty, then sample $x_t$ uniformly from $S_t$.  Otherwise, sample $x_t$ independently from $p_t$.  We now show that this process exhibits the desired properties.

    It is clear form the construction that $Z_t^j$ are iid according to $\mu$.  To verify that $x_t \in \{Z_t^j\}$, we note that for any $t, j$, we have
    \begin{align}\label{eq:coupling1}
        \pp(Z_t^j \in S_t) = \ee_\mu\left[\sigma \frac{d p_t}{d \mu}(Z_t^j)\right] = \sigma
    \end{align}
    Because the $Z_t^j$ are added to $S_t$ independently, the probability that $S_t$ is empty is $(1 - \sigma)^k$.  Thus, by a union bound, the probability that there exists some $t \leq T$ such that any $S_t$ is empty is bounded by $T (1 - \sigma)^k \leq T e^{- \sigma k}$.

    Finally, to see that $x_t$ are distributed according to $p_t$, let $A \subset \cX$ be measurable and $\chi_A$ denote the indicator for $A$.  We compute:
    \begin{align}
        \pp(Z_t^j \in A | Z_t^j \in S_t) &= \frac{\pp\left(Z_t^j \in A \text{ and } Z_t^j \in S_t\right)}{\pp\left(Z_t^j \in S_t\right)} \\
        &= \frac{\pp\left(Z_t^j \in A \text{ and } Z_t^j \in S_t\right)}{\sigma} \\
        &= \frac{\ee_\mu\left[\chi_A \sigma \frac{d p_t}{d \mu}(Z_t^j)\right]}{\sigma} \\
        &= \ee_{p_t}\left[\mathbf{1}_A\right] = p_t(A)
    \end{align}
    where the second equality follows from \eqref{eq:coupling1}, the third equality follows from the construction of $S_t$, and the penultimate equality following from the definnition of the Radon-Nikodym derivative.  The result follows.
\end{proof}
We further note that the coupling in Lemma \ref{lem:coupling} is optimal with respect to the dependence on $k$ in the third requirement, as seen in the following proposition.
\begin{proposition}\label{prop:couplinglowerbound}
    For any $\sigma < 1$ and non-atomic measure $\mu$ on $\cX$, there exists a measure on $\cX$, $p$ such that $p$ is $\sigma$-smooth with respect to $\mu$ and the following property holds.  For any coupling $\Pi$ which has random variables $Z^j$ for $1 \leq j \leq k$ and $X$ such that $Z^j \sim \mu$ are independent and $X \sim p$, the probability that $X \in \{Z^j\}$ is bounded below by $1 - (1 - \sigma)^k$.
\end{proposition}
\begin{proof}
    Given $\mu$, let $A \subset \cX$ denote a measurable set such that $\mu(A) = \sigma$.  Let $p$ be a measure on $\cX$ such that $\frac{dp}{d \mu} = \frac 1\sigma \chi_A$.  Then $p$ is $\sigma$-smooth with respect to $\mu$.  Note that if $Z^j \sim \mu$ then with probability $1 - \sigma$, $Z^j \not\in A$.  Thus with probability $(1 - \sigma)^k$, none of $Z^j$ are in $A$.  Thus with probability at least $(1 - \sigma)^k$, $X \not\in \{Z^j| 1 \leq j \leq k\}$ if $X \sim p$.  The result follows.
\end{proof}

\subsection{Preliminaries on Distribution-Dependent Sequential Rademacher Complexity}
In this section, we recall the definition of the distribution-dependent sequential Rademacher complexity from \cite{rakhlin2011online} and how it relates to the minimax regret.  To begin, we formally construct a measure $\rho_{\sD}$ used in the definition of distribution-dependent sequential Rademacher complexity from \cite{rakhlin2011online}.

Throughout, we follow \cite{rakhlin2015sequential,rakhlin2011online} and introduce as our basic object in analyzing sequential complexities a tree.  Specifically, we consider complete binary trees $\z$ of depth $T$ with each vertex of the tree labelled by some element of $\cX$.  We associate each $\epsilon \in \{\pm 1\}^T$ to a path in the tree from the root to a leaf, where the path is constructed recursively by beginning at the root and at each level going to the left if $\epsilon_{t-1} = 1$ and to the right otherwise.  For a given tree $\z$, we denote by $\z_t(\epsilon)$ the label of the $t^{th}$ vertex along the path $\epsilon$.

Let $\sD$ be the joint distribution of $z_1, \dots, z_T \in \MZ$.  Define $p_t(\cdot, | z_1, \dots, z_{t-1})$ as the distribution under $\sD$ of $z_t$, given $z_s$ for $ s < t $.  We recursively construct the measure $\rho_{\sD}$ on pairs of binary trees as follows.  We first construct the roots of each tree by sampling $\z_0(\epsilon), \z_0'(\epsilon) \sim p_0$ independently.  Suppose we have $\z_{1:t-1}, \z'_{1:t-1}$ already constructed.  For any $s < t$, let
\begin{equation}
    \chi_s(\epsilon) = \begin{cases}
        \z_s(\epsilon) & \epsilon_s = 1 \\

        \z_s'(\epsilon) & \epsilon_s = -1
    \end{cases}
\end{equation}
then sample $\z_t(\epsilon), \z_t'(\epsilon)$ independently from $p_t(\cdot | \chi_1(\epsilon), \dots, \chi_{t-1}(\epsilon))$.  In this way, we can recursively construct the measure $\rho_{\sD}$.

With the definition of $\rho_{\sD}$ completed, we can now define the key notion of complexity.
\begin{definition}[Definition 2 from \cite{rakhlin2011online}]
    \label{def:distr-rc}
      Given a space $\MZ$, a function class $\F \subset [-1,1]^\MZ$, and a joint distribution $\sD$, let $\rho_\sD$ be the measure on an ordered pair of binary trees of depth $T$ with values in $\MZ$, defined above.  Then, we define the distribution-dependent sequential Rademacher complexities as
      \begin{align}
          \R_T^{seq}(\F, \sD) = \ee_{(\z, \z') \sim \rho_\sD} \ee_\epsilon\left[\sup_{f \in \F} \sum_{t = 1}^T \epsilon_t f(\z_t(\epsilon))\right] 
      \end{align}
      If $\p$ is a class of distributions $\sD$, we define
      \begin{equation}
        \R_T^{seq}(\F, \p) = \sup_{\sD \in \p} \R_T^{seq}(\F, \sD)
      \end{equation}
      for any class of distributions $\p$.
  \end{definition}
  Intuitively, depending on the nature of the class $\p$, $\R_T^{seq}(\F, \p)$ interpolates between the classical batch Rademacher complexity (if we force $\sD$ to be iid) and the fully adversarial sequential Rademacher complexity from \cite{rakhlin2015sequential}.  In the special case that $\p = \p(\sigma, \mu)$, we see that we are much closer to the classical Rademacher complexity than to the fully adversarial analogue.  Indeed, using Lemma \ref{lem:coupling2}, which is an extension of the coupling result contained in Lemma \ref{lem:coupling} above, we can bound the distribution-dependent sequential Rademacher complexity by that of the classical Rademacher complexity:
  \begin{lemma}
      Let $\F \subset [-1,1]^{\cX}$ be a function class.  Then, for any $k \in \mathbb{N}$,
      \begin{equation}
          \R_T^{seq}(\F, \p(\sigma, \mu)) \leq \left(\frac 4\sigma \log T\right) \ee_\mu\left[\R_{kT}(\F)\right] + 2 T^2 e^{- \sigma k}
      \end{equation}
      In particular, in the case that $\vc(\F) \leq d$, we have:
      \begin{equation}
        \R_T^{seq}(\F, \p(\sigma, \mu)) \lesssim \sqrt{T d \log\left(\frac T\sigma\right)}
      \end{equation}
      and in the case that $\vc(\F, \delta) \lesssim \delta^{-p}$, we have:
      \begin{equation}
        \R_T^{seq}(\F, \p(\sigma, \mu)) \lesssim \left(T \log\left(\frac T\sigma\right)\right)^{\max\left(\frac 12, 1 - \frac 1p\right)}
      \end{equation}
  \end{lemma}
  \begin{proof}
      Let $A$ be the high probability event in Lemma \ref{lem:coupling2} below, i.e., the event that $x_t \in \left\{Z_t^j\right\}_{1 \leq j \leq k}$ for all $t$.  We have for any $\sD \in \p(\sigma, \mu)$,
      \begin{align}
          \R_T^{seq}(\F, \sD) &= \ee_{(\z, \z') \sim \rho_{\sD}} \ee_\epsilon\left[\sup_{f \in \F} \sum_{t = 1}^T \epsilon_t f(\z_t(\epsilon))\right] \\
          &= \ee_{\Pi} \ee_\epsilon\left[\chi_A\sup_{f \in \F} \sum_{t = 1}^T \epsilon_t f(\z_t(\epsilon))\right] + \ee_{\Pi} \ee_\epsilon\left[\chi_{A^c}\sup_{f \in \F} \sum_{t = 1}^T \epsilon_t f(\z_t(\epsilon))\right] \\
          &\leq \ee_{\Pi} \ee_\epsilon\left[\chi_A\sup_{f \in \F} \sum_{t = 1}^T \epsilon_t f(\z_t(\epsilon)) + \sum_{j \text{ such that } Z_t^j \neq \z_t(\epsilon)} \ee_{\epsilon_{t,j}}\left[\epsilon_{t,j} f(Z_t^j)\right]\right] + 2 T^{2} e^{-\sigma k} \\
          &\leq \ee_{\Pi} \ee_\epsilon\left[\sup_{f \in \F} \sum_{j = 1}^k\sum_{t = 1}^T \epsilon_{t,j} f(Z_t^j)\right] + 2 T^2 e^{- \sigma k} \\
          &\leq 2 T^2 e^{-\sigma k} + \ee_\mu\left[\R_{kT}(\F)\right]
      \end{align}
      where $\Pi$ is the coupling in Lemma \ref{lem:coupling2}, the first inequality follows because $\epsilon_t$ is mean zero, the second inequality follows by Jensen's, and the lastfollows by definition of Rademacher complexity.  Setting $k = \frac{2}{\sigma} \log T$ concludes the proof.
  \end{proof}
  As is the case in both the fully adversarial and classical regimes, we see that $\V_T^{prop}(\F, \p)$ is determined up to constants by $\R_T^{seq}(\F, \p)$:
    \begin{proposition}[Theorem 3 and Lemma 20 from \cite{rakhlin2011online}]\label{prop:rakhlinrademacher}
      For any $\F$, we have
      \begin{equation}
          \V_T^\prop(\F, \p(\sigma, \mu)) \leq 2 \sup_{\sD \in \tps{T}{\sigma}{\mu}} \R_T^{seq}(\ell \circ \F, \sD).
      \end{equation}
      where we recall from Definition \ref{def:smooth} that $\tps{T}{\sigma}{\mu}$ is the class of distributions on $(x_t, y_t)$ such that the $x_t$ are chosen in a $\sigma$-smooth way and the $y_t$ are adversarial.  In the special case where $\ell$ is absolute loss, we also have
      \begin{equation}
        \sup_{\sD \in \tps{T}{\sigma}{\mu}} \R_T^{seq}(\ell \circ \F, \sD) \leq \V_T^\prop(\F, \p(\sigma, \mu))
      \end{equation}
    \end{proposition}
In the statement of Proposition \ref{prop:rakhlinrademacher}, $\ell \circ \MF$ denotes the class of functions in $[0,1]^{\MX \times \YY}$ of the form $(x,y) \mapsto \ell(f(x), y)$, for $f \in \MF$. 
By Proposition \ref{prop:rakhlinrademacher}, it suffices to provide upper and lower bounds on $\R_T^{seq}(\ell \circ \F, \tps{T}{\sigma}{\mu})$, which is significantly more tractable than working with the iterated operators involved in $\V_T^\prop(\F, \p)$.  

In the proof below, we will also need a sequential analogue of $\vc(\F, \alpha)$:
\begin{definition}[Definition 7 from \cite{rakhlin2015sequential}] \label{def:seqfat}
    We say that a $\cX$-valued binary tree of depth $T$, $\x$, is shattered by $\F$ at scale $\delta \geq 0$ if there exists an $\mathbb{R}$-valued binary tree $\mathbf{s}$ of depth $T$ such that for all $\epsilon \in \{\pm 1\}^T$, there exists an $f_\epsilon \in \F$ such that
    \begin{equation}
        \epsilon_t\left(f(\x_t(\epsilon)) - \mathbf{s}_t(\epsilon)\right) \geq \frac \alpha 2
    \end{equation}
    Define the sequential fat-shattering dimension of $\F$, $\fat_\delta(\F)$ as the maximal $T$ such that there exists a tree of depth $T$ shattering $\F$ at scale $\delta$.
\end{definition}

Finally, we require a structural result showing that worst-case sequential Rademacher complexity contracts with Lipschitz loss functions:
\begin{lemma}[Lemma 13 from \cite{rakhlin2015sequential}] \label{lem:contraction}
    Let $\F$ be a function class with values in $[-1,1]$ and let $\ell$ be $L$-Lipschitz.  Then,
    \begin{equation}
        \sup_{\sD \in \Delta\left(\cX^{\times T}\right)} \R_T^{seq}(\ell \circ \F, \sD) \lesssim L \log^{\frac 32}(T) \sup_{\sD \in \Delta\left(\cX^{\times T}\right)} \R_T^{seq}(\F, \sD)
    \end{equation}
\end{lemma}
Because the supremum in Lemma \ref{lem:contraction} is taken over all distributions on the product space $\cX^{\times T}$, the above distribution-dependent sequential Rademacher complexities are reduced to the adversarial sequential Rademacher complexities of \cite{rakhlin2015sequential}.  In the following section, we show that on small domains, we can control $\fat_\delta(\F)$ by $\vc(\F, \delta)$.

\subsection{Sequential and Batch Complexities}
In this section, we prove the following lemma, which bounds the sequential fat-shattering dimension by the scale-sensitive VC dimension when the domain is small:
\begin{lemma}\label{lem:finitedomain}
    Let $\F$ be a function class from $\cX$ to $[-1,1]$ and let $\fat_\delta(\F)$ denote the sequential fat-shattering dimension of $\F$ (Definition \ref{def:seqfat}).  Then, for any $\alpha > 0$,
    \begin{equation}
        \fat_\delta(\F) \lesssim \vc\left(\F, c \alpha \delta\right) \log^{1 + \alpha}\left(\frac{C\abs{\cX}}{\vc\left(\F, c \delta\right) \delta}\right)
    \end{equation}
\end{lemma}
In order to prove this result, we require a generalization of the Sauer-Shelah lemma \citep{sauer1972density,shelah1972combinatorial}.  We first define covering numbers with respect to the $\sup$ norm:
\begin{definition}
    Let $\F$ be a class of functions on $\cX$.  A set $S$ of functions on $\cX$ is a $\delta$ covering if for all $f \in \F$, there exists a $s_f \in S$ such that
    \begin{equation}
        \sup_{x \in \cX} \abs{s_f(x) - f(x)} \leq \delta
    \end{equation}
    We let $N(\F, \delta)$ to be the minimal size of a $\delta$-covering of $\F$.
\end{definition}
In order to bound $\fat_\delta(\F)$ by $\vc(\F, \delta)$, we first recall a result that bounds $N(\F, \delta)$ by $\vc(\F, \delta)$:
\begin{theorem}[Theorem 4.4 from \cite{rudelson2006combinatorics}]\label{thm:rudelson}
    Let $\F$ be a function class on $\cX$, a finite set, to $[-1,1]$.  Then for any $\alpha > 0$, there are constants $c, C > 0$ such that
    \begin{equation}
        \log N(\F, \delta) \lesssim \vc(\F, c \alpha \delta) \log^{1+\alpha}\left(\frac{C \abs{\cX}}{\vc(\F, c \delta) \delta}\right)
    \end{equation}
\end{theorem}
The above theorem is an intermediate result, so our bound will come down to comparing $\fat_\delta(\F)$ to the covering numbers.  We can now provide the main proof in the section.
\begin{proof}[(Lemma \ref{lem:finitedomain})]
    We first note that $2^{\fat_\delta(\F)} \leq N\left(\F, \frac \delta 3\right)$.  To see this, let $d = \fat_\delta(\F)$ and let $\x$ denote a depth $d$ tree that shatters $\F$ at scale $\delta$ with witness tree $\mathbf{s}$.  Let $S$ be a $\frac{\delta}{2}$ net for $\F$.  For each $\epsilon \in \{\pm 1\}^d$, let $f_\epsilon$ be the function that realizes the shattering on path $\epsilon$.  If $v_{f_\epsilon} \in \F$ is the projection into $S$, then we note that the function $\epsilon \mapsto v_{f_\epsilon}$ is injective.  Indeed, if there are two different $\epsilon, \epsilon'$ mapping to the same $v \in S$, then there is some $t$ such that $\epsilon_t = - \epsilon_t'$ but $\epsilon_s = \epsilon_s'$ for $s < t$.  Thus $\x_t(\epsilon) = \x_t(\epsilon')$ and $\mathbf{s}_t(\epsilon) = \mathbf{s}_t(\epsilon')$.  We know, however, that
    \begin{equation}
        \abs{f_{\epsilon}(\x_t(\epsilon)) - f_{\epsilon'}(\x_t(\epsilon'))} \leq \abs{f_\epsilon(\x_t(\epsilon)) - v(\x_t(\epsilon))} + \abs{v(\x_t(\epsilon)) - f_{\epsilon'}(\x_t(\epsilon))} \leq \frac{2 \delta}{3}
    \end{equation}
    Thus we have by the shattering assumption,
    \begin{align}
        \frac \delta 2 &\leq \epsilon_t\left(f_\epsilon(\x_t(\epsilon) - \mathbf{s}_t(\x_t(\epsilon)))\right) \\
        &= - \epsilon_t' \left(f_{\epsilon}(\x_t(\epsilon) - \mathbf{s}_t(\x_t(\epsilon))) \right) \\
        &\leq - \epsilon_t' \left(f_{\epsilon'}(\x_t(\epsilon) - \epsilon_t' \frac{2\delta}3 - \mathbf{s}_t(\x_t(\epsilon))) \right) \\
        &\leq - \frac \delta 2 + \frac{2 \delta}{3} \\
        &\leq \frac{\delta}{6} 
    \end{align}
    Thus we have a contradiction and the mapping is injective.  But this means then that $2^d \leq N\left(\F, \frac \delta 2\right)$ as desired.  By Theorem \ref{thm:rudelson}, for any $\alpha > 0$, we have
    \begin{align}
        \fat_\delta(\F) &\lesssim \log\left(2^{\fat_\delta(\F)}\right) \lesssim \log\left({N\left(\F, \frac \delta 3\right)}\right) \\
        &\lesssim \vc\left(\F, c \alpha \delta \right) \log^{1 + \alpha}\left(\frac{C\abs{\cX}}{\vc\left(\F,c \alpha \delta\right) \delta}\right)
    \end{align}
    as desired.
\end{proof}
We are now ready to prove the main results from Section \ref{sec:minimaxvalue}.
\subsection{Proof of Theorem \ref{thm:chainingupperbound}}
By Proposition \ref{prop:rakhlinrademacher}, it suffices to control the distribution-dependent sequential Rademacher complexity of Definition \ref{def:distr-rc}, specialized to the case that $\MZ = \MX \times \YY$, and $\sD \in \ps{T}{\sigma}{\mu}$. We first adapt Lemma \ref{lem:coupling} to construct a coupling with $\epsilon, \rho_{\sD}$ and independent samples from $\mu$; this will allow us to move from sequential Rademacher complexity to standard Rademacher complexity.  The lemma is again a variant of the coupling in \cite{haghtalab2021smoothed}, albeit simpler to describe and preserving independence between $\epsilon$ and $Z_t^j$.  We have the following lemma:
\begin{lemma}\label{lem:coupling2}
    Suppose that $\sD \in \tps{T}{\sigma}{\mu}$.  Then for any $T$ there exists a measure $\Pi$ with random variables $(\epsilon_{1:T}, \z(\epsilon), \z'(\epsilon), Z_t^j, Z_t^{j'})_{\substack{1 \leq t \leq T \\ 1 \leq j \leq k}}$ satisfying the following properties, where we write $\z(\ep) = (\bx(\ep), \by(\ep)), \z'(\ep) = (\bx'(\ep), \by'(\ep))$ to separate the $\MX$- and $\YY$-components of $\z(\ep) \in \MZ$:
    \begin{enumerate}
        \item $\epsilon_{1:T}$ are iid Rademacher random variables.
        \item $(\z, \z')$ is distributed according $\rho$.
        \item $\{Z_t^j, Z_t^{j'}\}$ are iid according to $\mu$
        \item $\{\epsilon, Z_t^j, Z_t^{j'}\}$ are independent
        \item With probability at least $1 - 2T(1 - \sigma)^{k}$, $\bx_t(\epsilon) \in \{Z_t^j\}_{1 \leq j \leq k}$ for all $t$
    \end{enumerate}
\end{lemma}
\begin{proof}
  Given $\sD$, let:
  \begin{itemize}
  \item  $p_t(\cdot | (x_1, y_1), \ldots, (x_{t-1}, y_{t-1}))$ denote the distribution of $x_t$ under $\sD$ given $(x_s, y_s)$ for $s < t$ (since $\sD \in \tps{T}{\sigma}{\mu}$, $p_t$ is $\sigma$-smooth with respect to $\mu$ a.s.);
  \item $q_t(\cdot | (x_1, y_1), \ldots, (x_{t-1}, y_{t-1}), x_t)$ denote the distribution of $y_t$ under $\sD$ given $(x_s, y_s)$ for $s < t$ and $x_t$. 
  \end{itemize}
  
  We construct the coupling recursively.  For any $t$, suppose that $Z_s^j, Z_s^{j'}, \epsilon_s, \z_s(\epsilon)$ has been constructed for $s < t$.  If $t = 0$ then this is the empty set.  Now, sample $Z_t^j, Z_t^{j'}$ iid according to $\mu$ and $\epsilon_t$ a Rademacher random variable.  Let $\pi_t^j = \sigma \frac{d p_t}{d \mu}(Z_t^j)$.  Note that $\pi_t^j \leq 1$ by the assumption of $\sigma$-smoothness.  As in the proof of Lemma \ref{lem:coupling}, construct the random set $S_t$ by adding each $Z_t^j$ to $S_t$ with probability $\pi_t^j$. If $S_t$ is nonempty, then sample $\bx_t(\ep)$ independently from $S_t$ uniformly at random; if $S_t$ is empty, sample $\bx_t(\ep)$ from $p_t(\cdot | \z_1(\ep), \ldots, \z_{t-1}(\ep))$. 
  Then sample $\by_t(\ep)$ independently from $q_t(\cdot | \bz_1(\ep), \ldots, \bz_{t-1}(\ep), \bx_t(\ep))$, and set $\bz_t(\ep) = (\bx_t(\ep), \by_t(\ep))$.
  
  We may construct the $\z_t'(\epsilon)$ similarly by constructing a set $S_t'$ in the same way, using $Z_t^{j'}$ instead of $Z_t^j$. Finally sample $\ep_t$ independently. 

    It is clear from the construction that $\epsilon_t$, $t \in [T]$, are independent Rademacher random variables.  Similarly, it is clear that $Z_t^j, Z_t^{j'}$ are iid according to $\mu$ and independent of $\epsilon$.  The remainder of the properties are proved in the same way as in Lemma \ref{lem:coupling}.
\end{proof}
We are now ready to prove the theorem.
\begin{proof}(Theorem \ref{thm:chainingupperbound})
    By Theorem \ref{prop:rakhlinrademacher}, it suffices to bound the distribution-dependent sequential Rademacher complexity.  Let $\Pi$ be the coupling in Lemma \ref{lem:coupling2} and let $A$ denote the event that $\bx_t(\epsilon) \in \{Z_t^j\}_{1 \leq j \leq k}$ for all $t$ (we continue to write $\z_t(\ep) = (\bx_t(\ep), \by_t(\ep))$).  Then we compute for a fixed $k$,
    \begin{align}
      &\sup_{\sD \in \p} \ee_{\rho_{\sD}, \epsilon}\left[\sup_{f \in \F} \sum_{t = 1}^T \epsilon_t \ell(f(\x_t(\epsilon)), \by_t(\ep)) \right] \\
      &\leq \ee_{\Pi, \epsilon}\left[\sup_{f \in \F} \sum_{t = 1}^T \epsilon_t \ell(f(\x_t(\epsilon)), \by_t(\ep)) \right] \\
        &\leq \ee_{\Pi, \epsilon}\left[\chi_{A^c}\sup_{f \in \F} \sum_{t = 1}^T \epsilon_t \ell(f(\x_t(\epsilon)), \by_t(\ep)) \right] + \ee_{\Pi, \epsilon}\left[\chi_A \sup_{f \in \F} \sum_{t = 1}^T \epsilon_t \ell(f(\x_t(\epsilon)),\by_t(\ep)) \right] \\
        & \leq  2 T^2 e^{- \sigma k} + \ee_{\Pi, \epsilon}\left[\chi_A\sup_{f \in \F} \sum_{t = 1}^T\epsilon_t \ell(f(\x_t(\epsilon)),\by_t(\ep)) \right] 
    \end{align}
    By the tower property of expectations, denoting by $\F|_{\{Z_t^j\}}$ the restriction of $\F$ to the set of all $Z_t^j$, we have
    \begin{align}
        \ee_{\Pi, \epsilon}\left[\chi_A\sup_{f \in \F} \sum_{t = 1}^T \epsilon_t \ell(f(\x_t(\epsilon)) ,\by_t(\ep))\right]  &= \ee_{Z_t^j \stackrel{iid}{\sim} \mu}\left[\ee_\Pi\left[\chi_A\sup_{f \in \F} \sum_{t = 1}^T \epsilon_t \ell(f(\x_t(\epsilon)),\by_t(\ep)) \bigg|  \{Z_t^j\}_{\substack{1 \leq t \leq T \\ 1 \leq j \leq k}} \right] \right] \\
        &= \ee_{Z_t^j \stackrel{iid}{\sim} \mu}\left[\ee_\Pi\left[\chi_A\sup_{f \in \F|_{\{Z_t^j\}}} \sum_{t = 1}^T \epsilon_t \ell(f(\x_t(\epsilon)),\by_t(\ep)) \bigg|  \{Z_t^j\}_{\substack{1 \leq t \leq T \\ 1 \leq j \leq k}} \right] \right] \\
        \lesssim L \log^{\frac 32}(T) &\sup_{\left\{Z_t^j\right\}_{\substack{1 \leq t \leq T \\ 1 \leq j \leq k}}} \sup_{\x} \ee_\epsilon\left[\sup_{f \in \F|_{\{Z_t^j\}}} \sum_{t = 1}^T \epsilon_t f(\x_t(\epsilon)) \right]
    \end{align}
    where the last inner supremum is over all $\x$ such that $\x$ is a $\{Z_t^j\}$-labelled binary tree of depth $T$; the last inequality follows, then, from Lemma \ref{lem:contraction}.  Let $\fat_\delta(\F)$ denote the sequential fat-shattering dimension in Definition \ref{def:seqfat}.  We may apply \cite[Corollary 10 and Proposition 15]{pmlr-v134-block21a}, which bounds the worst-case sequential Rademacher complexity by the sequential fat-shattering dimension to get
    \begin{align}
        \sup_{\left\{Z_t^j\right\}_{\substack{1 \leq t \leq T \\ 1 \leq j \leq k}}} \sup_{\x} \ee_\epsilon\left[\sup_{f \in \F|_{\{Z_t^j\}}} \sum_{t = 1}^T \epsilon_t f(\x_t(\epsilon)) \right] \lesssim \sup_{\left\{Z_t^j\right\}_{\substack{1 \leq t \leq T \\ 1 \leq j \leq k}}}\inf_{\alpha > 0}\left\{\alpha T + \sqrt{T} \int_{\alpha}^1 \sqrt{\fat_\delta\left(\F_{\{Z_t^j\} }\right)} d \delta    \right\}
    \end{align}
    By Lemma \ref{lem:finitedomain}, we have for any $\beta > 0$,
    \begin{equation}
        \fat_\delta\left(\F_{\{Z_t^j\} }\right) \lesssim \vc(\F, c\beta \delta) \log^{1 + \beta}\left(\frac{C\abs{\{Z_t^j\}_{\substack{1 \leq t \leq T \\ 1 \leq j \leq k}}}}{\vc\left(\F, c \beta \delta\right) \delta} \right) \lesssim \vc(\F, c \beta \delta) \log^{1 + \beta}\left(\frac{C k T}{\vc\left(\F, c \beta \delta\right)\delta}\right)
    \end{equation}
    independent of the realization of $Z_t^j$.  Thus we have
    \begin{align}
        \ee_{Z_t^j \stackrel{iid}{\sim} \mu}&\left[\ee_\Pi\left[\chi_A\sup_{f \in \F|_{\{Z_t^j\}}} \sum_{t = 1}^T \epsilon_t \ell(f(\x_t(\epsilon)),\by_t(\ep)) \bigg|  \{Z_t^j\}_{\substack{1 \leq t \leq T \\ 1 \leq j \leq k}} \right] \right] \\
        &\lesssim L \log^{\frac 32}(T)\inf_{\alpha > 0} \left\{\alpha T + \sqrt{T} \int_{\alpha}^1 \sqrt{\vc(\F, c \beta \delta) \log^{1 + \beta}\left(\frac{C k T}{\vc\left(\F, c \beta \delta \right)\delta}\right))} d \delta \right\}
    \end{align}
    Putting everything together, we have
    \begin{align}
        \sup_{\sD \in \p} &\ee_{\rho_{\sD}, \epsilon}\left[\sup_{f \in \F} \sum_{t = 1}^T \epsilon_t \ell(f(\x_t(\epsilon)), \by_t(\ep)) \right] \\
        &\lesssim 2 T^2 e^{- \sigma k} + L \log^{\frac 32}(T)\inf_{\alpha > 0} \left\{\alpha T + \sqrt{T \log^{1+\beta}\left(\frac{3 k T}{\vc\left(\F,  c \beta \alpha\right)\alpha}\right)} \int_{\alpha}^1 \sqrt{\vc(\F, c \beta \delta)} d \delta \right\}
    \end{align}
    Setting $k = \frac{2\log T}{\sigma}$ and $\beta = 1$ concludes the proof.

\end{proof}

\subsection{Proof of Proposition \ref{prop:unknownmu}}
Let $\F$ be the class of thresholds on the unit interval, i.e.,
\begin{equation}
    \F = \left\{x \mapsto \sign(x - \theta) | \theta \in [0,1]\right\}
\end{equation}
It is well-known that $\vc(\F) = 1$.  Consider an adversary that sets $x_1 = 0$, $x_2 = 1$, $y_1 = -1$, $y_2 = 1$ and for all $t > 2$, sets
\begin{align}
    x_t = x_{t-1} - y_{t-1} 2^{-(t - 2)}
\end{align}
and the $y_t$ are independent Rademacher random variables.  Note that by construction, the adversary is realizable with respect to $\F$ in the sense that for any realization of the $(x_t, y_t)$, there is some $f \in \F$ with $f(x_t) = y_t$ for all $t$.  Also by construction, we see that the expected number of mistakes in $T$ rounds is $\frac T2$.  For fixed $T$, let
\begin{equation}
    \mathscr P_T = \left\{\frac 1T \sum_{t = 1}^T \delta_{x_t} \right\}
\end{equation}
be the set of empirical distributions generated by the contexts over all realizations of $x_1, \dots, x_T$.  Note that for each $\mu \in \mathscr P_T$, the support has size $T$ and thus the adversary constructed above is $\left(\frac 1T\right)$-smooth with respect to some $\mu \in \mathscr P_T$.  The result follows by noting that if $T \leq \frac 1\sigma$ then the adversary is $\sigma$-smooth with respect to some $\mu \in \mathscr P_T$.


\section{Proofs from Section \ref{sec:relax}}\label{app:relax}
\subsection{Proofs Related to Relaxations}
\begin{proof}(Proposition \ref{prop:slowrelaxation})
It suffices to prove \eqref{eq:relaxation} as the other property follows immediately from the construction.  For the sake of convenience, we denote
\begin{equation}
    L_t(f) = \sum_{s = 1}^t \ell(f(x_t), y_t)
\end{equation}
We begin by noting that \cite[Lemma 5.1]{rakhlin2014statistical} tells us that due to the convexity of $\ell$ in the first argument, it suffices to replace distributions $q_t$ over $[-1,1]$ with values $\yhat_t \in [-1,1]$.  In particular,
\begin{align}
    \inf_{q_t \in \Delta([-1,1])} &\sup_{y_t \in [-1,1]} \ee_{q_t}\left[\ell(\yhat_t, y_t)\right] + \ee_{\mu, \epsilon}\left[\sup_{f \in \F} 2  L\sum_{j = 1}^k \sum_{s = t+1}^T \epsilon_{s,j} f(x_{s,j}) - L_{t}(f)\right] \\
    &= \inf_{\yhat_t \in [-1,1]} \sup_{y_t \in [-1,1]} \ell(\yhat_t, y_t) + \ee_{\mu, \epsilon}\left[\sup_{f \in \F} 2  L\sum_{j = 1}^k \sum_{s = t+1}^T \epsilon_{s,j} f(x_{s,j}) - L_{t}(f)\right]
\end{align}
Now, omitting the feasible set for $\yhat_t, y_t$ to ease the notational load, we plug in our relaxation:
\begin{align}
    \inf_{\yhat_t} &\sup_{y_t} \ell(\yhat_t, y_t) + \ee_{\mu, \epsilon}\left[\sup_{f \in \F} 2  L\sum_{j = 1}^k \sum_{s = t+1}^T \epsilon_{s,j} f(x_{s,j}) - L_{t}(f)\right] \\
    &= \inf_{\yhat_t} \sup_{y_t} \ee_{\mu, \epsilon} \left[\sup_{f \in \F} 2  L\sum_{j = 1}^k \sum_{s = t+1}^T \epsilon_{s,j} f(x_{s,j}) - L_{t-1}(f) + \ell(\yhat_t, y_t) - \ell(f(x_t), y_t) \right]  \\
    &\leq \inf_{\yhat_t} \sup_{y_t} \ee_{\mu, \epsilon} \left[\sup_{f \in \F} 2  L\sum_{j = 1}^k \sum_{s = t+1}^T \epsilon_{s,j} f(x_{s,j}) - L_{t-1}(f) + \partial \ell(\yhat_t, y_t)(\yhat_t - f(x_t))] \right] \\
    &\leq \inf_{\yhat_t} \sup_{y_t} \sup_{g_t \in [-L, L]} \ee_{\mu, \epsilon} \left[\sup_{f \in \F} 2  L\sum_{j = 1}^k \sum_{s = t+1}^T \epsilon_{s,j} f(x_{s,j}) - L_{t-1}(f) + g_t (\yhat_t - f(x_t)) \right] \\
    &= \inf_{\yhat_t} \max_{g_t \in \{-L, L\} }\ee_{\mu, \epsilon}\left[g_t \yhat_t  + \sup_{f \in \F} 2  L\sum_{j = 1}^k \sum_{s = t+1}^T \epsilon_{s,j} f(x_{s,j}) - L_{t-1}(f) - g_t f(x_t) \right]
\end{align}
where we let $\partial\ell$ denote a subgradient of $\ell$ with respect to the first argument.  The first inequality follows by convexity, the second inequality follows by Lipschitzness, and the last equality follows because the inner expectation is convex as a function of $g_t$ and so obtains its maximum on the boundary.  Let $d_t$ denote a distribution on $\{-L, L\}$; the $y_t$ vanishes because it only appeared in the $\partial \ell(\yhat_t, y_t)$ and this was bounded by $g_t$.  Then by the minimax theorem, we have
\begin{align}
    \inf_{\yhat_t} &\max_{g_t \in \{-L, L\} }\ee_{\mu, \epsilon}\left[g_t \yhat_t  + \sup_{f \in \F} 2  L\sum_{j = 1}^k \sum_{s = t+1}^T \epsilon_{s,j} f(x_{s,j}) - L_{t-1}(f) - g_t f(x_t) \right] \\
    &= \sup_{d_t} \inf_{\yhat_t} \ee_{g_t \sim d_t} \ee_{\mu, \epsilon}\left[ g_t \yhat_t + \sup_{f \in \F} 2  L\sum_{j = 1}^k \sum_{s = t+1}^T \epsilon_{s,j} f(x_{s,j}) - L_{t-1}(f) - g_t f(x_t)\right] \\
    &\leq \sup_{d_t} \ee_{g_t \sim d_t} \ee_{\mu, \epsilon}\left[\inf_{\yhat_t} \ee_{g_t' \sim d_t}[g_t' \yhat_t] + \sup_{f \in \F} 2  L\sum_{j = 1}^k \sum_{s = t+1}^T \epsilon_{s,j} f(x_{s,j}) - L_{t-1}(f) - g_t f(x_t) \right] \\
    &\leq \sup_{d_t} \ee_{\mu, \epsilon} \ee_{g_t \sim d_t}\left[\sup_{f \in \F} 2  L\sum_{j = 1}^k \sum_{s = t+1}^T \epsilon_{s,j} f(x_{s,j}) - L_{t-1}(f) + (\ee_{g_t' \sim d_t}[g_t'] - g_t) f(x_t)\right] \\
    &\leq \sup_{d_t} \ee_{\mu, \epsilon} \ee_{g_t, g_t' \sim d_t}\left[\sup_{f \in \F} 2  L\sum_{j = 1}^k \sum_{s = t+1}^T \epsilon_{s,j} f(x_{s,j}) - L_{t-1}(f) + \epsilon_t (g_t' - g_t) f(x_t)\right] \\
    &\leq \sup_{d_t} \ee_{\mu, \epsilon} \ee_{g_t \sim d_t}\left[\sup_{f \in \F} 2  L\sum_{j = 1}^k \sum_{s = t+1}^T \epsilon_{s,j} f(x_{s,j}) - L_{t-1}(f) + 2 \epsilon_t g_t  f(x_t)\right] \\
    &= \ee_{\mu, \epsilon} \left[\sup_{f \in \F} 2  L\sum_{j = 1}^k \sum_{s = t+1}^T \epsilon_{s,j} f(x_{s,j}) - L_{t-1}(f) + 2 L \epsilon_t f(x_t)\right]
\end{align}
Now, it would seem that we are done, but note that $x_t \sim p_t$ while $x_s \sim \mu$ for $s > t$.  We thus apply $\sup_{p_t \in \p_t} \ee_{x_t \sim p_t}$ to all of the preceding equations and, adding back in the additive constant, we have shown
\begin{align}
    \sup_{p_t \in \p} &\ee_{x_t \sim p_t}\inf_{q} \sup_{y_t} \left[\ee_{\yhat_t \sim q}[\ell(\yhat, y)] + \rel_T(\F| x_1, y_1, \dots, x_t, y_t)\right] \\
    &\leq \sup_{p_t \in \p} \ee_{x_t \sim p_t} \ee_{\mu, \epsilon}\left[\sup_{f \in \F} 2  L\sum_{j = 1}^k \sum_{s = t+1}^T \epsilon_{s,j} f(x_{s,j}) - L_{t-1}(f) + 2 L \epsilon_t f(x_t)\right] + (T - t)^3 e^{- \sigma k}
\end{align}
Now, applying the coupling  $\Pi$ from Lemma \ref{lem:coupling}, we have
\begin{align}
    \ee_{x_t \sim p_t} &\ee_{\mu, \epsilon} \left[\sup_{f \in \F} 2  L\sum_{j = 1}^k \sum_{s = t+1}^T \epsilon_{s,j} f(x_{s,j}) - L_{t-1}(f) + 2 L \epsilon_t f(x_t)\right] \\
    &= \ee_{x_t \sim \Pi} \ee_{\mu, \epsilon} \left[\sup_{f \in \F} 2  L\sum_{j = 1}^k \sum_{s = t+1}^T \epsilon_{s,j} f(x_{s,j}) - L_{t-1}(f) + 2 L \epsilon_t f(x_t)\right] \\
    &= \ee_{x_t \sim \Pi} \ee_{\mu, \epsilon} \left[\chi_{A^c}\sup_{f \in \F} 2  L\sum_{j = 1}^k \sum_{s = t+1}^T \epsilon_{s,j} f(x_{s,j}) - L_{t-1}(f) + 2 L \epsilon_t f(x_t)\right] \\
    &+ \ee_{x_t \sim \Pi} \ee_{\mu, \epsilon} \left[\chi_{A}\sup_{f \in \F} 2  L\sum_{j = 1}^k \sum_{s = t+1}^T \epsilon_{s,j} f(x_{s,j}) - L_{t-1}(f) + 2 L \epsilon_t f(x_t)\right] 
\end{align}
where $A$ is the event that $x_t \in \{Z_t^j\}$ for $1 \leq j \leq k$ and $\chi_A$ is the indicator.  For the first term, we have
\begin{align}
    \ee_{x_t \sim \Pi} &\ee_{\mu, \epsilon} \left[\chi_{A^c}\sup_{f \in \F} 2  L\sum_{j = 1}^k \sum_{s = t+1}^T \epsilon_{s,j} f(x_{s,j}) - L_{t-1}(f) + 2 L \epsilon_t f(x_t)\right] \\
    &\leq \pp(A^c) (n - t+1) \leq (n - t + 1)^2 e^{- \sigma k}
\end{align}
For the second term, we have
\begin{align}
    \ee_{x_t \sim \Pi} &\ee_{\mu, \epsilon} \left[\chi_{A}\sup_{f \in \F} 2  L\sum_{j = 1}^k \sum_{s = t+1}^T \epsilon_{s,j} f(x_{s,j}) - L_{t-1}(f) + 2 L \epsilon_t f(x_t)\right]  \\
    &\leq \ee_{\Pi, \mu, \epsilon} \chi_A \sup_{f \in \F}\left[2  L\sum_{j = 1}^k \sum_{s = t+1}^T \epsilon_{s,j} f(x_{s,j}) - L_{t-1}(f) + \sum_{j = 1}^k \epsilon_{t,j} f(Z_t^j) \right] \\
    &= \ee_{\mu, \epsilon}\left[\sup_{f \in \F} 2  L\sum_{j = 1}^k \sum_{s = t+1}^T \epsilon_{s,j} f(x_{s,j}) - L_{t-1}(f) \right]
\end{align}
Putting this back together, we have
\begin{align}
    \sup_{p_t \in \p_t} &\ee_{x_t \sim p_t}\inf_{q} \sup_{y_t} \left[\ee_{\yhat_t \sim q}[\ell(\yhat, y)] + \rel_T(\F| x_1, y_1, \dots, x_t, y_t)\right] \\
    &\leq \ee_{\mu, \epsilon}\left[\sup_{f \in \F} 2  L\sum_{j = 1}^k \sum_{s = t+1}^T \epsilon_{s,j} f(x_{s,j}) - L_{t-1}(f) \right] + (T - t + 1)^2  e^{- \sigma k} + (T - t)^3  e^{- \sigma k} \\
    &\leq \ee_{\mu, \epsilon}\left[\sup_{f \in \F} 2  L\sum_{j = 1}^k \sum_{s = t+1}^T \epsilon_{s,j} f(x_{s,j}) - L_{t-1}(f) \right] + (T - t + 1)^3  e^{- \sigma k} \\
    &= \rel_T(\F| x_1, y_1, \dots, x_{t-1}, y_{t-1})
\end{align}
as desired.  Thus we have an admissable relaxation.
\end{proof}

\begin{proof}(Theorem \ref{thm:fastrelaxation})
    It suffices to show the following claim:
    \begin{align}
        \sup_{p_t \in \p} &\ee_{x_t \sim p_t}\left[\sup_{y_t \in [-1,1]} \ell(\yhat_t, y_t)+ \ee_{\mu, \epsilon}\left[\sup_{f \in \F} 2L \sum_{s = t+1}^T \sum_{j=1}^k \epsilon_{s,j} f(x_{s,j}) - L_t(f) + (T - t)^3 e^{- \sigma k}  \right]\right] \\
        &\leq \rel_T(\F | x_1, y_1, \dots, x_{t-1}, y_{t-1})
    \end{align}
    Indeed, if this is the case, then $\yhat_t$ is admissable with respect to $\rel_T(\F| \cdot)$, for which we already have a regret bound in Proposition \ref{prop:slowrelaxation}. To prove the stated claim, we have
    \begin{align}
        \sup_{y_t} &\ee_{\mu, \epsilon}\left[\ell(\yhat_t, y_t) + \sup_{f \in \F}\left[2  L\sum_{j = 1}^k \sum_{s = t+1}^T \epsilon_{s,j} f(x_{s,j}) - L_t(f) \right]\right] \\
        &\leq \ee_{\mu, \epsilon}\left[\sup_{y_t} \ell(\yhat_t, y_t) + \sup_{f \in \F}\left[2  L\sum_{j = 1}^k \sum_{s = t+1}^T \epsilon_{s,j} f(x_{s,j}) - L_t(f) \right]\right] \\
        &= \ee_{\mu, \epsilon}\left[\inf_{\yhat}\sup_{y_t} \ell(\yhat, y_t) + \sup_{f \in \F}\left[2  L \sum_{s = t+1}^n \sum_{j=1}^k \epsilon_{s,j} f(x_{s,j}) - L_t(f) \right]\right]
    \end{align}
    where the inequality follows by Jensen's and the equality follow from the construction of $\yhat_t$ in \eqref{eq:fastrelaxation}.  Now we may apply the proof of Proposition \ref{prop:slowrelaxation} with the expectation with respect to $\mu, \epsilon$ taking place outside of the minimax operation.  This shows that
    \begin{align}
        \ee_{\mu, \epsilon}&\left[\inf_{\yhat}\sup_{y_t} \ell(\yhat, y_t)] + \sup_{f \in \F}\left[2  L \sum_{s = t+1}^n \sum_{j=1}^k \epsilon_{s,j} f(x_{s,j}) - L_t(f) \right]\right] + (T - t)^3 e^{- \sigma k} \\
        &\leq \ee_{\mu, \epsilon}\left[\sup_{f \in \F} 2  L\sum_{j = 1}^k \sum_{s = t+1}^T \epsilon_{s,j} f(x_{s,j}) - L_{t-1}(f) \right] + (T - t + 1)^3  e^{- \sigma k} \\
        &= \rel_T(\F| x_1, y_1, \dots, x_{t-1}, y_{t-1})
    \end{align}
    as desired and the claim holds.

    To prove the oracle efficiency claims, for a fixed $\delta$, let $S$ be a $\delta$-discretization of $[-1,1]$ of size $\frac 2\delta$.  If we solve the minimax problem \eqref{eq:fastrelaxation} over $S$, then by the assumption of $\ell$ being $L$-Lipshitz, we the regret bound for our approximate solution is greater than that of the exact solution by at most an additive constant of $L \delta T$.  In general, for any fixed $\yhat$, we can optimize
    \begin{equation}\label{eq:outersup}
        \sup_{y_t \in S} \left\{\ell(\yhat, y_t) + \sup_{f \in \F}\left[6L \sum_{j = 1}^k  \sum_{s = t+1}^T \epsilon_{s,j} f(x_{s,j}) - L_t(f)\right] \right\}
    \end{equation}
    with $\abs{S}$ calls to the ERM oracle.  Now, note that \eqref{eq:outersup} is convex in $\yhat$ by the assumption of convexity in the first argument of $\ell$.  By a simple three point method from zeroth order optimization \cite{agarwal2011stochastic}, which we prove as Lemma \ref{lem:threepointmethod} for the sake of completeness below, we can minimize \eqref{eq:outersup} with respect to $\yhat \in S$ with $O(\log \abs{S})$ evaluations of the supremum.  Each evaluation of the supremum requires $O(\abs{S})$ calls to the ERM oracle so we require $O\left(\abs{S} \log \abs{S}\right)$ calls in total.  Noting that $\abs{S} \leq \frac 1\delta$, we can get
    \begin{equation}\label{eq:approxregret}
        \reg_T(\yhat) \leq 2 L \R_{kT}(\F) + T^3 e^{- \sigma k} + L T \delta
    \end{equation}
    with $O\left(\frac{T}{\delta} \log\left(\frac 1\delta\right)\right)$ calls to the ERM oracle.  Setting $\delta = \frac 1{L \sqrt{T}}$ recovers the bound in the theorem statement.  The last statement, on linear losses, is proven in Lemma \ref{lem:linearloss}, where we give an explicit representation of the solution using only 2 oracle calls.  Optimizing $k$ concludes the proof.
\end{proof}
\begin{lemma}\label{lem:threepointmethod}
    Let $f : [0,1] \to \mathbb{R}$ be a convex function and let $S \subset [0, 1]$.  Then
    \begin{equation}
        x_0 \in \argmin_{S} f
    \end{equation}
    can be found with $O(\log \abs{S})$ calls to a value oracle that returns $f(x)$ given input $x \in S$.
\end{lemma}
\begin{proof}
    Motivated by \cite{agarwal2011stochastic}, we describe the following recursive algorithm that shrinks $S$ until it contains only one point, yet always includes the minimizer.  Let $S_0 = S$.  To construct $S_{i+1}$ from $S_i$, order the points $x_1, \dots, x_m \in S_i$ such that $x_i < x_j$ for all $i < j$.  Let $z_1, z_2, z_3$ be the $\frac 14$, $\frac 12$, and $\frac 34$ quanatiles of $S$ respectively and evaluate $f(z_1)$, $f(z_2)$, and $f(z_3)$ with three calls to the value oracle.  There are several cases:
    \paragraph{$\mathbf{f(z_1) > f(z_2) < f(z_3)}$} If the middle point is smaller than either point on the end, then be the convexity of $f$ we know that the minimum must occcur for some $x$ such that $z_1 < x < z_3$.  In this case let $S_{i+1}$ contain all the points $x \in S_i$ such that $z_1 < x < z_3$.  Note that $\abs{S_{i+1}} \leq \frac 12 \abs{S_i}$.

    \paragraph{$\mathbf{f(z_1) < f(z_2) > f(z_3)}$}  This case corresponds to the middle point being higher than the end points.  This is not possible, however, as $f$ is convex.

    \paragraph{$\mathbf{f(z_1) < f(z_2) < f(z_3)}$}  In this case, convexity assures us that the minimizer cannot be at any point $x \geq z_2$ and so we let $S_{i+1}$ to be the set of all points $x \in S_i$ such that $x < z_2$.  Note that $\abs{S_{i+1}} \leq \frac 12\abs{S_i}$.

    \paragraph{$\mathbf{f(z_1) > f(z_2) > f(z_3)}$}  This is the mirror image of the previous case and can be handled similarly.

    \paragraph{$\mathbf{f(z_1) = f(z_2) = f(z_3)}$}  In this case, convexity ensures that the minimizer must have $x \leq z_1$ or $x \geq z_3$ and so we let $S_{i+1}$ be the set of $x \in S_i$ satisfying this constraint.  Again, $\abs{S_{i+1}} \leq \frac 12 \abs{S_i}$.

    In any case, with three calls to the value oracle, we reduce the size of $S_i$ by a factor of 2.  Thus we can find $x_0$ in $O(\log \abs{S})$ calls as claimed.
\end{proof}
\begin{lemma}\label{lem:linearloss}
    Suppose we are in the situation of Theorem \ref{thm:fastrelaxation} and the loss $\ell(\yhat, y) = \frac{1 - \yhat y}{2}$.  Then the problem \eqref{eq:fastrelaxation} can be solved with two calls to the ERM oracle.  In fact, $\yhat_t$ is given by
    \begin{align}
        \frac 12&\sup_{f \in \F}\left[ 6L \sum_{j = 1}^k  \sum_{s = t+1}^T \epsilon_{s,j} f(x_{s,j}) - L_{t-1}(f) - \ell(f(x_t), 1)\right]\\ &- \frac 12\sup_{f \in \F}\left[ 6L\sum_{j = 1}^k  \sum_{s = t+1}^T \epsilon_{s,j} f(x_{s,j}) - L_{t-1}(f) - \ell(f(x_t), -1)\right] 
    \end{align}
\end{lemma}
\begin{proof}
    We are trying to minimize with respect to $\yhat \in [-1,1]$.
    \begin{equation}\label{eq:outersup2}
        \sup_{y_t \in S} \left\{\ell(\yhat, y_t) + \sup_{f \in \F}\left[6L \sum_{j = 1}^k  \sum_{s = t+1}^T \epsilon_{s,j} f(x_{s,j}) - L_t(f)\right] \right\}
    \end{equation}
    Independent of $\yhat$, if $\ell$ is linear in $y$, then the expresion over which we are taking the supremum in \eqref{eq:outersup2} is convex in $y$ and thus $y_t \in \{\pm 1\}$.  For the sake of simplicity, suppose that
    \begin{equation}
        \sup_{f \in \F} \left[6L \sum_{j = 1}^k  \sum_{s = t+1}^T \epsilon_{s,j} f(x_{s,j}) - L_t(f)\right]
    \end{equation}
    is maximized by functions $f_+$ and $f_-$ depending on if $y_t = 1$ or $y_t = -1$ (in the general case, we could take a sequence of functions attaining the supremum).  
    Let
    \begin{align}
        a_+ &=  6L \sum_{j = 1}^k \sum_{s = t+1}^T \epsilon_{s,j} f_+(x_{s,j}) - L_t(f_+) \\
        a_- &=  6L \sum_{j = 1}^k \sum_{s = t+1}^T \epsilon_{s,j} f_-(x_{s,j}) - L_t(f_-)
    \end{align}
    Then the solution to \eqref{eq:fastrelaxation} is given by
    \begin{equation}
        \min_{\yhat \in [-1,1]} \max\left(\frac{1 - \yhat}{2} +  a_+, \frac{1 + \yhat}{2} + a_-  \right)
    \end{equation}
    The maximum is taken over two linear functions of $\yhat$ with opposite slope and so the minimax result is where they intersect, assuming they intersect somewhere in $[-1,1]$, which they do if $\abs{a_+ - a_-} \leq 1$.  Note that
    \begin{align}
        a_+ &= 6L \sum_{j = 1}^k  \sum_{s = t+1}^T \epsilon_{s,j} f_+(x_{s,j}) - L_t(f_+) \\
        &= 6L \sum_{j = 1}^k  \sum_{s = t+1}^T \epsilon_{s,j} f_+(x_{s,j}) - L_{t-1}(f_+) - \ell(f_+(x_t), 1) + \ell(f_+(x_t), -1) - \ell(f_+(x_t), -1) \\
        &\geq 6L \sum_{j = 1}^k  \sum_{s = t+1}^T \epsilon_{s,j} f_-(x_{s,j}) - L_t(f_-) - \ell(f_+(x_t), 1) + \ell(f_+(x_t), -1) \\
        &\geq a_- - 1
    \end{align}
    By symmetry, $\abs{a_- - a_+} \leq 1$.  Thus,
    \begin{equation}
        \yhat_t = \frac{a_+ - a_-}{2} \in [-1,1]
    \end{equation}
    solves the minimax problem with two calls to the ERM oracle.
\end{proof}

\subsection{Proof of Proposition \ref{prop:nonparametriclowerbound}}
In order to prove Proposition \ref{prop:nonparametriclowerbound}, we consider the following setting.  Fix an $\alpha > 0$, suppose that $\vc(\F, \alpha) = m$ and let $x_1, \dots, x_m$ shatter $\F$ at scale $\alpha$.  We let $p_t$ be uniform on $x_1, \dots, x_m$ and let $\mu = (1 - \sigma) \delta_{x^\ast} + \sigma p_t$, where $x^\ast$ is a distinguished point satisfying $f(x^\ast) = 0$ for all $f \in \F$.  Note that $p_t$ is $\sigma$-smooth with respect to $\mu$.  We compare the expected Rademacher complexity sampling $n$ points according to $\mu$ to that when sampling according to $p_t$.  We require two lemmata.
\begin{lemma}\label{lem:numpoints1}
    Suppose we are in the above setting and $X_1, \dots, X_T$ are sampled independently according to $\mu$.  Then, with probability at least $1 - \delta$, the number of indices $i$ such that $X_i \neq x^\ast$ is at most $2 \sigma T $ if $T \geq \frac 3\sigma \log\left(\frac 1\delta\right)$.
\end{lemma}
\begin{proof}
    Let $Y_i = \mathbf{1}[X_i \neq x^\ast]$.  Then $Y_i$ are independent Bernoulli random variables with parameter $\sigma$ and the number of such indices is the sum of $Y_i$.  Applying Chernoff's inequality, we have
    \begin{equation}
        \pp\left(\sum_{i = 1}^T Y_i \geq 2 \sigma T\right) \leq e^{- \frac{\sigma T}{3}}
    \end{equation}
    The assumption of $T$ large enough concludes the proof.
\end{proof}
\begin{lemma}\label{lem:numpoints2}
    Suppose that we are in the setting described above and $X_1, \dots, X_T$ are sampled according to $p_t$.  Suppose that $T \geq 8 m \log\left(\frac m\delta\right)$.  Then with probability at least $1 - \delta$, for each $1 \leq j \leq m$, there are at least $\frac{T}{2m}$ indices $i$ such that $X_i = x_j$.
\end{lemma}
\begin{proof}
    Fix $j$ and let $Y_i = \mathbf{1}[X_i = x_j]$.  Then the $Y_i$ are independent Bernoulli random variables with parameter $\frac 1m$.  Letting $S_T$ denote the sum of the $Y_i$, which is the desired number of indices, we may apply Chernoff's inequality to get
    \begin{equation}
        \pp\left(S_T \leq \frac 12 \frac{T}{m}\right) \leq e^{- \frac{T}{8 m}}
    \end{equation}
    and so with probability at least $1 - \frac \delta m$, there are at least $\frac{T}{2m}$ indices $i$ such that $X_i = x_j$.  Applying a union bound concludes the proof.
\end{proof}
We may now adapt an argument from \cite{mendelson2002rademacher} and \cite[Lemma A.2]{srebro2010smoothness} to lower bound the Rademacher complexity according to $p_t$:
\begin{lemma}\label{lem:rademacherlowerbound}
    Suppose that we are in the above setting and suppose that $T \geq 8 m \log(2m)$.  Then,
    \begin{equation}
        \ee_{p_t}\left[\R_T(\F)\right] \geq \frac \alpha 8 \sqrt{\vc(\F, \alpha) T}
    \end{equation}
\end{lemma}
\begin{proof}
  Let $\chi_A$ denote the indicator of the high probability event $A$ from Lemma \ref{lem:numpoints2}. Consider any fixed choice of $X_1, \ldots, X_T$ so that the event $A$ holds. For each $j \in [m], k \in [T/(2m)]$, let $\phi(j,m) \in [T]$ denote the $k$th smallest value of $i$ so that $X_i = x_j$ (that all such $\phi(j,m)$ exist is guaranteed by $A$). Furthermore let $\Phi \subset [T]$ denote the image of $\phi$, so that $|\Phi| = T/2$. 
  Next, for any $\epsilon \in \{-1,1\}^n$, define $f^\epsilon := \argmax_{f \in \MF} \sum_{j=1}^m \sum_{k=1}^{T/(2m)} \ep_{\phi(j,m)} f(x_j)$.\footnote{If the argmax does not exist, we may instead consider a sequence of functions that approximates the argmax to arbitrarily small precision.}  
    \begin{align}
        \ee_{p_t}\left[\R_T(\F)\right] &= \ee_{p_t}\left[\ee_\epsilon\left[\sup_{f \in \F} \sum_{i = 1}^T \epsilon_i f(X_i) \bigg| X_1, \dots, X_T\right]\right] \\
                                       &\geq \ee_{p_t}\left[\chi_A \ee_\epsilon\left[\sup_{f \in \F} \sum_{i = 1}^T \epsilon_if(X_i) \bigg| X_1, \dots, X_T\right]\right] \label{eq:wc-rc-nn}\\
       &\geq \ee_{p_t}\left[\chi_A \ee_\epsilon\left[\sum_{i = 1}^T \epsilon_i f^\epsilon(X_i) \bigg| X_1, \dots, X_T\right]\right]  \\
      &  \geq  \ee_{p_t}\left[\chi_A \ee_\epsilon\left[ \sum_{j = 1}^m \sum_{k = 1}^{\frac{T}{2 m}} \epsilon_{jk} f^\ep(x_j)\bigg| X_1, \dots, X_T\right]\right] + \ee_{p_t}\left[\chi_A \ee_\epsilon\left[\sum_{i \not \in \Phi} \epsilon_if^\ep(X_i) \bigg| X_1, \dots, X_n\right]\right] \\
        &\geq \ee_{p_t}\left[\chi_A \ee_\epsilon\left[\sup_{f \in \F} \sum_{j = 1}^m \sum_{k = 1}^{\frac{T}{2 m}} \epsilon_{\phi(j,k)} f(x_j)\bigg| X_1, \dots, X_T\right]\right] \label{eq:remove-0term}\\
        &= \pp(A)\ee_\epsilon\left[\sup_{f \in \F} \sum_{j = 1}^m \sum_{k = 1}^{\frac{T}{2 m}} \epsilon_{\phi(j,k)} f(x_j)\right]
    \end{align}
    where \eqref{eq:wc-rc-nn} follows by Jensen's inequality coupled with the fact that the $\epsilon_i$ are mean zero, and \eqref{eq:remove-0term} follows since $\{\ep_i :\ i \in \Phi\}$ are independent of $\{ \ep_i :\ i \not \in \Phi\}$ (so that $\ee_{p_t}\left[\chi_A \ee_\epsilon\left[\sum_{i \not \in \Phi} \epsilon_if^\ep(X_i) \bigg| X_1, \dots, X_T\right]\right] = 0$), and by definition of $f^\ep$. 
    By the triangle inequality and symmetry of the $\ep_i$, $i \in [T]$, we have
    \begin{align}
        \ee_\epsilon\left[\sup_{f \in \F} \sum_{j = 1}^m \sum_{k = 1}^{\frac{T}{2 m}} \epsilon_{\phi(j,k)} f(x_j)\right] &\geq \frac 12 \ee_\epsilon\left[\sup_{f, f' \in \F} \sum_{j = 1}^m \sum_{k = 1}^{\frac{T}{2 m}} \epsilon_{\phi(j,k)} (f(x_j) - f'(x_j))\right] \\
        &\geq \frac 12 \ee_\epsilon\left[ \sum_{j = 1}^m \sum_{k = 1}^{\frac{T}{2 m}} \epsilon_{\phi(j,k)} (f_\epsilon(x_j) - f_\epsilon'(x_j))\right]
    \end{align}
    where $f_\ep, f_\ep'$ are chosen so that 
    \begin{align}
        \sign\left(\sum_{k = 1}^{\frac T{2m}} \epsilon_{\phi(j,k)}\right)\left(f_\epsilon(x_j) - s_j\right) \geq \frac \alpha 2 && \sign\left(\sum_{k = 1}^{\frac T{2m}} \epsilon_{\phi(j,k)}\right)\left(f_\epsilon'(x_j) - s_j\right) \leq -\frac \alpha 2
    \end{align}
    for some $s_1, \dots, s_m \in \rr$.  Note that there exist such $f_\epsilon, f_\epsilon' \in \F$ by the assumption that $x_1, \dots, x_m$ shatter $\F$ at scale $\alpha$.  We thus have
    \begin{align}
        \ee_\epsilon\left[ \sum_{j = 1}^m \sum_{k = 1}^{\frac{T}{2 m}} \epsilon_{\phi(j,k)} (f_\epsilon(x_j) - f_\epsilon'(x_t))\right] &\geq \frac 12 \ee_\epsilon\left[ \sum_{j = 1}^m \abs{\sum_{k = 1}^{\frac{T}{2 m}} \epsilon_{\phi(j,k)}} \alpha \right] \\
        &\geq \frac m2 \alpha \ee_\epsilon\left[\abs{\sum_{i = 1}^{\frac T{2m}} \epsilon_i}\right] \\
        &\geq \frac m2 \alpha \sqrt{\frac{T}{4m}} \label{eq:khintchine}\\
        &= \frac{\alpha}{4} \sqrt{T m},
    \end{align}
    where \eqref{eq:khintchine} follows from Khintchine's inequality. 
    By Lemma \ref{lem:numpoints2}, $\pp(A) \geq \frac 12$.  Thus, putting everything together, we have
    \begin{equation}
        \ee_{p_t}\left[\R_T(\F)\right] \geq \frac{\alpha}{8}\sqrt{mT}
    \end{equation}
We finally recall that $m = \vc(\F, \alpha)$ and conclude the proof.
\end{proof}
The last thing we need to do is provide an upper bound on $\ee_\mu\left[\R_T(\F)\right]$.  We can do this using chaining and Lemma \ref{lem:numpoints1}.
\begin{lemma}\label{lem:donskerrademacherupperbound}
    Suppose we are in the setting above and suppose that $T \geq \frac 3\sigma \log T$.  Then there is an absolute constant $C$ such that
    \begin{equation}
        \ee_\mu\left[\R_T(\F)\right] \leq C \sqrt{\sigma T}
    \end{equation}
\end{lemma}
\begin{proof}
    Let
    \begin{equation}
        A = \left\{\abs{\{i | X_i \neq x^\ast \}} \leq 2\sigma T \right\}
    \end{equation}
    and let $\chi_A$ denote the indicator for this event.  By Lemma \ref{lem:numpoints1}, $\pp(A) \geq 1 - \frac 1T$.  Note that as $\F$ is uniformly bounded by $1$, we always have the trivial upper bound of $\R_T(\F) \leq T$ on any data set.  We can thus compute
    \begin{align}
        \ee_\mu\left[\R_T(\F)\right] &= \ee_\mu\left[\chi_A \R_T(\F)\right] + (1 - \pp(A)) T \leq \ee_\mu\left[\chi_A \R_T(\F)\right] + 1
    \end{align}
    By the definition of the event $A$, we have:
    \begin{align}
        \ee_\mu\left[\chi_A \R_T(\F)\right] &= \ee_\mu\left[\chi_A \ee_\epsilon\left[\sup_{f \in \F} \sum_{X_i \neq x^\ast} \epsilon_i f(X_i) + \sum_{X_i = x^\ast} \epsilon_i f(X^\ast) \bigg| X_1, \dots, X_T \right]\right] \\
        &=  \ee_\mu\left[\chi_A \ee_\epsilon \left[\sup_{f \in \F} \sum_{X_i \neq x^\ast} \epsilon_i f(X_i)\bigg| X_1, \dots, X_T\right] \right]\\
        &\leq \sup_{X_i} \ee_\epsilon\left[\sup_{f \in \F} \sum_{i = 1}^{2\sigma T} \epsilon_i f(X_i)\right]
    \end{align}
    We may now apply \cite[]{rudelson2006combinatorics} to get that
    \begin{equation}
        \sup_{X_i} \ee_\epsilon\left[\sup_{f \in \F} \sum_{i = 1}^{2 \sigma T } \epsilon_i f(X_i)\right] \leq C \sqrt{2\sigma T } \int_0^{1} \sqrt{\vc(\F, \alpha)} d \alpha
    \end{equation}
    Because $\vc(\F, \alpha) \leq C \alpha^{- p}$ for some $p < 2$, the result follows.
\end{proof}

We are now ready to prove the main bound.
\begin{proof}(Proposition \ref{prop:nonparametriclowerbound})
    Let $\alpha = \sqrt{\sigma}$ and set $\F, \mu, p_t, \cX$ as above.  By Lemma \ref{lem:rademacherlowerbound} and Lemma \ref{lem:donskerrademacherupperbound}, we have
    \begin{align}
        \frac{\ee_{p_t}[\R_T(\F)]}{\ee_{\mu}[\R_T(\F)]} &\geq \frac{c \sqrt{\sigma \vc(\F, \sqrt \sigma) T}}{C \sqrt{\sigma T}}  = c \sqrt{\vc(\F, \sqrt{\sigma})} \geq c \sigma^{- \frac p4}
    \end{align}
    where the last inequality follows from the assumption on the complexity of $\F$.  We may now apply the lower bound in Proposition \ref{prop:rakhlinrademacher} with $\sD$ just independent copies of $p_t$.  The result follows.
\end{proof}


\section{Proof of Theorem \ref{thm:generalftpl}}\label{app:ftpl}

In this section we prove Theorem \ref{thm:generalftpl}, which gives a regret bound for the follow-the-perturbed-leader (FTPL) algorithm \eqref{eq:ftpl-erm} with respect to general classes $\F$ for convex, Lipschitz loss functions.  This result bounds the regret of an FTPL style algorithm by a stability term and a term corresponding to the size of the perturbation.  In particular, we bound the stability term by controlling the Wasserstein distance between the laws of $\yhat_t$ and $\yhat_{t+1}$.  The techniques involved are of independent interest as we develop a novel Gaussian anti-concentration inequality that applies even when the labels are not assumed smooth.  
We begin by stating and proving the relevant variant of the BTL lemma.  Then, in Appendix \ref{subsec:wasserstein} we provide the stability bound, using our Gaussian anti-concentration approach.  We continue in Appendix \ref{subsec:generalizationerror} by controlling the final stability term in the below decomposition, in the special case of linear loss.  Finally, we conclude the proof in Appendix \ref{subsec:conclude} by extending from linear loss to general loss in the case of smooth labels and applying a discretization approach to recover full generality.

We consider the smoothed online setting with distribution $\mu$.
More specifically, we consider  the following setting: for some parameter $n \in \mathbb{N}$, for each time step $t \in [T+1]$, consider points  $X_{t,1}, \ldots, X_{t,n} \in \MX$, and define
\begin{align}
\forall f \in \MF: \qquad \hat \omega_{t,n}(f) := \frac{1}{\sqrt n} \cdot \sum_{i=1}^n \gamma_{t,i} \cdot f(X_{t,i})\label{eq:hat-omega},
\end{align}
where $\gamma_{t,1}, \ldots, \gamma_{t,n}$ are i.i.d.~standard normal random variables. We define $\hat \mu_{t,n}$ to be the distribution $\hat \mu_{t,n} := \frac{1}{n} \sum_{i=1}^n \delta_{X_{t,i}}$, where $\delta_{X_{t,i}}$ denotes the point mass at $X_{t,i}$. In what follows we will consider iterates $f_t$, $1 \leq t \leq T+1$, satisfying, for some $\erma > 0$,
\begin{equation}
L_{t-1}(f_t) + \eta \cdot \hat \omega_{t,n}(f_t) \leq \argmin_{f \in \MF} L_{t-1}(f) + \eta \cdot \hat \omega_{t,n}(f) + \erma \label{eq:ft-hat-define}.
\end{equation}
We begin with a classic regret decomposition based on the well-known ``Be-the-Leader'' Lemma \citep{kalai2005efficient,cesa2006prediction}.  We first prove a related, auxiliary result that allows us to deal with different perturbations at each time step:
 \begin{lemma}\label{lem:btl-real}
   Suppose $f_t$, for $t \in [T]$, is defined as in \eqref{eq:ft-hat-define}, for any (adaptively chosen) sequence $(x_1,y_1), \ldots, (x_T, y_T) \in \MX \times [-1,1]$. Then it holds that
   \begin{align}
     \E \left[ \sum_{t=1}^T \ell(f_{t+1}(x_t), y_t) - \inf_{f \in \MF}\sum_{t=1}^T \ell(f(x_t), y_t) \right] \leq \zeta \cdot (T+1) + 2 \cdot \E \left[ \sup_{f \in \MF} \hat \omega_{1,n} \right].
   \end{align}
 \end{lemma}
 \begin{proof}
   We first use induction on $T \geq 0$ to show the following statement:
   \begin{align}
\E \left[ \sum_{t=1}^T \ell(f_{t+1}(x_t), y_t) \right] \leq \E \left[ \sum_{t=1}^T \ell(f_{T+1}(x_t), y_t) + \hat \omega_{T+1,n}(f_{T+1}) \right] + \zeta \cdot T + \E \left[ \sup_{f \in \MF} \hat \omega_{1,n}(f) \right]\label{eq:btl-induction}.
   \end{align}
To establish the base case $T = 0$, we note that $0 \leq \E \left[ \hat \omega_{1,n}(f_1) \right] +  \E \left[ \sup_{f \in \MF} \hat \omega_{1,n}(f) \right]$, which follows because $\E \left[ \sup_{f \in \MF} - \hat \omega_{1,n}(f) \right] = \E \left[ \sup_{f \in \MF}  \hat \omega_{1,n}(f) \right]\geq 0$ (by symmetry of the process $\hat \omega_{1,n}$). 

   Now assume that \eqref{eq:btl-induction} holds at some step $T-1$, namely that
   \begin{align}
\E \left[ \sum_{t=1}^{T-1} \ell(f_{t+1}(x_t), y_t) \right] \leq \E \left[ \sum_{t=1}^{T-1} \ell(f_{T}(x_t), y_t) + \hat \omega_{T,n}(f_{T}) \right] + \zeta \cdot (T-1) +\E \left[ \sup_{f \in \MF} \hat \omega_{1,n}(f) \right] .\label{eq:inductive-case-tm1}
   \end{align}
   By definition of $f_{T}$ in \eqref{eq:ft-hat-define}, we have
   \begin{align}
     \E \left[ \sum_{t=1}^{T-1} \ell(f_{T}(x_t), y_t) + \hat \omega_{T,n}(f_{T}) \right] \leq & \E \left[ \inf_{f \in \MF} \sum_{t=1}^{T-1} \ell(f(x_t), y_t) + \hat \omega_{T,n}(f) \right] + \zeta \\
     = &  \E \left[\inf_{f \in \MF} \sum_{t=1}^{T-1} \ell(f(x_t), y_t) + \hat \omega_{T+1,n}(f) \right] + \zeta\label{eq:important-equality}\\
     \leq &  \E \left[ \sum_{t=1}^{T-1} \ell(f_{T+1}(x_t), y_t) + \hat \omega_{T+1,n}(f_{T+1}) \right]  + \zeta\label{eq:use-omega-tt1},
   \end{align}
   where \eqref{eq:important-equality} follows because, conditioned on $(x_1, y_1), \ldots, (x_{T-1}, y_{T-1})$, the process $\hat \omega_{T,n}$ is drawn independently, as is the process $\hat \omega_{T+1,n}$, and both have the same conditional distribution. From \eqref{eq:inductive-case-tm1} and \eqref{eq:use-omega-tt1} we have that
   \begin{align}
\E \left[ \sum_{t=1}^{T-1} \ell(f_{t+1}(x_t), y_t) \right] \leq \E \left[ \sum_{t=1}^{T-1} \ell(f_{T+1}(x_t), y_t) + \hat \omega_{T+1,n}(f_{T+1}) \right] + \zeta \cdot T + \E \left[ \sup_{f \in \MF} \hat \omega_{1,n}(f) \right].
   \end{align}
   Adding $\E[\ell(f_{T+1}(x_T), y_T)]$ to both sides establishes \eqref{eq:btl-induction}, thus completing the inductive hypothesis.
   
   To complete the proof of the lemma, we note that
   \begin{align}
     & \E \left[ \sum_{t=1}^T \ell(f_{T+1}(x_t), y_t) + \hat \omega_{T+1,n}(f_{T+1}) \right]\\
     \leq & \E \left[ \inf_{f \in \MF} \sum_{t=1}^T \ell(f(x_t), y_t) + \hat \omega_{T+1,n}(f) \right] + \zeta \\
     \leq & \E \left[ \inf_{f \in \MF} \sum_{t=1}^T \ell(f(x_t), y_t) +   \sup_{f' \in \MF} \hat \omega_{T+1,n}(f')\right] + \zeta \\ 
     = & \E \left[ \inf_{f \in \MF} \sum_{t=1}^T \ell(f(x_t), y_t) \right] +  \E \left[ \sup_{f \in \MF} \hat \omega_{T+1,n}(f) \right] + \zeta,
   \end{align}
   which implies, combined with \eqref{eq:btl-induction} and the fact that $\hat \omega_{1,n}$ and $\hat \omega_{T+1,n}$ are identically distributed, that
   \begin{align}
\E \left[ \sum_{t=1}^T \ell(f_{t+1}(x_t), y_t) - \inf_{f \in \MF}\sum_{t=1}^T \ell(f(x_t), y_t) \right] \leq \zeta \cdot (T+1) + 2 \cdot \E \left[ \sup_{f \in \MF} \hat \omega_{1,n} \right].
   \end{align}
 \end{proof}
Using Lemma \ref{lem:btl-real}, we get a decomposition of the regret into a stability term and a perturbation size term.  The stability term is further decomposed for the future analysis.
\begin{lemma}\label{lem:btl}
  Let $f_t$ be defined as in \eqref{eq:ft-hat-define} and let $(x_1, y_1), \dots, (x_T, y_T)$ be any sequence of elements in $\cX \times[-1,1]$.  Let $(x_1', y_1'), \dots, (x_T', y_T')$ be a tangent sequence, meaning that for all $1 \leq t \leq T$, $(x_t', y_t')$ is independent and identically distributed as $(x_t, y_t)$ conditioned on $(x_s, y_s)$ for $s < t$.  Then we may upper bound the expected regret by the following expression:
  \begin{align}
    2 \eta \ee\left[\sup_{f \in \F} \hat{\omega}_{1,n}(f) \right] + \sum_{t = 1}^T \ee\left[\ell(f_t(x_t'), y_t') - \ell(f_{t+1}(x_t'), y_t')\right] + \sum_{t = 1}^T \ee\left[\ell(f_{t+1}(x_t'), y_t') - \ell(f_{t+1}(x_t), y_t)\right] \label{eq:btl}
  \end{align}
\end{lemma}
\begin{proof}
  By Lemma \ref{lem:btl-real}, we have
    \begin{equation}
      \E \left[ \sum_{t=1}^T \ell(f_{t+1}(x_t), y_t) - \inf_{f \in \MF}\sum_{t=1}^T \ell(f(x_t), y_t) \right] \leq \zeta \cdot (T+1) + 2 \cdot \E \left[ \sup_{f \in \MF} \hat \omega_{1,n} \right]
    \end{equation}
    Adding and subtracting $\ell(f_t(x_t), y_t)$ from both sides and rearranging yields
    \begin{align}
      \ee\left[\reg_T(f_t)\right] \leq \ee\left[\sum_{t = 1}^T \ell(f_t(x_t), y_t) - \ell(f_{t+1}(x_t), y_t)\right] + 2 \eta \ee\left[\sup_{f \in \F} \hat{\omega}_{1,n}(f)\right] + \zeta T
    \end{align}
    Now, note that $f_t$ is independent of $(x_t, y_t)$ be construction, so $\ee\left[\ell(f_t(x_t), y_t)\right] = \ee\left[\ell(f_t(x_t'), y_t')\right]$.  Adding and subtracting $\ell(f_{t+1}(x_t'), y_t')$ yields
    \begin{align}
      \ee\left[\ell(f_t(x_t), y_t) - \ell(f_{t+1}(x_t), y_t)\right] = \ee\left[\ell(f_t(x_t'), y_t') - \ell(f_{t+1}(x_t'), y_t')\right] + \ee\left[\ell(f_{t+1}(x_t'), y_t') - \ell(f_{t+1}(x_t), y_t)\right]
    \end{align}
    Applying linearity of expectation concludes the proof.
\end{proof}
The classic decomposition in \eqref{eq:btl} allows for the control of each term independently.  For the first, empirical process theory allows us to control $\ee\left[\sup \hat{\omega}_{1,n}(f)\right]$ by the entropy of $\F$.  The last term, called ``generalization error'' in \cite{haghtalab2022oracle}, can be controlled in the case of linear loss by appealing to standard uniform deviations bounds; this is done in Appendix \ref{subsec:generalizationerror}.  The key term is the middle one, whose control guarantees that $f_t$ and $f_{t+1}$ are close in an appropriate sense.  We now present this bound.

  \subsection{Regret Bound Using the Wasserstein Distance}\label{subsec:wasserstein}

  In this section we provide our bound on the middle term of \eqref{eq:btl}.  In particular, for any $t$, we show that $\ee\left[\ell(f_t(x_t'), y_t') - \ell(f_{t+1}(x_t'), y_t')\right]$ is small.  We leverage the fact that $\ell$ is Lipschitz in the first coordinate and use this fact along with the smoothness of $x_t'$ to reduce to showing that $\norm{f_t - f_{t+1}}_{L^2(\mu)}$ is small in expectation over the perturbation.  

  We first argue that it suffices to consider $\F$ such that
    \begin{equation}\label{eq:infnorm}
        \inf_{f \in \F} \norm{f}_{L^2(\mu)} \geq \frac 23
    \end{equation}
  Indeed, take any $\F$ and any $\mu$.  Enlarge $\cX$ to $\cX \cup \{x^\ast\}$, where $x^\ast$ is a new point such that $f(x^\ast) = 1$ for all $f \in \F$.  Let $\widetilde{\mu} = \frac 13 \mu + \frac 23 \delta_{x^\ast}$.  Then if $p_t$ is $\sigma$-smooth with respect to $\mu$ then it is $\frac \sigma 3$-smooth with respect to $\widetilde{\mu}$.  Moreover, $\norm{f}_{L^2(\widetilde{\mu})} \geq \frac 23$ for all $f \in \F$, and since all $f$ take the same value on $x^\ast$, an ERM oracle for the original class clearly yields an ERM oracle for the new class with domain $\MX \cup \{ x^\ast\}$ (the oracle can simply ignore all points of the form $(x^\ast, y)$).  Thus, at the cost of shrinking $\sigma$ by a factor of $3$, we will suppose this lower bound throughout this section.

  In Lemma \ref{lem:stabilityftplprob} below, we show that if $f_t, f_{t+1}$ are defined with respect to a common noise process $\omega(\cdot)$, then they are close with high probability. In the lemma, we consider an arbitrary Lipschitz loss function $\ell$, and define $L_t(f) := \sum_{s=1}^t \ell(f(x_s), y_s)$.
  \begin{lemma}\label{lem:stabilityftplprob}
  Fix any $\zeta > 0$, $t \in \BN$, $\ell$ as above, and an arbitrary sequence $(x_1, w_1), \ldots, (x_{t-1}, w_{t-1}) \in \MX \times \BR$. 
  Let $\omega$ denote a Gaussian process on a separable class $\F$ with covariance $\Sigma_{fg} = \ee_{X \sim \mu}\left[f(X)g(X)\right]$ for some measure $\mu$ on $\cX$ and, by abuse of notation, let $\omega(f)$ denote a single sample from this process.  Suppose that $f_t$ satisfies
    \begin{equation}
        L_{t-1}(f_t) + \eta \omega(f_t) \leq  \inf_{f \in \F} L_{t-1}(f) + \eta \omega(f) + \erma,
      \end{equation}
      and for each $(x,w) \in \MX \times [-1,]$, there is some $f_{t+1,x,w}$ such that
          \begin{equation}
        L_{t,x,w}(f_{t+1,x,w}) + \eta \omega(f_{t+1,x,w}) \leq  \inf_{f \in \F} L_{t,x,w}(f) + \eta \omega(f) + \erma.
      \end{equation}
      Suppose further that $f_t, f_{t+1,x,w}$ are measurable with respect to the $\sigma$-algebra generated by $\omega$. \footnote{Here we write $L_{t,x,w}(f) = \sum_{s=1}^{t-1} \ell(f(x_s), w_s) + \ell(f(x), w)$.}
    Then,
    \begin{equation}
        \pp\left(\sup_{(x,w) \in \MX \times [-1,1]} \norm{f_t - f_{t+1,x,w}}_{L^2(\mu)} > \alpha\right) \leq \frac{8(L+2\erma)^2}{\alpha^4\eta^2 \inf_{f \in \F}\norm{f}_{L^2(\mu)}^6} + \frac{4(L+2\erma)}{\alpha^2 \eta\inf_{f \in \F}\norm{f}_{L^2(\mu)}^4} \ee\left[\sup_{f\in \F} \omega(f)\right].\label{eq:bound-wasserstein}
    \end{equation}
\end{lemma}
\begin{proof}
  \nc{\sel}[1]{f_t({#1})}
  \nc{\aand}{\mbox{ and }}
  \nc{\far}[2]{\mathcal{D}_{#1}({#2})}
    As $\F$ is seperable, it suffices to take a countable dense subset and assume that $\F$ is countable. By assumption we can write $f_t = \sel{\omega}$ for some measurable function $\sel{\cdot}$ (where measurability is with respect to the product topology on $\BR^\MF$). 
  
    Let
    \begin{equation}
        A_t = \left\{g \in \F | \norm{g - f_t}_{L^2(\mu)} > \alpha \text{ and } L_{t-1}(g) + \eta \omega(g) \leq L_{t-1}(f_t) + \eta \omega(f_t) + 2 L + \erma \right\}
    \end{equation}
    Note that for any $g$ for which $L_{t-1}(g) + \eta \omega(g) > L_{t-1}(f_t) + \eta \omega(f_t) + 2L + \erma$ and for any $(x,w) \in \MX \times [-1,1]$, 
    \begin{align}
        L_{t,x,w}(f_t) + \eta \omega(f_t) &= L_{t-1}(f_t) + \eta \omega(f_t) + \ell(f_t(x), w) \\
        &< L_{t-1}(g) + \eta \omega(g) + \ell(f_t(x), w) -2L-\erma \\
                                    &= L_t(g) + \eta \omega(g) + \ell(f_t(x), w) - \ell(g(x), w) -2L-\erma\\
      & \leq L_t(g) + \eta \omega(g)-\erma,
    \end{align}
    where the final inequality follows because  $\abs{\ell(f_t(x), w) - \ell(g(x), w)} \leq 2 L$ as $\ell$ is $L$-Lipschitz. 
    Suppose that $g \not\in A_t$.  Then either $\norm{g - f_t}_{L^2(\mu)} \leq \alpha$ or, using the above display, for all  $x,w$, $L_{t,x,w}(f_t) + \eta \omega(f_t) + \erma < L_{t,x,w}(g) + \eta \omega(g)$, meaning that $f_{t+1,x,w}$ cannot be equal to $g$ for any choice of $x,w$.   Hence, the event that $\sup_{(x,w) \in \MX \times [-1,1]} \| f_t - f_{t+1,x,w} \|_{L^2(\mu)} > \alpha$ implies that for some $(x,w)$, $f_{t+1,x,w} \in A_t$. Thus, it suffices to bound the probability that $A_t$ is nonempty.

    Let $\far{\alpha}{f} := \{ g \in \MF | \norm{g-f}_{L^2(\mu)} > \alpha \}$. 
    As $\F$ is assumed countable, we have
    \begin{align}
      & \pp(\abs{A_t} = 0) \\
      &\geq \sum_{f \in \F} \pp\left(\sel{\omega} = f \mbox{ and } \inf_{g \in \far{\alpha}{f}} L_{t - 1}(g) + \eta \omega(g) - (2L+4\erma)  \geq L_{t - 1}(f) + \eta \omega(f)-\erma\right) \\
        &= \sum_{f \in \F} \ee_{y} \left[\pp\left[\sel{\omega} = f \mbox{ and } \inf_{g \in \far{\alpha}{f}} L_{t - 1}(g) + \eta \omega(g) - (2L + 4\erma) \geq y -\erma| L_{t - 1}(f) + \eta \omega(f) = y \right]\right]\label{eq:sum-exp-prob},
    \end{align}
    where in \eqref{eq:sum-exp-prob} the expectation is over the distribution of $y=L_{t-1}(f) + \eta \omega(f)$. 
    We now fix an $f$ and let, for all $g \in \MF$, 
    \begin{equation}
        \Omega_t(g) = L_{t - 1}(g) + \eta \omega(g).
    \end{equation}
    Note that the process $\Omega_t$ is a Gaussian process and, conditioning on $\Omega_t(f) = y$ remains a Gaussian process.  
    Then conditioned on $\Omega_t(f) = y$, $\Omega_t$ has mean
    \begin{equation}
        m_{f,y}(g) = L_{t-1}(g) + \frac{\ee_{X \sim \mu}\left[f(X) g(X)\right]}{\norm{f}_{L^2(\mu)}^2} \left(y - L_{t-1}(f)\right)
    \end{equation}
    and covariance $\Sigma^f$; critically, $\Sigma^f$ does not depend on $y$.  Let
    \begin{align}
        \gamma(g) = \frac{4(L+2\erma)}{\alpha^2 \norm{f}_{L^2(\mu)}^2} \ee_{X \sim \mu}[f(X)g(X)] && \beta(g) = \frac {4(L+2\erma)}{\alpha^2 \norm{f}_{L^2(\mu)}} - \gamma(g)
    \end{align}
Then we have
\begin{equation}
  \label{eq:mean-shift}
        m_{f, y+ \frac {4(L+2\erma)}{\alpha^2}} = m_{f, y} + \gamma.
    \end{equation}
    Now suppose that $\norm{g - f}_{L^2(\mu)} > \alpha$, i.e., $g \in \far{\alpha}{f}$.  Then
    \begin{align}
        \ee_{X \sim \mu}\left[f(X) g(X)\right] = \frac{\norm{f}_{L^2(\mu)}^2 + \norm{g}_{L^2(\mu)}^2 }{2} - \frac 12 \norm{f - g}_{L^2(\mu)}^2 \leq 1 - \frac{\alpha^2}{2}
    \end{align}
    by $\norm{f}_\infty \leq 1$ for all $f \in \F$.  Thus for all such $g$,
    \begin{equation}
        \beta(g) = \frac {4(L+2\erma)}{\alpha^2 \norm{f}_{L^2(\mu)}^2} \left(1 - \ee_{X \sim \mu}[f(X)g(X)]\right) \geq \frac{2(L+2\erma)}{\norm{f}_{L^2(\mu)}^2} \geq 2(L+2\erma).
    \end{equation}
    We now fix $y$ and note that 
    \begin{align}
      & \pp\left(\sel{\omega} =f \aand \inf_{g \in \far{\alpha}{f}} \Omega_{t}(g) - 2(L+2\erma)\geq y-\erma\ |\ \Omega_t(f) = y \right) \\
      &\geq \pp\left(\sel{\omega}=f\aand \inf_{g \in \far{\alpha}{f}} \Omega_t(g) - \beta(g) \geq y-\erma \ | \ \Omega_t(f) = y \right) \\
        &= \pp\left(\sel{\omega}=f\aand\inf_{g \in \far{\alpha}{f}} \Omega_t(g) - \beta(g) - \gamma(g) + \gamma(g)  \geq y-\erma\ | \ \Omega_t(f) = y\right) \\
        &= \pp\left(\sel{\omega}=f\aand\inf_{g \in \far{\alpha}{f}} \Omega_t(g) + \gamma(g)  \geq y-\erma + \frac {4(L+2\erma)}{\alpha^2 \norm{f}_{L^2}(\mu)^2}\ | \ \Omega_t(f) = y\right) \\
        &= \pp\left(\sel{\omega}=f\aand\inf_{g \in \far{\alpha}{f}} \Omega_t(g)\geq y-\erma + \frac {4(L+2\erma)}{\alpha^2 \norm{f}_{L^2(\mu)}^2}\ | \ \Omega_t(f) = y+ \frac {4(L+2\erma)}{\alpha^2 \norm{f}_{L^2(\mu)}^2}\right)
    \end{align}
    where the inequality follows from the control of $\chi_B$ by $\beta$, the second equality follows from $\gamma + \beta = \frac {4(L+2\erma)}{\alpha^2 \norm{f}_{L^2(\mu)}^2}$, and the last equality follows from the fact that a Gaussian process is determined only by its covariance and mean (in particular, we are using \eqref{eq:mean-shift}). 

    Note that $L_{t-1}(f) + \eta \omega(f)$ is a Gaussian random variable with mean $L_{t-1}(f)$ and variance $\eta ^2 \norm{f}_{L^2(\mu)}^2$.  Denote by $q_f(y)$ the density of this distribution with respect to the Lebesgue measure on $\rr$.    Now, we compute 
    \begin{align}
      & \pp(\abs{A_t} = 0)\\
      &= \sum_{f \in \F} \int_{-\infty}^\infty q_f(y) \cdot \pp\left(\sel{\omega} =f \aand \inf_{g \in \far{\alpha}{f}} \Omega_t(g) - 2 (L+2\erma) \geq y-\erma\ |\ \Omega_t(f) =y \right) dy\\
        &\geq \sum_{f \in \F}   \int_{-\infty}^\infty q_f(y) \cdot \pp\left(\sel{\omega} =f \aand \inf_{g \in \far{\alpha}{f}} \Omega_t(g) \geq y-\erma + \frac {4(L+2\erma)}{\alpha^2 \norm{f}_{L^2(\mu)}^2}  \ | \  \Omega_t(f) = y + \frac{4(L+2\erma)}{\alpha^2 \norm{f}_{L^2(\mu)}^2} \right) dy\\
        &= \sum_{f \in \F}  \int_{-\infty}^\infty  q_f(y) \cdot  \pp\left(\sel{\omega} =f \aand \inf_{g \in  \far{\alpha}{f}} \Omega_t(g) \geq y-\erma \ | \ \Omega_t(f) = y \right)dy \\
        & + \sum_{f \in \F}  \int_{-\infty}^\infty  q_f(y) \cdot  \pp\left(\sel{\omega} =f \aand \inf_{g \in \far{\alpha}{f}} \Omega_t(g) \geq y-\erma+ \frac{4(L+2\erma)}{\alpha^2 \norm{f}_{L^2(\mu)}^2} \ | \ \Omega_t(f) = y + \frac{4(L+2\erma)}{\alpha^2 \norm{f}_{L^2(\mu)}^2} \right) dy\\
        & - \sum_{f \in \F} \int_{-\infty}^\infty  q_f(y) \cdot  \pp\left(\sel{\omega} =f \aand \inf_{g \in  \far{\alpha}{f}} \Omega_t(g) \geq y-\erma \ | \ \Omega_t(f) = y \right)dy.
    \end{align}
    For the first term, we have
    \begin{align}
      & \int_{-\infty}^\infty q_f(y) \cdot  \pp\left(\sel{\omega}=f \aand \inf_{g \in \far{\alpha}{f}} \Omega_t(g) \geq y-\erma \ | \ \Omega_t(f) = y \right)dy  \\
      & = \int_{-\infty}^\infty q_f(y) \cdot \pp \left( f_t(\omega) = f \ | \ \Omega_t(f) = y \right)dy = \pp\left(\sel{\omega} = f\right),
    \end{align}
    where the first equality above follows since, conditioned on $\Omega_t(f) = y$, $\sel{\omega}$ implies that $\inf_{g \in \far{\alpha}{f}} \Omega_t(g) \geq \inf_{g \in \MF} \Omega_t(g) \geq y-\erma$. 
Hence
    \begin{align}
       \sum_{f \in \F}  \int_{-\infty}^\infty  q_f(y) \cdot  \pp\left(\sel{\omega} =f \aand \inf_{g \in  \far{\alpha}{f}} \Omega_t(g) \geq y-\erma \ | \ \Omega_t(f) = y \right)dy = \sum_{f \in \MF} \pp(\sel{\omega} = f ) = 1\nonumber.
    \end{align}
    For the second and third terms, using that
      \begin{align}
q_f(z) = \frac{1}{\sqrt{2\pi \cdot \eta^2 \norm{f}_{L^2(\mu)}^2 }} \cdot \exp \left( - \frac{1}{2} \cdot \frac{(z - L_{t-1}(f))^2}{\eta^2 \norm{f}_{L^2(\mu)}^2 }\right)\nonumber,
      \end{align}
we observe that
    \begin{align}
        q_f(y) - q_f\left(y - \frac{4(L+2\erma)}{\alpha^2 \norm{f}_{L^2(\mu)}^2}\right) &= q_f(y) \left(1 - \exp\left(\frac{\left(y - L_{t-1}(f)\right)^2}{2 \eta^2 \norm{f}_{L^2(\mu)}^2} -  \frac{\left(y - L_{t-1}(f) - \frac{4(L+2\erma)}{\alpha^2 \norm{f}_{L^2(\mu)}^2}\right)^2}{2 \eta^2 \norm{f}_{L^2(\mu)}^2}\right)\right) \\
        &\leq \frac{q_f(y)}{2 \eta^2 \norm{f}_{L^2(\mu)}^2} \left( \left(y - L_{t-1}(f) - \frac{4(L+2\erma)}{\alpha^2 \norm{f}_{L^2(\mu)}^2}\right)^2 -  \left(y - L_{t-1}(f)\right)^2\right) \\
        &= \frac{q_f(y)}{2 \eta^2 \norm{f}_{L^2(\mu)}^2} \left(\frac{16(L+2\erma)^2}{\alpha^4 \norm{f}_{L^2(\mu)}^4} - \frac{8(L+2\erma)}{\alpha^2 \norm{f}_{L^2(\mu)}^2} \left(y - L_{t- 1}(f)\right) \right)
    \end{align}
    Thus we have
    \begin{align}
      & - \sum_{f \in \F}  \int_{-\infty}^\infty  q_f(y) \cdot  \pp\left(\sel{\omega} =f \aand \inf_{g \in \far{\alpha}{f}} \Omega_t(g) \geq y-\erma+ \frac{4(L+2\erma)}{\alpha^2 \norm{f}_{L^2(\mu)}^2} \ | \ \Omega_t(f) = y + \frac{4(L+2\erma)}{\alpha^2 \norm{f}_{L^2(\mu)}^2} \right) dy\\
        &\qquad + \sum_{f \in \F} \int_{-\infty}^\infty  q_f(y) \cdot  \pp\left(\sel{\omega} =f \aand \inf_{g \in  \far{\alpha}{f}} \Omega_t(g) \geq y-\erma \ | \ \Omega_t(f) = y \right)dy \\
        &= \sum_{f \in \F} \int_{-\infty}^\infty \left(q_f(y) - q_f\left(y - \frac{4(L+2\erma)}{\alpha^2 \norm{f}_{L^2(\mu)}^2}\right)\right) \pp\left(\sel{\omega} = f \aand \inf_{g \in \far{\alpha}{f}} \Omega_t(g) \geq y-\erma \ | \ \Omega_t(f) = y \right) d y \\
        &\leq \sum_{f \in \F} \int_{-\infty}^\infty \frac{q_f(y)}{2 \eta^2 \norm{f}_{L^2(\mu)}^2} \left(\frac{16(L+2\erma)^2}{\alpha^4 \norm{f}_{L^2(\mu)}^4} - \frac{8(L+2\erma)}{\alpha^2 \norm{f}_{L^2(\mu)}^2} \left(y - L_{t- 1}(f)\right) \right) \pp\left(\sel{\omega} = f | \Omega_t(f) = y \right) d y \\
        &\leq \frac{8(L+2\erma)^2}{\alpha^4\eta^2 \inf_{f \in \F}\norm{f}_{L^2(\mu)}^6} \sum_{f \in \F} \pp(f_t = f) \\
        &\quad - \frac{4 (L + 2 \zeta)}{\alpha^2 \eta^2} \sum_{f \in \F} \int_{-\infty}^\infty \frac{y - L_{t-1}(f)}{\norm{f}^4} \pp\left(f_t(\omega) = f | \Omega(f_t) = y\right)q_f(y) d y 
        \\
        \\
        &= \frac{8(L+2\erma)^2}{\alpha^4\eta^2 \inf_{f \in \F}\norm{f}_{L^2(\mu)}^6} -\frac{4 (L + 2 \zeta)}{\alpha^2 \eta^2} \sum_{f \in \F} \int_{-\infty}^\infty \frac{y - L_{t-1}(f)}{\norm{f}^4} \pp\left(f_t(\omega) = f | \Omega(f_t) = y\right)q_f(y) d y \\
        &\stackrel{\hypertarget{eqstar}{(\ast)}}{\leq} \frac{8(L+2\erma)^2}{\alpha^4\eta^2 \inf_{f \in \F}\norm{f}_{L^2(\mu)}^6} + \frac{4(L+2\erma)}{\alpha^2 \eta \inf_{f \in \F} \norm{f}_{L^2(\mu)}^4} \ee\left[\sup_{f\in \F} \omega(f)\right],
    \end{align}
    where the first inequality uses the previous computation, the second follows by linearity, and the last equality follows because $f_t \in \F$ is distinct.  To see that inequality \hyperlink{eqstar}{$(\ast)$} holds, note that, by definition, $(y - L_{t-1}(f)) \stackrel{d}{=} \eta \omega(f)$ and thus,
    \begin{align}
      \frac{1}{\eta} \sum_{f \in \F} \int_{-\infty}^\infty \frac{y - L_{t-1}(f)}{\norm{f}^4} \pp\left(f_t(\omega) = f | \Omega(f_t) = y\right)q_f(y) d y = \ee\left[\frac{\omega(f_t)}{\norm{f_t}^4}\right]
    \end{align}
    We now have
    \begin{align}
      - \ee\left[\frac{\omega(f_t)}{\norm{f_t}^4}\right] &\leq - \ee\left[\inf_{f \in \F} \frac{\omega(f)}{\norm{f}^4}\right] \\
      &\stackrel{(a)}{=} \abs{\ee\left[- \inf_{f \in \F} \frac{\omega(f)}{\norm{f}^4}\right]} \\
      &\stackrel{(b)}{\leq} \ee\left[\abs{- \inf_{f \in \F} \frac{\omega(f)}{\norm{f}^4}  }\right] \\
      &\stackrel{(c)}{=}\ee\left[\sup_{f \in \F} \frac{\abs{\omega(f)}}{\norm{f}^4}\right] \\
      &\stackrel{(d)}{\leq} \frac{1}{\inf_{f \in \F} \norm{f}^4} \ee\left[\sup_{f \in \F} \abs{\omega(f)}\right] \\
      &\stackrel{(e)}{\leq} \frac{1}{\inf_{f \in \F} \norm{f}^4} \ee\left[\sup_{f \in \F} \omega(f)\right]
    \end{align}
    Note that by Jensen's inequality,
    \begin{equation}
      \ee\left[\inf_{f \in \F} \frac{\omega(f)}{\norm{f}^4}\right] \leq \inf_{f \in \F} \ee\left[\frac{\omega(f)}{\norm{f}^4}\right] = 0
    \end{equation}
    and so (a) holds.  Then (b) follows from Jensen's inequality, (c) follows from the symmetry of the Gaussian, (d) follows by linearity, and (e) follows from the Sudakov-Fernique inequality \citep{sudakov1971gaussian,fernique1975regularite} applied to the contraction $\abs{\cdot}$.  The result follows.
  \end{proof}
  Note that Lemma \ref{lem:stabilityftplprob} holds for an arbitrary measure $\mu$ on $\cX$ and applies even in the case where $x_t, w_t$ are adversarially chosen.  To apply this result, we choose $\mu$ to be the empirical measure on the perturbation samples $X_{t,i}$.  We consider two cases.  First, we suppose that we are in a classification setting, where we get better rates.  We then prove the more general setting where $\F$ is real-valued.
\begin{lemma}\label{lem:stabilityftplclassification}
    Suppose that we are in the smoothed online setting, $f_t$ is chosen so as to satisfy \eqref{eq:ft-hat-define}, and the empirical distribution $\hat \mu_{t,n}$ satisfies 
  \begin{align}
\sup_{f,f' \in \MF} \left| \norm{f-f'}_{L^2(\mu)}^2 - \norm{f-f'}_{L^2(\hat \mu_{t,n})}^2 \right| \leq \Delta\label{eq:delta-asm}.
  \end{align}
  Suppose further that for all $f \in \F$ and all $x \in \cX$, $f(x) \in \{\pm 1\}$.  Then
    \begin{align}
      \ee\left[\ell(f_t(x_t), y_t) - \ell(f_{t+1}(x_t'), y_t') \right] \leq & \frac{30 (L+2\erma)^3 \log \eta}{\sigma \eta} \ee\left[1 + \sup_{f \in \F} \hat\omega_{t,n}(f)\right] + \frac{2L \Delta}{\sigma}
    \end{align}
\end{lemma}
\begin{proof}
  Let $(x_t', y_t') \in \MX \times [-1,1]$ be a sample distributed independently and identically to $(x_t, y_t)$ conditioned on $(x_1, y_1), \ldots, (x_{t-1}, y_{t-1}), f_1, \ldots, f_{t-1}$. Since $f_t$ is selected independently of $(x_t, y_t)$, it is immediate that
  \begin{align}
\E[\ell(f_t(x_t), y_t)] = \E[\ell(f_t(x_t'), y_t')].\label{eq:ghost-sample-eq}
  \end{align}
  Therefore, it suffices to bound
  \begin{align}
\E[\ell(f_t(x_t'), y_t') - \ell(f_{t+1}(x_t'), y_t')].\label{eq:ft-xprimes}
  \end{align}

    Let us now fix any values of $S := \{(x_1, y_1), \ldots, (x_{t-1}, y_{t-1}), f_1, \ldots, f_{t-1}\}$. By Lemma \ref{lem:stabilityftplprob}, there is a joint distribution $\nu$ over $(f_t', f_{t+1}')$, so that, conditioned on $S$:
  \begin{enumerate}
  \item The marginal distribution of $f_t'$, conditioned on $S$, is equal to the marginal distribution of $f_{t}$, conditioned on $S$.
  \item  The marginal distribution of $f_{t+1}'$, conditioned on $S$, is equal to the marginal distribution of $f_{t+1}'$, conditioned on $S$.
  \item It holds that
    \begin{align}
  \hspace{-1cm}      \pp_\nu\left( \norm{f_t' - f_{t+1}'}_{L^2(\mu)} > \alpha\right) \leq \frac{8(L+2\erma)^2}{\alpha^4\eta^2 \inf_{f \in \F}\norm{f}_{L^2(\hat \mu_{t,n})}^6} + \frac{4(L+2\erma)}{\alpha^2 \eta\inf_{f \in \F}\norm{f}_{L^2(\mu)}^4} \ee\left[\sup_{f\in \F} \hat\omega_{t,n}(f)\right].\label{eq:nu-wasserstein}
    \end{align}
  \end{enumerate}
  In particular, this joint distribution $\nu$ is constructed by setting $f_t'$ to equal $f_t$ from \eqref{eq:ft-hat-define} and then defining $f_{t+1}'$ so that
  \begin{align}
     L_{t-1}(f_{t+1}')+ \ell(f_{t+1}'(x_t), y_t) + \eta \hat \omega_{t,n}(f_{t+1}') \leq \argmin_{f \in \MF} L_{t-1}(f) + \ell(f(x_t), y_t) +\eta \cdot \hat \omega_{t,n}(f) + \zeta.
  \end{align}
  Note that $\hat \omega_{t,n}$ has been used here as opposed to $\hat \omega_{t+1,n}$. Since $\hat \omega_{t,n}$ and $\hat \omega_{t+1,n}$ have the same distribution, the first two requirements of $\nu$ above are immediate. To see that the third holds, we note that, in the notation of Lemma \ref{lem:stabilityftplprob}, $f_{t+1}'$ is exactly $f_{t+1,x,w}$ with $x = x_t, w = y_t$, and thus \eqref{eq:nu-wasserstein} is immediate from \eqref{eq:bound-wasserstein} (with the distribution $\mu$ set to $\hat \mu_{t,n}$ and the Gaussian process $\omega$ set to $\hat \omega_{t,n}$).

By the first two conditions above of the coupling $\nu$ and since $(x_t', y_t')$ is drawn independently from $(f_t', f_{t+1}')$, it holds that $\E[\ell(f_t(x_t'), y_t')] = \E[\ell(f_t'(x_t'), y_t')]$ and $\E[\ell(f_{t+1}(x_t'), y_t')] = \E[\ell(f_{t+1}'(x_t'), y_t')]$.

Fix any $0 < \beta < \alpha$.    By $L$-Lipschitzness of $\ell$, the fact that $f_t', f_{t+1}' \in \{\pm 1\}$, and the fact that $(x_t', y_t')$ are drawn independently from $f_t', f_{t+1}'$, we have
    \begin{align}
        \ee_{\nu,\ x_t' \sim p_t,\ y_t'}&\left[(\ell(f_t'(x_t'), y_t') - \ell(f_{t+1}'(x_t'), y_t')) \cdot \chi_{\beta \leq \norm{f_t' - f_{t+1}'}_{L^2(\hat \mu_{t,n})} \leq \alpha }\right] \label{eq:alpha-beta-loss}\\
        &\leq L \ee_{\nu,\ x_t' \sim p_t} \left[\abs{f_t'(x_t') - f_{t+1}'(x_t')}\cdot \chi_{\beta \leq \norm{f_t' - f_{t+1}'}_{L^2(\hat \mu_{t,n})} \leq \alpha }\right] \\
                                        &= L\ee_{\nu}\left[\E_{x_t' \sim p_t}[(f_t'(x_t') - f_{t+1}'(x_t'))^2 \condt f_t' ,f_{t+1}'] \cdot \chi_{\beta \leq \norm{f_t' - f_{t+1}'}_{L^2(\hat \mu_{t,n})} \leq \alpha }\right] \\
      \leq & \frac{L}{\sigma} \cdot \ee_{\nu}\left[\E_{x_t' \sim \mu}[(f_t'(x_t') - f_{t+1}'(x_t'))^2 \condt f_t', f_{t+1}'] \cdot \chi_{\beta \leq \norm{f_t' - f_{t+1}'}_{L^2(\hat \mu_{t,n})} \leq \alpha }\right] \\
        &\leq \frac{L \cdot (\alpha^2+\Delta)}{\sigma} \pp_\nu( \norm{f_t' - f_{t+1}'}_{L^2(\hat{\mu}_{t,n})} > \beta)
    \end{align}
    Set $S = \lceil \log \min \{ \eta, 1/\Delta \} \rceil$ and let $\alpha_i = 2^{\frac{1 - i}{2}}$.  Then, noting that $\norm{f}_{L^2(\hat \mu_n)} = 1$ for all $f \in \F$, we see, using \eqref{eq:nu-wasserstein},
    \begin{align}
      &\ee_{x_t' \sim p_t}\left[\ell(f_t(x_t'), y_t') - \ell(f_{t+1}(x_t'), y_t')\right] \\
      &= \ee_{x_t' \sim p_t} \left[ \ell(f_t'(x_t'), y_t') - \ell(f_{t+1}'(x_t'), y_t') \right] \\
         &\leq \ee_{x_t' \sim p_t} \left[ \left(\ell(f_t'(x_t'), y_t') - \ell(f_{t+1}'(x_t'), y_t')\right) \cdot  \chi_{ \norm{f_t' - f_{t+1}'}_{L^2(\hat \mu_{t,n})} \leq \alpha_S} \right] \\
         &\quad + \sum_{i = 0}^S  \ee_{x_t' \sim p_t} \left[ \left(\ell(f_t'(x_t'), y_t') - \ell(f_{t+1}'(x_t'), y_t')\right) \cdot \chi_{ \alpha_i < \norm{f_t' - f_{t+1}'}_{L^2(\hat \mu_{t,n})} \leq \sqrt{2} \alpha_i} \right] \\
         &\leq \frac{L (\alpha_S^2+\Delta)}{\sigma} + \sum_{i = 0}^S \left(\frac{8(L+2\erma)^2}{\alpha_i^4\eta^2 } + \frac{4(L+2\erma)}{\alpha_i^2 \eta } \ee\left[\sup_{f\in \F} \hat\omega_{t,n}(f)\right] \right) \frac{L (\alpha_i^2+\Delta)}{\sigma} \\
         &\leq \frac{4L}{\sigma} \cdot \left( \frac{1}{\eta} + \Delta \right) + \sum_{i = 0}^S \left(\frac{8(L+2\erma)^2 }{\alpha_i^2\eta^2 } + \frac{4(L+2\erma)}{ \eta} \ee\left[\sup_{f\in \F} \hat\omega_{t,n}(f)\right] \right) \frac{2L}{\sigma} \\
         &\leq \frac{4L}{\sigma} \cdot \left( \frac{1}{\eta} + \Delta \right) + \sum_{i = 0}^S \left(\frac{8(L+2\erma)^2}{\eta } + \frac{4(L+2\erma)}{ \eta } \ee\left[\sup_{f\in \F} \hat\omega_{t,n}(f)\right] \right) \frac{2L}{\sigma} \\
         & \leq \frac{30 (L+2\erma)^3 \log \eta}{\sigma \eta} \ee\left[1 + \sup_{f \in \F} \hat\omega_{t,n}(f)\right] + \frac{2L\Delta}{\sigma},
    \end{align}
    where the second inequality follows by the above argument (setting $\beta = 0$ for the first term) and from \eqref{eq:nu-wasserstein}, the third inequality follows from $\Delta \leq \alpha_i^2$ for all $i \leq S$, the penultimate inequality follows from $\frac 1{\alpha_i^2} \leq \frac 1\eta$ for $i \leq S$ and the last inequality follows from $S \leq \log \eta$.  The result follows from the above display and \eqref{eq:ghost-sample-eq}.
  \end{proof}  
We now prove a more general result that has worse dependence on $\eta$.
\begin{lemma}\label{lem:stabilityftplregression}
    Suppose that we are in the smoothed online setting, $f_t$ is chosen so as to satisfy \eqref{eq:ft-hat-define}, and the empirical distribution $\hat \mu_{t,n}$ satisfies 
  \begin{align}
\sup_{f,f' \in \MF} \left| \norm{f-f'}_{L^2(\mu)}^2 - \norm{f-f'}_{L^2(\hat \mu_{t,n})}^2 \right| \leq \Delta\label{eq:delta-asm-reg}.
  \end{align}
  Suppose further that $\inf_{f \in \MF} \norm{f}_{L^2(\hat \mu_{t,n})}^2 \geq 1/2$. 
  Then
    \begin{align}
      \ee\left[\ell(f_t(x_t), y_t) - \ell(f_{t+1}(x_t'), y_t')\right] \leq &  \frac{1200(L+2\erma)^3 \log \eta}{ \sqrt{\sigma\eta}} \ee\left[1 + \sup_{f \in \F} \hat\omega_{t,n}(f)\right] + 4L \cdot \sqrt{\frac{\Delta}{\sigma}}
    \end{align}
\end{lemma}
\begin{proof}
  Exactly as in the proof of Lemma \ref{lem:stabilityftplclassification}, we introduce the independent sample $(x_t', y_t')$, as well as the coupling $\nu$ over $(f_t', f_{t+1}')$. In particular, \eqref{eq:ghost-sample-eq} and \eqref{eq:nu-wasserstein} continue to hold.
  Next, we bound the expression in \eqref{eq:alpha-beta-loss} as in the proof of Lemma \ref{lem:stabilityftplclassification} but this time applying Jensen's inequality: 
    \begin{align}
        \ee_{x_t' \sim p_t}&\left[(\ell(f_t(x_t'), y_t') - \ell(f_{t+1}(x_t'), y_t'))\chi_{\beta \leq \norm{f_t - f_{t+1}}_{L^2(\mu_n)} \leq \alpha }\right] \\
        &\leq L \ee_{x_t' \sim p_t} \left[\abs{f_t(x_t) - f_{t+1}(x_t)}\chi_{\beta \leq \norm{f_t - f_{t+1}}_{L^2(\mu_n)} \leq \alpha }\right] \\
        &\leq L \pp\left(\sup_{x,y} \norm{f_t - f_{t+1, x, y}}_{L^2(\hat{\mu}_n)} > \beta\right)\sqrt{\ee_{x_t' \sim p_t}\left[(f_t(x_t') - f_{t+1}(x_t'))^2\chi_{\norm{f_t - f_{t+1}}_{L^2(\mu_n)} \leq \alpha }\right]} \\
        &\leq L \pp\left(\sup_{x,y} \norm{f_t - f_{t+1, x, y}}_{L^2(\hat{\mu}_n)} > \beta\right) \sqrt{\frac{\alpha^2 + \Delta}{\sigma}}. 
    \end{align}
    Setting $S = \lceil \log \min \{\sqrt \eta, 1/\sqrt\Delta \} \rceil$ and $\alpha_i = 2^{1 - i}$ for $0 \leq i \leq S$, we have:
    \begin{align}
        &\ee_{x_t' \sim p_t}\left[\ell(f_t(x_t'), y_t') - \ell(f_{t+1}(x_t'), y_t')\right] \\
        &\leq \ee_{x_t' \sim p_t} \left[\left( \ell(f_t(x_t'), y_t') - \ell(f_{t+1}(x_t'), y_t')\right)\cdot \chi_{ \norm{f_t - f_{t+1}}_{L^2(\hat\mu_n)} \leq \alpha_S} \right] \\
        &\quad + \sum_{i = 0}^S \pp\left(\norm{f_{t} - f_{t+1}}_{L^2(\hat \mu_n)} > \alpha_i \right) \ee_{x_t \sim p_t} \left[ \left(\ell(f_t(x_t'), y_t') - \ell(f_{t+1}(x_t'), y_t') \right) \cdot \chi_{ \norm{f_t - f_{t+1}}_{L^2(\hat \mu_n)} \leq {2} \alpha_i} \right] \\
        &\leq L \cdot \sqrt{\frac{\alpha_S^2 + \Delta}{\sigma}} + \sum_{i = 0}^S \left(\frac{8(L+2\erma)^2}{\alpha_i^4\eta^2 \inf_{f \in \F} \norm{f}_{L^2(\mu)}^6} + \frac{4(L+2\erma)}{\alpha_i^2 \eta \inf_{f \in \F} \norm{f}_{L^2(\mu)}^4} \ee\left[\sup_{f\in \F} \hat\omega_n(f)\right] \right) L \cdot \sqrt{\frac{\alpha_i^2 + \Delta}{\sigma}} \\
        &\leq L \cdot \sqrt{\frac{8(\Delta + 1/\eta)}{\sigma}} + \sum_{i = 0}^S \left(\frac{512(L+2\erma)^2}{\alpha_i^3 \eta^2 } + \frac{64(L+2\erma)}{\alpha_i \eta} \ee\left[\sup_{f\in \F} \hat\omega_n(f)\right] \right) L \cdot \sqrt{\frac{2}{\sigma}} \\
        &\leq  L \cdot \sqrt{\frac{8(\Delta + 1/\eta)}{\sigma}} + \sum_{i = 0}^S \left(\frac{512(L+2\erma)^2}{\sqrt{\eta} } + \frac{64(L+2\erma)}{ \sqrt{\eta} } \ee\left[\sup_{f\in \F} \hat\omega_n(f)\right] \right) L \cdot \sqrt{\frac{2}{\sigma}}  \\
        & \leq \frac{1200(L+2\erma)^3 \log \eta}{ \sqrt{\sigma\eta}} \ee\left[1 + \sup_{f \in \F} \hat\omega_n(f)\right] + 4L \cdot \sqrt{\frac{\Delta}{\sigma}},
   \end{align}
   where we used the fact that $\norm{f}_{L^2(\hat\mu_n)} \geq \frac 12$, $\alpha_i^2 \geq \Delta$, and $\frac 1{\alpha_i} \leq \sqrt{\eta}$ for all $i \leq S$.
 \end{proof}
Finally, we need to verify that $\norm{\cdot}_{L^2(\mu)}$ and $\norm{\cdot}_{L^2(\mu_n)}$ are close together, a key condition of Lemmas \ref{lem:stabilityftplclassification} and \ref{lem:stabilityftplregression}.  The below standard result shows that this condition holds in high probability.
\begin{lemma}
  \label{lem:f-l2mu}
There is a constant $C > 0$ so that the following holds.  Consider any distribution $\mu$ over $\MX$, suppose $x_1, \ldots, x_n \sim \mu$ are sampled independently, and define $\hat \mu_n := \frac{1}{n} \sum_{i=1}^n \delta_{x_i}$. For any $\delta > 0$, with probability at least $1-\delta$ over the $x_i$, we have
  \begin{align}
\sup_{f,f' \in \MF} \left| \norm{f-f'}_{L^2(\mu)}^2 - \norm{f-f'}_{L^2(\hat \mu_n)}^2 \right| \leq \frac{C}{\sqrt{n}} \cdot \left( \frac 1n\gc{n}{\MF} + \sqrt{\frac{\log\left( \frac 1\delta\right)}n }\right).
  \end{align}
\end{lemma}
\begin{proof}
 Write $\MF^2 = \{ x \mapsto f(x)^2 : f \in \MF \}$.  Standard results in empirical processes, such as \cite[Theorem 4.10]{wainwright2019high} guarantee that with probability at least $1 - \delta$,
  \begin{equation}
    \sup_{f,f' \in \F} \abs{\norm{f - f'}_{L^2(\mu)}^2 - \norm{f - f'}_{L^2(\mu_n)}^2} \leq C \left(\frac 1n \R_n(\F^2) + \sqrt{\frac{\log \left(\frac 1\delta\right)}n} \right)
  \end{equation}
  Noting that $\F$ has image in $[-1,1]$ and thus the square is $2$-Lipschitz, we may apply contraction and the bound of Rademacher complexity by Gaussian complexity to conclude the proof.
\end{proof}

\subsection{Bounding the generalization error}\label{subsec:generalizationerror}

In this section, we bound the final term in \eqref{eq:btl}.  This term was called the ``Generalization Error'' in \cite{haghtalab2022oracle} and our control of this quantity follows a similar general approach as their Lemma 4.5.  For our proof, we require the following variant of the coupling approach of Lemma \ref{lem:coupling}:
  \begin{lemma}[Lemma 4.6 of \cite{haghtalab2022oracle}]
    \label{lem:variant-coupling}
    Fix a distribution $\mu$ on a set $\MX$ and suppose that $p \in \p(\sigma, \mu)$. Suppose that $X_1, \ldots, X_m \sim \mu$ are iid. Then there is an external probability space with sample space $\Omega$ and measure $\nu$ which produces a sample $R \sim \nu$ so that the following holds. There is a measurable function $I : \MX^m \times \Omega \ra [m]$ so that, for some event $\ME = \ME(X_1, \ldots, X_m, R)$ with $\Pr(\ME) \geq 1-(1-\sigma)^m$, $(X_I | \ME, (X_i)_{i \neq I}) \sim p$ (in words, conditioned on the event $\ME$ and the value of any measurable function of $(x_i)_{i \neq I}$, $X_I$ has conditional distribution $p$).
  \end{lemma}
  We restrict our focus to linear loss $\ell(f(x), y) = y f(x)$. and provide the following bound:

  \begin{lemma}
    \label{lem:gen-error}
    Suppose that we are in the smoothed online setting. Fix any $t \in [T-1]$ and suppose that $f_{t+1}$ is chosen so as to satisfy (\ref{eq:ft-hat-define}), with the process $\hat \omega_{n,t}(\cdot)$ defined as in \eqref{eq:hat-omega}, and the parameters $\eta, n$ satisfy $\eta / \sqrt{n} \geq L$. Furthermore, let $(x_t', y_t')$ be an independent sample drawn from the conditional distribution of $(x_t, y_t)$ given $\{ (x_s, y_s) \}_{s \leq t-1}$ and $\{ f_s \}_{s \leq t-1}$. Then, for some constant $c_0 \in (0,1)$, it holds that
    \begin{align}
\E[\ell(f_{t+1}(x_t'), y_t') - \ell(f_{t+1}(x_t), y_t)] \leq  4L \cdot \frac{ \log T}{c_0 \sigma n} \cdot \rc{c_0 \sigma n / (2 \log T)}{\MF} + 2\zeta + \frac{2Ln\sigma }{T^2}.
    \end{align}
\end{lemma}
\begin{proof}
  Fix any realization of $(x_1, y_1), \ldots, (x_{t-1}, y_{t-1}), f_1, \ldots, f_{t-1}$. Recalling the definition of smoothed adversary, let $p_t$ denote the conditional distribution of $x_t$ (which is the same as the conditional distribution of $x_t'$) given $(x_1, y_1), \ldots, (x_{t-1}, y_{t-1})$. Also let $q_t(\cdot | x_t)$ denote the conditional distribution of $y_t$ given $x_t$ (and conditioned on the fixed values of $(x_s, y_s),\ s  < t$, which are omitted for clarity). Recall that we make no smoothness assumption on $q_t$. We denote the distribution of $(x_t, y_t)$, where $x_t \sim p_t$ and $y_t \sim q_t(\cdot | x_t)$ as $p_t \odot q_t$. Furthermore let $\tilde p_t$ denote the conditional distribution of $(x_t, \sign(y_t))$ given $(x_1, y_1), \ldots, (x_{t-1}, y_{t-1})$ (i.e., where $x_t \sim p_t$ and $y_t \sim q_t(\cdot | x_t)$). Set $c_0 := \Pr_{\gamma \sim \MN(0,1)} (\gamma \geq 1) > 0$, where $\MN(0,1)$ is the standard normal distribution. Let $\tilde \mu \in \Delta(\MX \times \{-1,0,1\})$ denote the product of $\mu$ and the distribution over $\{-1,0,1\}$ which puts mass $c_0$ on $1,-1$ and mass $1-2c_0$ on $0$.  For any measurable subset $\MA \subset \MX$ and and $b \in \{-1,1\}$, we have, from $\sigma$-smoothness of $p_t$ that for each $b \in \{-1,1\}$,
  \begin{align}
\frac{\tilde p_t(\MA \times \{ b \})}{\tilde \mu(\MA \times \{ b \})} = \frac{p_t(\MA) \cdot \Pr_{(x,y) \sim \tilde p_t}(y = b | x \in \MA)}{\mu(\MA) \cdot c_0} \leq \frac{1}{c_0} \cdot \frac{p_t(\MA)}{\mu(\MA)} \leq \frac{1}{c_0\sigma}\nonumber,
  \end{align}
  meaning that $\tilde p_t \in \p(c_0\sigma, \tilde \mu)$.

  Define the function $\thr : \BR \ra \{-1,0,1\}$ as follows:
  \begin{align}
    \thr(y) := \begin{cases}
      1 &: y \geq 1 \\
      0 &: y \in (-1,1) \\
      -1 &: y \leq -1.
      \end{cases}
  \end{align}
  Recall the i.i.d.~samples $(x_i, \gamma_i)$, $i \in [n]$ defining the process $\hat \omega_n(\cdot)$ in \eqref{eq:hat-omega}; note that $(x_i, \thr(\gamma_i)) \sim \tilde \mu$ by the definition of $\tilde \mu$ and $\thr$. For $i \in [n]$ set $z_i = (x_i, \thr(\gamma_i))$. 
  Fix $m = 2 \log T \cdot \frac{1}{c_0 \sigma}$. We divide the i.i.d.~sample $(z_1, \ldots, z_n)$ into $n/m$ groups of $m$ samples each: the first group consists of $(z_1, \ldots, z_m)$, the second consists of $(z_{m+1}, \ldots, z_{2m})$, and so on. By Lemma \ref{lem:variant-coupling}, for each group index $0 \leq j < n/m$, letting $\Omega_j$ denote the sample space of the external probability space in the statement (of Lemma \ref{lem:variant-coupling}) and $R_j \in \Omega_j$ denote the corresponding random variable, there is a function $I_j : (\MX \times \{-1,0,1\})^m \times \Omega_j$ so that for some event $\ME_j = \ME_j(z_{jm +1}, \ldots, z_{jm+m}, R_j)$ occuring with probability at least $1- (1-c_0 \sigma)^m$, letting $I_j =I_j(z_{jm+1}, \ldots, z_{jm+m}, R_j)$,
  \begin{align}
(z_{I} | \ME_j, (z_{jm+i} : i \neq I)) \sim p_t\nonumber.
  \end{align}
  In particular, we have applied Lemma \ref{lem:variant-coupling} with $\mu$ set to $\tilde \mu$ and $p$ set to $\tilde p_t$. 
  Write $\ME := \cap_{0 \leq j < n/m} \ME_j$, so that $\Pr(\ME) \geq 1 - (n/m) \cdot (1-c_0 \sigma)^m$. Let $\MI \in [n]^{n/m}$ be the (random) vector defined as $\MI = (I_0, \ldots, I_{n/m-1})$, and let $\bar \MI \in [n]^{n-(n/m)}$ be the vector defined as $(i \in [n] : i \not \in \MI)$.  Since the individual groups $(z_1, \ldots, z_m), (z_{m+1}, \ldots, z_{2m}), \ldots$ are mutually independent, it follows that
  \begin{align}
((z_i)_{i \in \MI} | \ME, (z_i : i \in \bar \MI)) \sim \tilde p_t^{\otimes m}\label{eq:pt-iid},
  \end{align}
  i.e., conditioned on all $z_i,\ i \in \bar \MI$, the distribution of $z_i,\ i \in \MI$ is i.i.d.~according to $\tilde p_t$. Let us write the vector $(z_i)_{i \in \MI}$ as $w \in (\MX \times \{-1,1\})^{n/m}$. 
  Note that by independence of $\hat{\omega}_{n,t}$ across $t$, the distribution of $(x_t, y_t)$ is independent of $z_1, \dots, z_n, \ME, \MI$; further, $x_t \sim p_t$ and $y_t \sim q_t(\cdot | x_t)$.

  For each $i \in [n]$, let $\hat y_i \in \BR$ denote an independent sample from $q_t(\hat y_i | x_i)$ conditioned on $\sign(\hat y_i) = \sign(\gamma_i)$. 
  Recalling the definition of the (random) index $I_j^{}$ above, we have $z_{I_j^{}} = (x_{I_j^{}}, \thr(\gamma_{I_j^{}}))$. 
  Recalling that $(x_{I_j^{}}, \thr(\gamma_{I_j^{}})) \sim \tilde p_t$ conditioned on $\ME, (z_i : i \in \bar \MI)$ (which follows from \eqref{eq:pt-iid}), which in particular means that $\sign(\gamma_{I_j^{}}) = \thr(\gamma_{I_j^{}}) \in \{-1,1\}$, it follows that $(x_{I_j^{}}, \hat y_{I_j^{}})$ has the same distribution as $(x_t, y_t)$ (namely, $p_t \odot q_t$) and both are independent, conditioned on $((z_i : i \neq I_j^{}), \ME)$. In particular, conditioned on the event $\ME$, the distributions of the following vectors in $(\MX \times \BR)^{n+1}$ are the same:
  \begin{align}
    ((x_t, y_t), (x_{I_j^{}}, \hat y_{I_j^{}}), (z_i : i \neq I_j^{})) \stackrel{d}{=} ((x_{I_j^{}}, \hat y_{I_j^{}}), (x_t, y_t), (z_i : i \neq I_j^{})),
  \end{align}
  where $\stackrel{d}{=}$ denotes equality in distribution and the above notation means that the entries $(z_i : i \neq I_j^{})$ are concatenated to the other entries. 
  Since the values of $\gamma_i$, $i \in [n]$ are independent and identically distributed conditioned on $(z_1, \ldots, z_n)$, 
  it follows that conditioned on the event $\ME$, the distributions of the following vectors are the same:
  \begin{align}
    ((x_t, y_t), (x_{I_j^{}}, \hat y_{I_j^{}}), ((x_i, \gamma_i) : i \neq I^{}_j)) \stackrel{d}{=} ((x_{I_j^{}}, \hat y_{I_j^{}}), (x_t, y_t), ((x_i, \gamma_i) : i \neq I^{}_j)).
  \end{align}
  Furthermore, note that under the event $\ME$, we have that $\gamma_{I^{}_j} = |\gamma_{I^{}_j}| \cdot \thr(\gamma_{I_j^{}}) = |\gamma_{I^{}_j}| \cdot \sign(\hat y_{I_j^{}})$ and $|\gamma_{I^{}_j}| \geq 1$ (as $\thr(\gamma_{I_j^{}}) \in \{-1,1\}$ under the event $\ME$). From \eqref{eq:ft-hat-define}, $f_{t+1}$ is a $\zeta$-approximate minimizer (among $f \in \MF$) of
  \begin{align}
    & \sum_{s=1}^t \ell(f(x_s), y_s) + \sum_{i=1}^n \frac{\eta \gamma_i}{\sqrt n} \cdot f(x_i)\\
    =&  y_t \cdot f(x_t) + \frac{\eta \cdot |\gamma_{I^{}_j}|}{\sqrt n} \cdot \sign(\hat y_{I_j^{}}) \cdot f(\hat x) + \sum_{i \neq I^{}_j} \frac{\eta \gamma_i}{\sqrt n} \cdot f(x_i) + \sum_{s=1}^{t-1} y_s \cdot f(x_s).
  \end{align}
  Since $\eta / \sqrt{n} \geq L$ by assumption, it follows from Lemma \ref{lem:switch-sample} that $\E[y_t \cdot f_{t+1}(x_t) \condt \ME] \geq \E[\hat y_{I_j^{}} \cdot f_{t+1}(x_{I_j^{}}) \condt \ME] - 2\zeta$. 

  Further, letting $(x_{t,1}', y_{t,1}'), \ldots, (x_{t,n/m}', y_{t,n/m}')$ denote i.i.d.~samples from the distribution of $(x_t, y_t)$ (independent of $(x_t, y_t)$), it is immediate that for all $0 \leq j < n/m$,
  \begin{align}
\E[y_{t,j}' \cdot f_{t+1}(x_{t,j}') \condt \ME] = \E[y_t' \cdot f_{t+1}(x_t') \condt \ME].
  \end{align}
  Then it follows that
  \begin{align}
    & \frac{n}{m} \cdot \E[y_t' \cdot f_{t+1}(x_t') - y_t \cdot f_{t+1}(x_t) - 2\zeta \condt \ME]\nonumber\\
    \leq & \E \left[ \sum_{j=0}^{n/m-1} y_{t,j}' \cdot f_{t+1}(x_{t,j}') - \sum_{j=0}^{n/m-1} \hat y_{I_j^{}} \cdot f_{t+1}(x_{I_j^{}}) \condt \ME \right]\nonumber\\
    \leq & \E \left[ \sup_{f \in \MF} \sum_{j=0}^{n/m-1} y_{t,j}' \cdot f (x_{t,j}') - \hat y_{I_j^{}} \cdot f(x_{I_j^{}}) \condt \ME\right]\nonumber\\
    \leq & \E_{(\bar x_j, \bar y_j), (\bar x_j', \bar y_j') \sim p_t \odot q_t \ :\ 0 \leq j < n/m} \left[ \sup_{f \in \MF} \sum_{j=0}^{n/m-1} \bar y_j' \cdot f(\bar x_j') - \bar y_j \cdot f(\bar x_j) \right]\label{eq:use-resampling}\\
    \leq & 2 \cdot \E_{(\bar x_j, \bar y_j) \sim p_t \odot q_t, \ \ep_j \sim \mathrm{Unif}(\pm 1) \ : \ 0 \leq j < n/m} \left[ \sup_{f \in \MF} \sum_{j=0}^{n/m-1} \ep_j \cdot \bar y_j \cdot f(\bar x_j)\right]\\
    \leq & 2L \cdot \rc{n/m}{\MF}\label{eq:contraction-bary}.
  \end{align}
  where \eqref{eq:use-resampling} follows since, conditioned on $\ME$, $(x_{t,j}', y_{t,j'}), (x_{I_j^{}},\hat y_{I_j^{}})$, $0 \leq j < n/m$ are all mutually independent (here we are using \eqref{eq:pt-iid} as well as the definition of the labels $\hat y_i$). Furthermore, in \eqref{eq:contraction-bary} above, we are using the contraction inequality for Rademacher complexity.

  By our choice of $m = 2 \log T \cdot \frac{1}{c_0 \sigma}$, we have that $\Pr(\ME) \geq  1 - (n/m) \cdot (1 - c_0 \sigma)^m \geq 1 - (n/m) \cdot \exp(-c_0 \sigma m) \geq 1 - \frac{n}{mT^2} \geq 1 - \frac{n\sigma}{T^2}$. Then we see that
  \begin{align}
\E[y_t' \cdot f_{t+1} (x_t') - y_t \cdot f_{t+1}(x_t)] \leq 2L \cdot \frac{m}{n} \cdot \rc{n/m}{\MF} + 2\zeta + \frac{2Ln\sigma }{T^2}.
  \end{align}
\end{proof}

  \begin{lemma}
    \label{lem:switch-sample}
Fix $L \geq 1$, $\zeta \geq 0$.    Consider random variables $x, x', x_1, \ldots, x_n \in \MX$, $y, y' \in [-L,L]$, $y_1, \ldots, y_n \in \BR$ drawn according to some distribution $Q$, and a constant $\gamma \geq L$. Suppose $h_1 \in \MF$ is a function of $(x, y), (x',y'), (x_1, y_1), \ldots, (x_n, y_n)$ satisfying 
    \begin{align}
      y \cdot h_1(x) + \gamma \cdot h_1(x') \cdot \sign(y') + \sum_{i=1}^n h_1(x_i) \cdot y_i & \leq  \min_{f \in \MF} y \cdot f(x) + \gamma \cdot f(x') \cdot \sign(y') + \sum_{i=1}^n f(x_i) \cdot y_i + \zeta \nonumber.
    \end{align}
Suppose that the distribution of the $(n+2)$-tuples $$((x,y), (x',y'), (x_1, y_1), \ldots, (x_n, y_n))$$ and $$((x',y'), (x,y), (x_1, y_1), \ldots, (x_n, y_n))$$ are identical. 
Then $\E_Q[y \cdot h_1(x)] \geq \E_Q[y' \cdot h_1(x')] - 2\zeta$. 
\end{lemma}
\begin{proof}
  Define $h_2$ as the function $h_1$ applied to the sequence $(x', y'), (x,y), (x_1, y_1), \ldots, (x_n, y_n)$, so that
  \begin{align}
    y' \cdot h_2(x') + \gamma \cdot h_2(x) \cdot \sign(y) + \sum_{i=1}^n h_2(x_i) \cdot y_i &\leq \min_{f \in \MF} y' \cdot f(x') + \gamma \cdot f(x) \cdot \sign(y) + \sum_{i=1}^n f(x_i) \cdot y_i\nonumber.
  \end{align}
  By definition of $h_1$, we have that
  \begin{align}
    & y \cdot h_1(x) + \frac{\gamma}{|y'|} \cdot h_1(x') \cdot y' + \sum_{i=1}^n h_1(x_i) \cdot y_i \nonumber\\
    \leq & y \cdot h_2(x) + \frac{\gamma}{|y'|} \cdot h_2(x') \cdot y' + \sum_{i=1}^n h_2(x_i) \cdot y_i + \zeta.
  \end{align}
  By definition of $h_2$, we have that
  \begin{align}
    & y' \cdot h_2(x') + \frac{\gamma}{|y|} \cdot h_2(x) \cdot y + \sum_{i=1}^n h_2(x_i) \cdot y_i\nonumber\\
    \leq & y' \cdot h_1(x') + \frac{\gamma}{|y|} \cdot h_1(x) \cdot y + \sum_{i=1}^n h_1(x_i) \cdot y_i + \zeta\nonumber.
  \end{align}
  Adding the two previous displays and simplifying gives
  \begin{align}
    & h_1(x) \cdot y \cdot (1- \gamma/|y|) + h_1(x') \cdot y' \cdot (\gamma/|y'| - 1)\\
    & \leq h_2(x) \cdot y \cdot (1 - \gamma/|y|) + h_2(x') \cdot y' \cdot (\gamma/|y'| - 1) + 2\zeta.
  \end{align}

  Since $(x,y)$ and $(x', y')$ are exchangable, we have by definition of $h_1, h_2$ that for all constants $a,b \in \BR$,
  \begin{align}
    & \E[h_2(x) \cdot y \cdot (1-\gamma/|y|) + h_2(x') \cdot y'\cdot (\gamma/|y'| - 1) \ | \ |y| = a, |y'| = b] \\
    &= \E[h_1(x') \cdot y' \cdot (1 - \gamma/|y'|) + h_1(x) \cdot y \cdot (\gamma/|y| - 1)\ | \ |y'| = a, |y| = b]\nonumber.
  \end{align}
  Combining the two above displays and rearranging gives
  \begin{align}
    & \E[ h_1(x) \cdot y \cdot (1 - \gamma/a) + h_1(x') \cdot y' \cdot (\gamma/b - 1) \ | \ |y| = a, |y'| = b]\\
    \leq & \E[h_1(x') \cdot y' \cdot (1 - \gamma/a) + h_1(x) \cdot y \cdot (\gamma/b - 1) \ | \ |y| = b, |y'| = a] + 2\zeta.
  \end{align}
  Interchanging the roles of $a,b$, we get
    \begin{align}
    & \E[ h_1(x) \cdot y \cdot (1 - \gamma/b) + h_1(x') \cdot y' \cdot (\gamma/a - 1) \ | \ |y| = b, |y'| = a]\\
    \leq & \E[h_1(x') \cdot y' \cdot (1 - \gamma/b) + h_1(x) \cdot y \cdot (\gamma/a - 1) \ | \ |y| = a, |y'| = b] + 2\zeta.
    \end{align}
    Exchangeability of $(x,y)$ and $(x',y')$ implies that $\Pr(|y| = a, |y'| = b) = \Pr(|y| = b, |y'| = a)$, and thus, by averaging the two above displays, we get
    \begin{align}
      & \E[h_1(x) \cdot y \cdot (2 - \gamma / a - \gamma/b) \ |\  \{ |y|, |y'|\}  = \{a,b\}] \\
        & \leq \E[h_1(x') \cdot y' \cdot (2 - \gamma / a - \gamma / b) \ |\ \{ |y|, |y'|\} = \{a,b\}] + 2\zeta\nonumber.
    \end{align}
    Using that $\gamma \geq L \geq \max\{a,b\}$ gives that
    \begin{align}
\E[h_1(x) \cdot y - h_1(x') \cdot y'\ | \ \{ |y|, |y'| \} = \{a,b\} ] \geq -2\zeta\nonumber.
    \end{align}

     Taking expectation over $\{ |y|, |y'| \}$ gives that $\E[h_1(x) \cdot y - h_1(x') \cdot y'] \geq -2\zeta$, as desired.

\end{proof}

\subsection{Conclusion of Proof}\label{subsec:conclude}
We are now ready to start putting everything together.  We first consider the case of binary labels.
\begin{proposition}\label{prop:ftplclassificationbound}
  Let $\F \subset \{\pm 1\}^{\cX}$ be a binary-valued function class and suppose that we are in the smoothed online learning setting, i.e., the conditional distribution of $x_t$ given the history is $\sigma$-smooth with respect to some measure $\mu$ on $\cX$.  Let $\ell(\yhat, y) = - \yhat y$ be indicator loss.  For each $1 \leq t \leq T$ and any $n$, define for any $f \in \F$,
  \begin{equation}
    \hat{\omega}_{t,n}(f) = \frac 1{\sqrt{n}}\sum_{i = 1}^n \gamma_i f(x_i)
  \end{equation}
  where $x_i \sim \mu$ are independent ant $\gamma_i \sim N(0,1)$ are independent standard normal random variables.  Let $f_t \in \F$ such that
  \begin{equation}
    L_{t-1}(f_t) + \eta \hat{\omega}_{t,n}(f_t) \leq \inf_{f \in \F} L_{t-1}(f) + \eta \hat{\omega}_{t,n}(f) + \zeta
  \end{equation}
  Then, if $\vc(\F) \leq d$, we have for $\eta = \sqrt{\frac{T\log(TL/\sigma)}{\sigma}}$ and $n = T/\sqrt{\sigma}$ that the regret satisfies:
  \begin{equation}
    \ee\left[\reg(f_t)\right] \lesssim \erma T + \sqrt{\frac{Td \log(T/\sigma)}{\sigma} }
  \end{equation}
\end{proposition}
\begin{proof}
  By Hoeffding's inequality and \eqref{eq:infnorm} for some constant $C > 0$, as long as $n \geq C \log \frac{1}{\delta}$, an i.i.d.~sample $x_1, \ldots, x_n \sim \mu$ contains at least $n/2$ copies of $x^\ast$ with probability $1-\delta$, meaning that $\inf_{f \in \MF} \norm{f}_{L^2(\hat \mu_n)} \geq \frac 12$ with probability $1-\delta$; let this probability $1-\delta$ event be denoted $\ME_1$.

    Further, for a sufficiently large constant $C > 0$, by Lemma \ref{lem:f-l2mu}, with probability $1-\delta$ over the sample $x_1, \ldots, x_n \sim \mu$, it holds that
    \begin{align}
\sup_{f,f' \in \MF} \left| \norm{f-f'}_{L^2(\mu)}^2 - \norm{f-f'}_{L^2(\hat \mu_n)}^2 \right| \leq \Delta_n := C \cdot \left( \frac 1n\gc{n}{\MF} + \sqrt{\frac{\log(1/\delta)}{n}}\right).\label{eq:define-deltan}
    \end{align}
    Let this event (i.e., that \eqref{eq:define-deltan} holds) be denoted $\ME_2$.

    Finally, it holds that with probability $1-\delta$ over the sample $x_1, \ldots, x_n \sim \mu$,
    \begin{align}
\left| \E\left[ \sup_{f \in \MF} \hat \omega_n(f) \right] - \frac 1{\sqrt{n}}\gc{n}{\MF}  \right| \leq C \sqrt{\log\left(\frac 1\delta\right)}\label{eq:omegan-conv}.
    \end{align}
    Let this event (i.e., that \eqref{eq:omegan-conv} holds) be denoted $\ME_3$.

    The event $\ME := \ME_1 \cap \ME_2 \cap \ME_3$ occurs with probability $1-3\delta$; taking $\delta = 1/T$, the contribution to expected regret on the complement of $\ME$ is at most $6 = O(1)$. Thus, it suffices to bound regret in expectation conditioned on the event $\ME$, which is what we proceed to do. 
  
    In particular, we use Lemma \ref{lem:btl} to decompose the regret into three terms:
    \begin{align}
      \ee\left[\reg_T(f_T) | \ME \right] &\leq \zeta T + 2 \eta \left(\frac 1{\sqrt{n}} \gc{n}{\F} + C \sqrt{\log T}\right) + T \max_{t \leq T} \ee\left[\ell(f_t(x_t'), y_t') - \ell(f_{t+1}(x_t'), y_t')\right] \\
      &+ T \max_{t \leq T} \ee\left[\ell(f_{t+1}(x_t'), y_t') - \ell(f_{t+1}(x_t), y_t)\right]
    \end{align}
    We can now apply Lemma \ref{lem:gen-error} and Lemma \ref{lem:stabilityftplclassification} coupled with Lemma \ref{lem:f-l2mu} to control $\Delta$.  In particular, we note that as we assume that $\vc(\F) \leq d$, we have
    \begin{equation}
      \ee\left[\sup_{f \in \F} \hat\omega_{t,n}(f)\right] \lesssim \sqrt{d}, \qquad \frac 1{\sqrt{n}} \gc{n}{\MF} \lesssim \sqrt{d}
    \end{equation}
    Then applying Lemmata  \ref{lem:stabilityftplclassification} and \ref{lem:gen-error} we may conclude that
      \begin{align}
       \ee\left[\sup_{f \in \F} \sum_{t = 1}^T \ell(f_t(x_t), y_t) - \ell(f(x_t), y_t)\right] 
        & \lesssim \erma T + \eta \sqrt{d} + T \cdot \frac{(1+\erma)^3 \log \eta}{\sigma \eta} \cdot \sqrt{d} + \frac{2T\Delta_n}{\sigma}\\
        &\quad+ \frac{T \log T}{\sigma n} \R_{\frac{\sigma n}{\log T}}(\F) + \frac{n \sigma}{T} + \zeta T
      \end{align}
      Now, noting that for any $m \in \mathbb{N}$, the assumption that $\vc(\F) \leq d$ implies that
      \begin{equation}
        \frac 1m \R_m(\F) \lesssim \sqrt{\frac dm}
      \end{equation}
      we get
      \begin{align}
        \ee\left[\reg_T(f_t)\right] & \lesssim \erma T + \eta \sqrt{d} + T \cdot \frac{(1+\erma)^3 \log \eta}{\sigma \eta} \cdot \sqrt{d} + \frac{2T\Delta_n}{\sigma}\\
        &\quad+ \sqrt{\frac{d T \log T}{\sigma n}} + \frac{n \sigma}{T} + \zeta T
      \end{align}
      We may choose $\eta = \sqrt{\frac{T\log(TL/\sigma)}{\sigma}}$ and $n = T/\sqrt{\sigma}$ to get a regret bound
      \begin{equation}
        \ee\left[\reg(f_t)\right] \lesssim \erma T + \sqrt{\frac{Td \log(T/\sigma)}{\sigma} }
      \end{equation}
      concluding the proof.
\end{proof}
Note that Proposition \ref{prop:ftplclassificationbound} suffices to prove Theorem \ref{thm:ftplclassification} in the case of binary values.

We now turn to the more challenging case of arbitrary labels.  To understand the difficulty, note that Lemma \ref{lem:gen-error} requires that the loss be linear.  If we assume that the labels $y_t$ are drawn in some smooth manner from a distribution $q_t(\cdot | x_t)$ so that the pair $(x_t, y_t) \sim \widetilde{p}_t$ with $\widetilde{p}_t$ being $\sigma$-smooth with respect to a distribution $\widetilde{\mu}$ on $\cX \times [-1,1]$, then we can reduce to the linear case by replacing $\F : \cX \to [-1,1]$ by $\ell \circ \F: \cX \times [-1,1] \to [-1,1]$ with functions in $\ell \circ \F$ consisting of maps of the form $(x, y) \mapsto \ell(f(x), y)$ for any $f \in \F$.  In the following result, we make use of this observation to bound the regret in the smoothed label setting, for arbitrary loss functions.
\begin{proposition}\label{prop:basesmoothregret}
  Let $\F$ be a function class mapping $\cX \to [-1,1]$ and suppose we are in the smoothed online learning setting with smooth labels, i.e., suppose that for all $t$, the adaptive adversary chooses a distribution $\widetilde{p}_t$ on $\cX \times [-1,1]$, $\sigma$ smooth with respect to some distribution $\widetilde{\mu}$, and samples $(x_t, y_t) \sim \widetilde{p}_t$.  Let $\ell: [-1,1] \times [-1,1] \to [-1,1]$ be a loss function that is $L$-Lipschitz in the first argument.  For each $1 \leq t \leq T$ and any $n$, let
  \begin{align}
    \hat{\omega}_{t,m}(f) =  \frac 1{\sqrt m}\sum_{i = 1}^m \gamma_i f(x_i)  && \hat{\omega}_{t,n}'(f) = \sum_{j = 1}^n \gamma_j' \ell(f(x_j'), y_j')
  \end{align}
  where $x_i \sim \mu$ and $(x_j', y_j') \sim \widetilde{\mu}$ are independent and $\gamma_i, \gamma_j' \sim N(0,1)$ are independent standard normal random variables.  Let $f_t \in \F$ such that
  \begin{equation}
    L_{t-1}(f_t) + \eta \hat{\omega}_{t,m}(f_t) + \hat{\omega}_{t,n}(f) \leq \inf_{f \in \F} L_{t-1}(f) + \eta \hat{\omega}_{t,n}(f) + \hat{\omega}_{t,n}(f) + \zeta \label{eq:ftplsmoothdef}
  \end{equation}
  Then,
  \begin{align}
    \ee\left[\reg_T(f_t)\right] &\lesssim \left(\frac L{\sqrt{m}} \gc{m}{\F} + L \gc{n}{\F} + \sqrt{\log T}\right) \left(2 \eta + T \frac{(1 + \zeta)^3 \log \eta}{\sqrt{\eta \sigma}}\right) \\
    &\quad + \frac{L^2T \log T}{\sigma n} \R_{\frac{\sigma n}{\log T}}(\F) + \frac{n \sigma}{T} + \zeta T \label{eq:basesmoothregret}
  \end{align}
  In particular, if $\vc(\F, \delta) \lesssim \delta^{-p}$ for $p < 2$, we may choose $\eta = T^{2/3} \sigma^{-1/3}$ and $n = T / \sigma$ to get
  \begin{equation}
    \ee\left[\reg_T(f_t)\right] \lesssim T^{\frac 23} \sigma^{- \frac 13} \log\left(\frac T\sigma\right)
  \end{equation}
\end{proposition}
\begin{proof}
  The proof proceeds in a similar manner to that of Proposition \ref{prop:ftplclassificationbound} except we replace $x$ by $(x, y)$, $\F$ by $\ell \circ \F$ and $\ell$ by the identity.  More formally, let $\ell \circ \F = \left\{(x,y) \mapsto \ell(f(x), y) | f \in \F\right\}$ and note that \eqref{eq:define-deltan}, \eqref{eq:omegan-conv}, and $\inf_{f \in \F} \norm{\ell \circ f}_{L^2(\hat{\mu}_n)} \geq \frac 12$ all hold with probability at least $1 - 4\delta$, just as in the proof of the earlier proposition.  Letting $\delta = T^{-2}$, let $\ME$ denote the event that all of these hold, i.e., for all $1 \leq t \leq T$,
  \begin{align}
    \sup_{f, f' \in \F} \abs{\norm{\ell \circ f - \ell \circ f'}_{L^2(\hat{\mu}_n)}^2 - \norm{\ell \circ f - \ell \circ f'}_{L^2(\mu)}^2} &\leq \Delta_n := C \left(\frac 1n \gc{n}{\ell \circ \F} + \sqrt{\frac{\log T}{n}}\right) \\
    \sup_{f, f' \in \F} \abs{\norm{f - f'}_{L^2(\hat{\mu}_m)}^2 - \norm{ f -  f'}_{L^2(\mu)}^2} &\leq \Delta_m \\
    \abs{\ee\left[\sup_{f \in \F} \hat{\omega}_{t,m}(f)\right] - \frac 1{\sqrt{n}}\gc{n}{\F}} 
    &\leq C \sqrt{\log T} \\
    \abs{\ee\left[\sup_{f \in \F} \hat{\omega}_{t,n}(f)\right] - \frac 1{\sqrt{n}}\gc{n}{\ell \circ \F}} &\leq C \sqrt{\log T}  \\
    \inf_{f \in \F} \norm{\ell \circ f}_{L^2(\hat{\mu}_n)} \geq \frac 12
  \end{align}
  The expected regret on the complement of $\ME$ is at most $4 T \delta \leq 4$ by boundedness of the loss.  We may now apply Lemma \ref{lem:btl} to get
  \begin{align}
    \ee\left[\reg_T(f_t); \ME\right] &\lesssim \ee\left[\sup_{f \in \F} \hat{\omega}_{t,m}(f); \ME\right] + \ee\left[\sup_{f \in \F} \hat{\omega}_{t,n}'(f); \ME\right] \\
    &\quad +T \max_{1 \leq t \leq T} \ee\left[\ell(f_{t}(x_t'), y_t') - \ell(f_{t+1}(x_t'), y_t');\ME\right] \\
    &\quad + T \max_{1 \leq t \leq T} \ee\left[\ell(f_{t+1}(x_t'), y_t') - \ell(f_{t+1}(x_t), y_t); \ME\right]
  \end{align}
  The first two terms are bounded by the restriction to $\ME$.  To control the third term, we consider a coupling where $\hat{\omega}_{t,n}' = \hat{\omega}_{t+1,n}'$ and $\hat{\omega}_{t,m} = \hat{\omega}_{t+1,m}$ and note that by the independence of $\hat{\omega}_{t,n}'$ and $\hat{\omega}_{t,m}$, we may condition on the value of the former and let
  \begin{equation}
    \widetilde{L}_{t}(f) = L_t(f) + \hat{\omega}_{t,n}'(f)
  \end{equation}
  We may then apply Lemma \ref{lem:stabilityftplregression} to the resulting expression and get
  \begin{align}
    T \max_{1 \leq t \leq T} \ee\left[\ell(f_{t}(x_t'), y_t') - \ell(f_{t+1}(x_t'), y_t');\ME\right] \lesssim T\frac{(L + 2 \zeta)^3 \log \eta}{\sqrt{\sigma \eta}} \ee\left[\sup_{f \in \F} \hat{\omega}_{t,m}(f)\right] + 4 L T \sqrt{\frac{\Delta_m}{\sigma}}
  \end{align}
  Note that this is further controlled using $\ME$ to bound the expected supremum of $\hat{\omega}_{t,m}$.

  To take care of the last term, we consider a coupling where $\hat{\omega}_{t,m} = \hat{\omega}_{t+1,m}$ but $\hat{\omega}_{t,n}'$ and $\hat{\omega}_{t+1,n}'$ are independent.  We may now condition on $\hat{\omega}_{t,m}$ as in the previous paragraph and apply Lemma \ref{lem:gen-error} to get
  \begin{align}
    T \max_{1 \leq t \leq T} \ee\left[\ell(f_{t+1}(x_t'), y_t') - \ell(f_{t+1}(x_t), y_t); \ME\right] \lesssim 4 L T \frac{\log T}{\sigma n} \R_{\frac{\sigma n}{2 \log T}}(\ell \circ \F) + \zeta T + \frac{2 L n \sigma}{T}
  \end{align}
  To conclude, we apply contraction to note that for all $k \in \mathbb{N}$,
  \begin{align}
    \gc{m}{\ell \circ \F} \leq L \gc{m}{\F} && \R_m(\ell \circ \F) \leq L \R_m(\F)
  \end{align}
  This proves the first statement.

  To prove the second statement, note that if $\vc(\F, \delta) \ll \delta^{-2}$, then
  \begin{align}
    \max\left(\R_k(\F), \gc{k}{\F}\right) \lesssim \sqrt{k} \label{eq:parametricrate}
  \end{align}
  The result then follows by a direct computation.
\end{proof}
While Proposition \ref{prop:basesmoothregret} attains no-regret, the assumption that the labels $y_t$ are drawn in a smoothed manner is much stronger than desired.  In order to mitigate this issue, we apply a discretization scheme.
\begin{proposition}\label{prop:generalftplregret}
  Let $\F$ be a function class mapping $\cX \to [-1,1]$ and suppose that we are in the smoothed online learning setting, where $x_t \sim p_t$ are drawn from a distribution that is $\sigma$-smooth with respect to $\mu$.  Suppose that $\ell: [-1,1] \times [-1,1] \to [-1,1]$ is a loss function that is $L$-Lipschitz in \emph{both} arguments.  Consider the following processes:
  \begin{align}
    \hat{\omega}_{t,m}(f) = \frac 1{\sqrt{n}} \sum_{ i = 1}^m \gamma_i f(x_i) && \hat{\omega}_{t,n}'(f) = \sum_{j = 1}^n \gamma_j' \ell(f(x_j'), y_j')
  \end{align}
  where $x_i, x_j' \sim \mu$, $\gamma_i \sim N(0,1)$ and $y_j'$ are uniform on $[-1,1] \cap \epsilon \mathbb{Z}$ for some fixed $\epsilon > 0$.  Suppose that $f_t$ is chosen such that
  \begin{equation}
    L_{t-1}(f_t) + \eta \hat{\omega}_{t,m}(f_t) + \hat{\omega}_{t,n'}(f_t)  \leq \inf_{f \in \F} L_{t-1}(f) + \eta \hat{\omega}_{t,m}(f) + \hat{\omega}_{t,n'}(f) + \zeta
  \end{equation}
  Then 
  \begin{align}
    \ee\left[\reg_T(f_t)\right] &\lesssim \left(\frac L{\sqrt{m}} \gc{m}{\F} + L \gc{n}{\F} +  \sqrt{\log T}\right) \left(2 \eta + T \frac{(1 + \zeta)^3 \log \eta}{\sqrt{\eta \sigma}}\right) \\
    &\quad + \frac{L^2T \log T}{\epsilon \sigma n} \R_{\frac{\epsilon \sigma n}{\log T}}(\F) + \frac{\epsilon n \sigma }{T} + (\zeta + L \epsilon) T
  \end{align}
\end{proposition}
\begin{proof}
  Let $S^\epsilon = \epsilon \mathbb{Z} \cap [-1,1]$ and, for any $y \in [-1,1]$, let $y^\epsilon$ be the projection of $y$ into $S^\epsilon$.  By assumption, we have $\abs{\ell(\cdot, y) - \ell(\cdot, y^\epsilon)} \leq L \epsilon$.  The key observation is that $(x_t, y_t^\epsilon)$ is $(\epsilon \sigma / 2)$-smooth with respect to $\mu \otimes \unif(S^\epsilon)$ by the fact that $\abs{S^\epsilon} \leq 2/\epsilon$.  We may now apply Lemma \ref{lem:btl} and note that the first term, the magnitude of the perturbation, is unchanged.  For the second term, we note that
  \begin{equation}
    \ee\left[\ell(f_t(x_t), y_t) - \ell(f_{t+1}(x_t'), y_t')\right] \leq \ee\left[\ell(f_t(x_t), y_t^\epsilon) - \ell(f_{t+1}(x_t'), (y_t')^\epsilon)\right] + 2 L \epsilon
  \end{equation}
  which is in turn controlled by Lemma \ref{lem:stabilityftplregression} by the same reasoning as the analogous statement in the proof of Proposition \ref{prop:basesmoothregret}.  To bound the generalization error, we note that, again by the Lipschitz assumption,
  \begin{equation}
    \ee\left[\ell(f_{t+1}(x_t'), y_t') - \ell(f_{t+1}(x_t), y_t)\right] \leq \ee\left[\ell(f_{t+1}(x_t'), (y_t')^\epsilon) - \ell(f_{t+1}(x_t), y_t^\epsilon)\right] + 2 L \epsilon \label{eq:discretization1}
  \end{equation}
  Now, note that
  \begin{align}
    L_{t-1}(f_{t+1}) + \ell(f_{t+1}(x_t), y_t^\epsilon) + \eta \hat{\omega}_{t+1,n}(f_{t+1}) &\leq L_{t-1}(f_{t+1}) + \ell(f_{t+1}(x_t), y_t) + \eta \hat{\omega}_{t+1,n}(f_{t+1}) + L \epsilon \\
     &\leq \inf_{f \in \F} L_{t-1}(f) + \ell(f(x_t), y_t) + \eta \hat{\omega}_{t+1,n}(f) + \zeta +  L \epsilon \\
    &\leq \inf_{f \in \F} L_{t-1}(f) + \ell(f(x_t), y_t^\epsilon) + \eta \hat{\omega}_{t+1,n}(f) + \zeta +  2 L \epsilon
  \end{align}
  Noting again that $(x_t, y_t^\epsilon)$ is $(\epsilon \sigma / 2)$-smooth with respect to $\mu \otimes \unif(S^\epsilon)$, we apply Lemma \ref{lem:gen-error}, adjusting $\zeta$ to $\zeta + 2 L \epsilon$, to get
  \begin{equation}
    \ee\left[\ell(f_{t+1}(x_t'), (y_t')^\epsilon) - \ell(f_{t+1}(x_t), y_t^\epsilon)\right]  \leq 8 \frac{\log T}{c_0 \sigma \epsilon n} \R_{c_0 \sigma \epsilon n / (4 \log T)}(\F) + \frac{n \sigma \epsilon}{T^2} + 2 \zeta + 2 L \epsilon
  \end{equation}
  after noting that $\abs{S^\epsilon} \leq \frac 2\epsilon$.  Combining this with \eqref{eq:discretization1} gives
  \begin{equation}
    \ee\left[\ell(f_{t+1}(x_t'), y_t') - \ell(f_{t+1}(x_t), y_t)\right]  \leq 4 \frac{\log T}{c_0 \sigma \epsilon n} \R_{c_0 \sigma \epsilon n / (2 \log T)}(\F) + \frac{2 n \sigma \epsilon}{T^2} + 2 \zeta + 4 L \epsilon
  \end{equation}
  Plugging back in to Lemma \ref{lem:btl} concludes the proof.
\end{proof}
As a corollary, we have the following bounds.
\begin{corollary}\label{cor:ftplparameters}
  Suppose we are in the setting of Proposition \ref{prop:generalftplregret} and, furthermore, $\vc(\F, \delta) \lesssim \delta^{-p}$ for some $p < 2$.  Then if $\eta = T^{2/3} \sigma^{-1/3}$, $n = \sqrt{T/\sigma}$, and $\epsilon = T^{- 1/3}$, we have
  \begin{equation}
    \ee\left[\reg_T(f_t)\right] \lesssim T^{\frac 23} \sigma^{- \frac 13} + \zeta T
  \end{equation}
  If $p \geq 2$, we may choose $n = T$, $\epsilon = (\sigma T)^{- \frac 1{p+1}}$, and $\eta = T^{\frac 2p}$ to yield $\ee\left[\reg_T(f_T)\right] = o(T)$.
\end{corollary}
\begin{proof}
  The first statement follows immediately from Proposition \ref{prop:generalftplregret} and \eqref{eq:parametricrate}.

  The second statement holds by direct computation and the fact that for all $k$,
  \begin{equation}
      \max\left(\gc{k}{\F}, \R_k(\F)\right) \lesssim k^{1 - \frac 1p}
  \end{equation}
\end{proof}
We see that Corollary \ref{cor:ftplparameters} contains Theorem \ref{thm:generalftpl}.

Finally, at the cost of a slightly worse regret bound, we may simplify the algorithm by considering a single perturbation.  Note that in Corollary \ref{cor:ftplparameters}, we may tune $\eta$ and $n$ independently because we have two distinct perturbations.  In the case where the perturbations are the same, Lemma \ref{lem:gen-error} tells us that $\eta \geq \sqrt n$.  We thus have the following regret bound for the simpler algorithm:
\begin{corollary}\label{cor:ftplsimplealgo}
  Suppose we are in the situation of Proposition \ref{prop:generalftplregret} and $\vc(\F, \delta) \lesssim \delta^{-p}$ for some $p < 2$.  Suppose that $f_t$ is chosen such that
  \begin{equation}
    L_{t-1}(f_t) + \frac \eta{\sqrt{n}} \hat{\omega}_{t,n}'(f_t) \leq \inf_{f \in \F} L_{t-1}(f) + \frac{\eta}{\sqrt{n}} \hat{\omega}_{t,n}(f) + \zeta
  \end{equation}  
  Then if we set $\eta = T^{5/12} \sigma^{-1/4}$, $n = \eta^2$, and $\epsilon = T^{- 3/4} \sigma^{-1/4}$, we have
  \begin{equation}
    \ee\left[\reg_T(f_t)\right] \lesssim T^{\frac 34} \sigma^{- \frac 14} \log\left(\frac T \sigma\right) + \zeta T
  \end{equation}
\end{corollary}
\begin{proof}
  Note that Proposition \ref{prop:basesmoothregret} may be proved with a single perturbation in much the same way, with the caveat that $\eta \geq \sqrt{n}$, and achieve the regret bound given in \eqref{eq:basesmoothregret} with $m = n$.  This proof may then be extended by discretization in the same way as Proposition \ref{prop:generalftplregret}, again with the caveat that $\eta \geq \sqrt{n}$.  We may use \eqref{eq:parametricrate} to control the size of the Gaussian and Rademacher complexities as before and then tune the parameters such that $\eta \geq \sqrt{n}$.  Plugging in the assumed parameters yields the desired result.
\end{proof}
Note that Corollary \ref{cor:ftplsimplealgo} suffices to prove the more general case of Theorem \ref{thm:ftplclassification}.

\section{Proofs from Section \ref{sec:lowerbound}}
\label{app:aldous}
\nc{\val}{\textsf{Val}}
\nc{\br}{\textsf{BR}}
In this section we prove the lower bounds on oracle-efficient algorithms from Section \ref{sec:lowerbound}. The proof structure closely follows that from \cite{hazan2016computational}, but some additional work is required since our setup allows more powerful algorithms than \cite{hazan2016computational}: in particular, the ERM oracle allows (possibly negative) real-valued weights to be attached to each pair $(x_i, y_i)$. In Section \ref{sec:aldous-modified} we recall the definition of Aldous' problem and introduce a slight variant; a known oracle lower bound for Aldous' problem forms the basis for our hardness results. In Section \ref{sec:nash-br-oracle} we introduce an intermediate problem, namely that of approxiating the Nash value in a two-player zero-sum game given a value oracle and best response oracles; we then show an oracle lower bound for this problem by reducing from Aldous' problem. Using this result, in Section \ref{sec:oracle-hardness-proof}, we prove Theorem \ref{thm:oracle-lb} and Corollary \ref{cor:stat-comp-gap} using a standard reduction from finding Nash equilibria to no-regret learning \citep{freund1999adaptive}.

\subsection{Modified Aldous' Problem}
\label{sec:aldous-modified}
We begin by recalling the definition of Aldous' problem and a slight variation we will use. 
Consider a function $\phi : \{ 0,1\}^d \ra \BZ$; for all such functions in this section, we assume that $|\phi(x)| \leq 2^{O(d)}$ for all $x \in \{0,1\}^d$. A point $x \in \{0,1\}^d$ is a \emph{local maximum} if $\phi(x) \geq \phi(x')$ for all $x'$ of Hamming distance at most 1 from $x$. The function $\phi$ is \emph{globally consistent} if it has a single local maximum (i.e., the only local maximum is also a global maximum). \emph{Aldous' problem} is the following problem: suppose we are given a globally consistent function $\phi : \{0,1\}^d \ra \BN$ with black-box oracle access in the sense that we can query a value $x \in \{0,1\}^d$ and the oracle will respond with the value $\phi(x)$. The objective is to determine whether the maximum value of $\phi$ is even or odd (with a minimum number of oracle queries). The following lower bound on the number of oracle calls needed to solve Aldous' problem is known:
\begin{theorem}[\cite{aaronson2006lower,aldous1983minimization,hazan2016computational}]
\label{thm:aa}
There is a constant $c > 0$ so that the following holds. Fix any $d \in \BN$, and consider any randomized algorithm for Aldous' problem that makes at most $c \cdot 2^{d/2} / d^2$ oracle queries in the worst case. Then there is a globally consistent function $\phi : \{0,1\}^d \ra \BN$ so that the algorithm cannot determine with probability higher than $2/3$ whether the maximum value of $\phi$ over $\{0,1\}^d$ is even or odd.
\end{theorem}

For our purposes we require a lower bound applying to a slightly more restricted class of functions than Theorem \ref{thm:aa}, specified in Definintion \ref{def:doubly-con} below.
\begin{definition}
\label{def:doubly-con}
We say that a function $\phi : \{0,1\}^d \ra \BZ$ is \emph{min-max consistent} if it has both a single local maximum and a single local minimum. 
\end{definition}

As an immediate corollary of Theorem \ref{thm:aa} we get an exponential lower bound for local search with min-max consistent functions:
\begin{corollary}
\label{cor:aa}
There is a constant $c' > 0$ so that the following holds. Consider any randomized algorithm for Aldous' problem that makes at most $c' \cdot 2^{d/2} / d^2$ oracle queries in the worst case. Then there is a min-max consistent function $\phi : \{0,1\}^d \ra \BZ$ so that the algorithm cannot determine with probability higher than $2/3$ whether the maximum value of $\phi$ over $\{0,1\}^d$ is even or odd.
\end{corollary}
\begin{proof}
Suppose to the contrary that $\MA$ is a (randomized) algorithm that makes at most $c' \cdot 2^{d/2}/d^2$ oracle queries in the worst case and determines, for any min-max consistent function $\phi : \{0,1\}^d \ra \BN$, the parity of its maximum value with probability at least 2/3.

Consider a globally consistent function $\phi : \{0,1\}^d \ra \BN$. We define a min-max consistent function $\phi' : \{0,1\}^{d+1} \ra \BZ$ as follows: for $x \in \{0,1\}^{d+1}$,
\begin{align}
\phi'(x) = \begin{cases}
    \phi(x_1, \ldots, x_d) \quad & x_{d+1} = 0 \\
    -\phi(x_1, \ldots, x_d) \quad & x_{d+1} = 1.
\end{cases}\label{eq:fprime-define}.
\end{align}
To see that $\phi'$ is min-max consistent, note that any local maximum $x^\st = (x_1^\st, \ldots, x_{d+1}^\st)$ of $\phi'$ must satisfy $x^\st_{d+1} = 0$, and furthermore, the point $(x^\st_1, \ldots, x^\st_d) \in \{0,1\}^d$ must be a local maximum of $\phi$. Thus $\phi'$ has a single local maximum. Similarly, for any local minimum $x_\st$ of $\phi'$, we must have  $x_{\st, d+1} = 1$ and $(x_{\st,1}, \ldots, x_{\st,d})$ is a local maximum of $\phi$; clearly there is a unique such point $x_\st \in \{0,1\}^{d+1}$.

We use $\MA$ to determine the parity of the maximum value of $\phi$ using in the worst case no more than $c' 2^{(d+1)/2}/(d+1)^2$ oracle queries (to $\phi$): we run the algorithm $\MA$ with the function $\phi'$, and for each oracle query $x \in \{0,1\}^{d+1}$, we can return the value of $\phi'(x)$ per \eqref{eq:fprime-define} using a single oracle query to $\phi$. By assumption $\MA$ determines the parity of the maximum value of $\phi'$, which is the same as the parity of the maximum value of $\phi$, with probability at least 2/3.

Letting $c$ be the constant of Theorem \ref{thm:aa}, as long as $c'$ is chosen so that $c' \cdot 2^{(d+1)/2}/(d+1)^2 < c \cdot 2^{d/2} /d^2$ for all $d$, we get a contradiction to Theorem \ref{thm:aa}, as desired.
\end{proof}

\subsection{Hardness of Computing Nash Equilibria with Best-Response Oracles}
\label{sec:nash-br-oracle}
Fix $N \in \BN$ which is a power of 2, and define $d = \log_2 N$. Throughout this section, we identify each vertex $v$ of the $d$-dimensional hypercube $\{0,1\}^d$ with the integer in $[N]$ whose binary representation corresponds to $v$. Let $\phi : [N] \ra \BZ$ be a min-max consistent input (Definition \ref{def:doubly-con}) to Aldous' problem, with maximum value $\phi^\st = \max_{i \in [N]} \{ \phi(i)\}$. We construct a 0-sum game with value $\lambda = \lambda(\phi^\st)$, with
\begin{align}
\lambda(k) = \begin{cases}
    -1 \quad & \mbox{ if $k$ is even} \\
    1 \quad & \mbox{ if $k$ is odd}.
    \end{cases}\nonumber
\end{align}
Further, for a subset $V \subset [N]$ (identified with the corresponding subset of the hypercube), let $\Gamma(V) \subset [N]$ denote the set of neighbors of $V$ in the hypercube (including the elements of $V$).

Given the function $\phi$, we construct the following game matrix $G^\phi \in \{-1,1\}^{N \times N}$:
\begin{align}
\forall i,j \in [N], \qquad G_{ij}^\phi = \begin{cases}
    \lambda(\phi(i)) \qquad & \mbox{ if $i,j$ are local maxima of $\phi$}\\
    -1 \qquad & \mbox{ if $\phi(i) \geq \phi(j)$ (and the first case does not apply)} \\
    1 \qquad & \mbox{ otherwise}.
    \end{cases}\label{eq:gf-game}
\end{align}
We let $k^\st := \max_{i \in [N]} \{ \phi(i) \}$ denote the (unique) global maximum of $\phi$. As an intermediate problem between Aldous' problem and the problem of oracle-efficient online (smoothed) learning, we consider the problem of approximating the Nash equilibrium value in the two-player zero-sum game induced by the matrix $G$, given access to the following 3 oracles: 
\begin{itemize}
\item The value oracle $\val(i,j)$ returns $G_{ij}^\phi$ for $i,j \in [N]$.
\item The best response oracle $\br^1(q)$, for $q \in \BR^N$, returns
\begin{align}
    \br^1(q) = \begin{cases}
    \argmin_{i \in \Gamma(\supp(q))} \{ e_i^\t G^\phi q\} \quad & k^\st \not \in \supp(q) \\
    \argmin_{i \in [N]} \{ e_i^\t G^\phi q\} \quad & \mbox{ otherwise}.\end{cases}\label{eq:br1}
\end{align}

\item The best response oracle $\br^2(p)$, for $p \in \BR^n$, returns
\begin{align}
    \br^2(p) = \begin{cases}
    \argmax_{j \in \Gamma(\supp(p))} \{ p^\t G^\phi e_j \} \quad & k^\st \not \in \supp(p)  \\
    \argmax_{j \in [N]} \{ p^\t G^\phi e_j\} \quad & \mbox{ otherwise}.\end{cases}\label{eq:br2}
\end{align}
\end{itemize}

We define computation given access to the above oracles $\val, \br^1, \br^2$ in an analogous way as to how computation was defined with respect to the ERM oracle in Section \ref{sec:erm-oracle}: $p$ is represented as a list of atoms $\{ (i, p_i) : p_i > 0\}$ and $q$ is represented as a list of atoms $\{ (j, q_j) : q_j > 0 \}$, and changing a single atom on either list takes unit time. Further, calling any of the oracles $\val, \br^1, \br^2$ takes unit time. Next we show that given the oracles $\val, \br^1, \br^2$, computing the approximate Nash equilibrium value of an $N \times N$ game $G$ cannot be done in $o(\sqrt{N})$ time (up to logarithmic factors). 

To begin, we establish some basic properties of the game $G^\phi$ constructed in \eqref{eq:gf-game}. 
\begin{lemma}
\label{lem:gf-lambda}
For any globally consistent function $\phi$, the minimax value of $G^\phi$ is $\lambda$.
\end{lemma}
\begin{proof}
Let $k^\st = \argmax_{i \in [N]} \{ \phi(i) \}$ denote the global maximum of $\phi$. We show that the pure strategy profile $(k^\st, k^\st)$ is a Nash equilibrium of the game $G^\phi$. The payoff with this profile is $\lambda(\phi(k^\st)) = \lambda$. For any $i \in [N]$, the strategy profile $(i, k^\st)$ generates a payoff of either $\lambda$ (in the case $i = k^\st$) or $1 \geq \lambda$ since for all $i \neq k^\st$, $\phi(i) < \phi(k^\st)$. Thus there is no useful deviation for player 1. Similar, for any $j \in [N]$, the strategy profile $(k^\st, j)$ generates a payoff of either $\lambda$ (in the case that $j = k^\st$) or of $-1 \leq \lambda$ since for all $j \neq k^\st$, $\phi(j) < \phi(k^\st)$. Thus $(k^\st, k^\st)$ is a Nash equilibrium, meaning that its value is the value of the game.
\end{proof}

\begin{lemma}
  Fix any min-max consistent function $\phi : [N] \ra \BZ$. The oracles $\val$ and $\br^1, \br^2$ are correct value and best-response oracles for the game $G^\phi$, in that:
  \begin{align}
\val(i,j) = G^\phi_{ij}, \quad \br^1(q) = \argmin_{i \in [N]} \{ e_i^\t G^\phi q \}, \quad \br^2(p) = \argmax_{j \in [N]} \{ p^\t G^\phi e_j \}.\nonumber
  \end{align}
\end{lemma}
\begin{proof}
The oracle $\val$ is clearly valid as a value oracle for the game $G^\phi$ since $\val(i,j) = G_{ij}^\phi$ for all $i,j \in [N]$ by definition. Furthermore, the best response oracles $\br^1(q),\ \br^2(p)$ are clearly valid for $G^\phi$ in the case that $k^\st \in \supp(q)$ or $k^\st \in \supp(p)$, respectively. We next verify that they are valid in the remaining case.

We begin by considering the best response oracle $\br^2$: fix some input $p \in \BR^N$ with $k^\st \not \in \supp(p)$, let $j = \br^2(p)$ denote  the output of the oracle defined above, and set $v = \max_{j \in [N]} p^\t G^\phi e_j$ to be the value of player 2's best response to $p$. Choose some $j^\st \in [N]$ so that $p^\t G^\phi e_{j^\st} = v$. Since $k^\st \not \in \supp(p)$, we have, for all $j \in [N]$,
\begin{align}
p^\t G^\phi e_j = \sum_{i \in \supp(p): \phi(i) < \phi(j)} p_i  -\sum_{i \in \supp(p): \phi(i) \geq \phi(j)} p_i \label{eq:pge}.
\end{align}
We consider the following cases regarding the value of $\phi(j^\st)$:
\begin{enumerate}
\item $\phi(j^\st) = \phi(i)$ for some $i \in \supp(p)$. Then since $p^\t G^\phi e_j$ only depends on $j$ through $\phi(j)$ (as is evident from \eqref{eq:pge}, it follows that $v = p^\t G^\phi e_{j^\st} = p^\t G^\phi e_i \leq p^\t G^\phi e_j$, as desired.    
\item $\phi(j^\st) > \max_{i \in \supp(p)} \{ \phi(i) \}$. Since $k^\st \not \in \supp(p)$, and $\phi$ is min-max consistent, there is some $j' \in \Gamma(\supp(p))$ so that $\phi(j') > \max_{i \in \supp(p)} \{ \phi(i)\}$. It is evident from \eqref{eq:pge} that $p^\t G^\phi e_{j'} = p^\t G^\phi e_{j^\st} = v$, which implies, by definition of $j$ and since $j' \in \Gamma(\supp(p))$, that $p^\t G^\phi e_j \geq p^\t G^\phi e_{j^\st}$, as desired.
\item Suppose the previous two cases do not hold. Choose $j' \in \supp(p)$ with $\phi(j')$ as small as possible so that $\phi(j') \geq \phi(j^\st)$. It is again evident from \eqref{eq:pge} that $p^\t G^\phi e_{j^\st} = p^\t G^\phi e_{j'} \leq p^\t G^\phi e_j$, as desired.
\end{enumerate}

We next consider the best response oracle $\br^1$: fix some input $q \in \BR^N$ with $k^\st \not \in \supp(q)$, let $i = \br^2(q)$ denote the output of the oracle defined above, and set $v = \min_{i \in [N]} e_i^\t G^\phi q$ to be the value of player 1's best response to $q$. Choose some $i^\st \in [N]$ so that $e_{i^\st}^\t G^\phi q = v$. Since $k^\st \not \in \supp(q)$, we have, for all $i \in [N]$,
\begin{align}
e_i^\t G^\phi q = \sum_{j \in \supp(q) : \phi(j) > \phi(i)} q_j - \sum_{j \in \supp(q) : \phi(j) \leq \phi(i)} q_j \label{eq:egq}.
\end{align}
We consider the following cases regarding the value of $\phi(i^\st)$:
\begin{enumerate}
\item $\phi(i^\st) = \phi(j)$ for some $j \in \supp(q)$. Then since $e_i^\t G^\phi q$ only depends on $i$ through $\phi(i)$ (as is evident from \eqref{eq:egq}), it follows that $v = e_{i^\st}^\t G^\phi q = e_j^\t G^\phi q \geq e_i^\t G^\phi q$, as desired.         
\item $\phi(i^\st) < \min_{j \in \supp(q)}\{ \phi(j) \}$. It cannot be the case that $k_\st \in \supp(q)$ since then we would have $\phi(i^\st) < \phi(k_\st)$. Therefore, since $\phi$ is min-max consistent, there is some $i' \in \Gamma(\supp(q))$ so that $\phi(i') < \min_{j \in \supp(q)} \{ \phi(j) \}$. It follows that $e_{i^\st}^\t G^\phi q \leq e_{i'}^\t G^\phi q = e_{i^\st}^\t G^\phi q = v$, as desired.
\item Suppose the previous two cases do not hold. Choose $i' \in \supp(q)$ with $\phi(i')$ as large as possible so that $\phi(i') \leq \phi(i^\st)$. It is again evident from \eqref{eq:egq} that $e_{i^\st}^\t G^\phi q = e_{i'}^\t G^\phi q \geq e_i^\t G^\phi q$, as desired.
\end{enumerate}
\end{proof}

\begin{lemma}
\label{lem:nash-oracle-lb}
There is a constant $c_0 > 0$ so that the following holds. Fix any $N \in \BN$. Any randomized algorithm $\MA$ for approximating the equilibrium of $N \times N$ $\{-1,1\}$-valued zero-sum games with the oracles $\br^1, \br^2, \val$ cannot guarantee with probability greater than $2/3$ that the algorithm $\MA$'s output value is at most $1/4$ from the game's true value in time $c_0 \cdot \sqrt{N} / \log^3 N$. 
\end{lemma}
\begin{proof}
  Fix any $N \in \BN$. At the cost of a constant factor (and by a standard padding argument) we may assume that $N$ is a power of 2. 
Let $\MA$ be an algorithm as in the theorem statement, and suppose for the purpose of contradiction that with probability greater than $2/3$, for any $N \times N$, $\{-1,1\}$-valued zero-sum game, $\MA$'s output value is at most $1/4$ away from the game's value and $\MA$ runs in time $c_0 \cdot \sqrt{N} / \log^{3}N$ for some constant $c_0 > 0$.

We use the algorithm $\MA$ to derive a contradiction to Corollary \ref{cor:aa}. Accordingly, let $\phi : [N] \ra \BZ$ be a min-max consistent function to which we can make black-box value queries. Consider the $N \times N$ $\{-1,1\}$-valued game $G^\phi$ defined in \eqref{eq:gf-game}. We run algorithm $\MA$ on the game $G^\phi$, simulating the oracles $\val, \br^1, \br^2$ as follows:
\begin{itemize}
\item The value oracle $\val(i,j)$ can be simulated using at most $\log(N) + 1$ queries to $\phi$ (namely, to $\phi(i)$ and $\phi(j)$, as well as, in the case that $i = j$, to all neighbors of $i$ to check whether it is a local maximum).
\item Fix some $q \in \BR^N$, and set $m_q := |\supp(q)|$; the best response oracle $\br^1(q)$ may be simulated as follows:
\begin{enumerate}
\item \label{it:phi-j-query} Query the value of $\phi(j)$ for all $j \in \Gamma(\supp(q))$; this requires $m_q \cdot (\log(N) + 1)$ oracle queries to $\phi$. 
\item By comparing, for each $j \in \supp(q)$, the value of $\phi(j)$ to the value of $\phi(j')$ for each neighbor $j'$ of $j$ (all of which were queried in the previous step), we may check if $k^\st \in \supp(q)$. 
\item If $k^\st \in \supp(q)$, then output the parity of $\phi(k^\st)$ and terminate the algorithm early.
\item Otherwise, if $k^\st \not \in \supp(q)$, then using the queried values of $\phi(j)$, $j \in \Gamma(\supp(q))$, we may compute $\br^1(q)$ per \eqref{eq:br1} -- here we use that $\argmin_{i \in \Gamma(\supp(q))} \{ e_i^\t G^\phi q\}$ may be computed entirely from the values of $\phi(j)$ for $j \in \Gamma(\supp(q))$. 
\end{enumerate}
\item For $p \in \BR^N$ and $m_p := |\supp(p)|$, the best response oracle $\br^2(p)$ may be simualted analogously to above, using at most $m_p \cdot (\log(N)+1)$ oracle queries to $\phi$. 
\end{itemize}
If none of the calls to $\br^1, \br^2$ terminates early, then given the output $\hat v \in \BR$ of the algorithm $\MA$, we simply output the sign of $\hat v$.

Write $\phi^\st = \max_{j \in [N]} \{ \phi (j) \}$. We claim that the resulting algorithm described above outputs with probability at least 2/3, $-1$ if $\phi^\st$ is even and $1$ if $\phi^\st$ is odd. To see this, we argue as follows: with probability 2/3 over the randomness of the algorithm $\MA$, one of the following must occur:
\begin{itemize}
\item Some call to either $\br^1, \br^2$ causes the algorithm to terminate early, in which case it is clear that the algorithm correctly outputs the parity of the maximum value of $\phi$.
\item The output of the algorithm $\MA$ is within $1/4$ of the value of the game $G^\phi$, which we denote by $\lambda \in \{-1,1\}$. By Lemma \ref{lem:gf-lambda}, $\lambda$ is equal to $-1$ if $\phi^\st$ is even and $1$ if $\phi^\st$ is odd. Thus, the output of the algorithm $\MA$ is $-1$ if $\phi^\st$ is even and $1$ if $\phi^\st$ is odd.
\end{itemize}
Having verified correctness (with probability at least 2/3) of the algorithm above to find the parity of $\phi^\st$, we proceed to analyze its oracle cost. The algorithm $\MA$ is assumed to take time $c_0 \cdot \sqrt{N} / \log^3(N)$, for some sufficiently small constant $c_0$. 
Let us denote the number of oracle calls $\MA$ makes to $\val$ by $\omega_\val$; further, denote the total time consumed by all oracle calls $\MA$ makes to $\br^2(p)$ (including the oracle calls themselves and the time spent writing the input atoms $(i, p_i)$) by $\omega_{\br^2}$; define $\omega_{\br^1}$ similarly for the oracle $\br^1$. By the definition of our oracle model above, it holds that $\omega_\val + \omega_{\br^1} + \omega_{\br^2} \leq c_0 \cdot \sqrt{N} / \log^3(N)$. 

Since each call by $\MA$ to $\val$ makes at most $\log(N)+1$ oracle queries to $\phi$, the total number of oracle calls to $\phi$ as a result of calls to the $\val$ oracle by $\MA$ is bounded above by $(\log(N)+1) \cdot \omega_\val$. Similarly, since we can store the result of oracle calls to $\phi$ for previously used atoms $(i, p_i)$ or $(j, q_j)$ (in step \ref{it:phi-j-query} above), 
the total number of oracle calls to $\phi$ as a result of calls to the $\br^2$ oracle by $\MA$ is bounded above by $(\log(N)+1) \cdot \omega_{\br^2}$. Using similar reasoning for calls to $\br^1$, we get that the total number of oracle calls to $\phi$ in our algorithm above is at most
\begin{align}
(\log(N)+1) \cdot  \left( \omega_\val + \omega_{\br^1} + \omega_{\br^2}\right) \leq (\log(N)+1) \cdot c_0 \sqrt{N}/\log^3(N) < c' \cdot \sqrt{N} / \log^2(N)\nonumber,
\end{align}
where $c'$ is the constant of Corollary \ref{cor:aa} (as long as the constant $c_0$ is chosen sufficiently small). This is a contradiction to the conclusion of Corollary \ref{cor:aa}, thus completing the proof of Lemma \ref{lem:nash-oracle-lb}. 
\end{proof}

\subsection{Hardness of oracle-efficient proper no-regret learning}
\label{sec:oracle-hardness-proof}
In this section we use the oracle lower bounds for finding Nash equilibria in two-player zero-sum games to derive oracle lower bounds for no-regret online learning against a $\sigma$-smooth adversary.

In particular, we first prove Theorem \ref{thm:oracle-lb}, stated below with precise logarithmic factors. The proof is a standard reduction from finding Nash equilibria in two-player zero-sum games to no-regret learning \citep{freund1999adaptive}, but we provide the details for completeness:
\begin{customthm}{\ref{thm:oracle-lb}}[Restated, precise]
For some constant $c > 0$, we have the following:  fix any $T \in \BN$ and $\sigma \in (0,1]$. 
In the ERM oracle model, any randomized algorithm cannot guarantee expected regret  smaller than $\frac{T}{200}$ against a $\sigma$-smooth online adversary and any $\MF$ with $|\MF| \leq 1/\sigma$ over $T$ rounds in total time smaller than $c \cdot \frac{1/\sqrt{\sigma}}{\log^3 1/\sigma}$; further, this result holds even for binary-valued classes.
\end{customthm}

\begin{proof}[Theorem \ref{thm:oracle-lb}]
  Fix $T,\sigma$ as in the theorem statement; at the cost of a constant factor we may assume that $1/\sigma$ is an integer. 
  Suppose $\MA$ is an algorithm which guarantees expected regret smaller than $\frac{T}{200}$ against all $\sigma$-smooth adversaries in time $\leq c \cdot \frac{1/\sqrt{\sigma}}{\log^3 1/\sigma}$. By Markov's inequality, for any $\sigma$-smooth adversary, the regret of $\MA$ is bounded above by $\frac{T}{20}$ with probability at least $9/10$.

Set $N := 1/\sigma$, and consider any $N \times N$ $\{-1,1\}$-valued zero-sum game, represented by a game matrix $G \in \{-1,1\}^{N \times N}$, with entries $G_{f_1,f_2}$, $f_1, f_2 \in [N]$; as a matter of convention we suppose that the min-player chooses the first coordinate $f_1$ and the max-player chooses the second coordinate $f_2$. Now consider the following procedure for approximating the Nash equilibrium value of $G$ (we will show below how to implement the below using the oracles $\val, \br^1, \br^2$ introduced in the previous section):
  \begin{enumerate}
  \item Initialize instances $\MA_1, \MA_2$ of the algorithm $\MA$ given the time horizon $T$; for $\MA_1$ the function class is $\{ f_2 \mapsto G_{f_1, f_2} : f_1 \in [N]\}$, and for $\MA_2$ the function class is $\{ f_1 \mapsto G_{f_1,f_2} : f_2 \in [N]\}$. 
    The loss functions of the algorithms are given as follows:
    \begin{itemize}
    \item The loss function of $\MA_1$ is $\ell(\hat y, y) = \hat y$; thus $\MA_1$ incurs loss of $G_{f_1, f_2}$ for predicting $f_1$ when it observes $f_2$. 
    \item  The loss function of $\MA_2$ is $\ell(\hat y, y) = -\hat y$; thus $\MA_2$ incurs loss of $-G_{f_1, f_2}$ for predicting $f_2$ when it observes $f_1$.
    \end{itemize}
  \item For $t = 1, 2, \ldots, T$:
    \begin{enumerate}
    \item Let the algorithms $\MA_1, \MA_2$ produce (random) decisions $f_{1,t}, f_{2,t} \in [N]$, respectively.
    \item Update $\MA_1$ with the context $f_{2,t}$.
    \item Update $\MA_2$ with the context $f_{1,t}$.
    \end{enumerate}
  \item Define mixed strategies $\bar f_{i,T} := \frac 1T \sum_{t=1}^T f_{i,t}$ for $i = 1,2$.
  \item Output the value $\hat v := \bar f_{1,T}^\t G \bar f_{2,T} = \frac{1}{T^2} \cdot \sum_{t,s=1}^T G_{f_{1,t}, f_{2,s}}$.\label{it:compute-vhat}
  \end{enumerate}
  Note that we do not need to specify the labels $y_t$ for either algorithm $\MA_1, \MA_2$ above, since their loss functions do not depend on the true labels $y_t$. By the union bound, with probability at least $4/5$, the regret of both $\MA_1, \MA_2$ is bounded above by $T/20$; in particular, with probability at least $4/5$ we have:
  \begin{align}
\sum_{t=1}^T G_{f_{1,t}, f_{2,t}} - \min_{f_1\in [N]} \sum_{t=1}^T G_{f_1, f_{2,t}} \leq \frac{T}{20}, \qquad \max_{f_2 \in [N]} \sum_{t=1}^T G_{f_{1,t}, f_2} - \sum_{t=1}^T G_{f_{1,t}, f_{2,t}} \leq \frac{T}{20}\nonumber.
  \end{align}
  Adding the two preceding equations, we obtain
  \begin{align}
\max_{f_2 \in [N]} \bar f_{1,T}^\t G e_{f_2} - \min_{f_1 \in [N]} e_{f_1}^\t G \bar f_{2,T} \leq \frac{1}{10},\label{eq:value-gap}
  \end{align}
  where $e_f,\ f \in [N]$ denotes the unit vector corresponding to $f$. Set $\ep := 1/10$. Letting $(f_1^\st, f_2^\st)$ denote a Nash equilibrium of $G$ and $v^\st := (f_1^\st)^\t G f_2^\st$ denotes the value of $G$, we have
  \begin{align}
    & v^\st - \ep \leq \bar f_T^\t G f_2^\st - \ep \leq \max_{f_2 \in [N]} \bar f_{1,T}^\t G e_{f_2} - \ep \stackrel{\eqref{eq:value-gap}}{\leq}  \min_{f_1 \in [N]} e_{f_1}^\t G \bar f_{2,T} \leq \bar f_{1,T}^\t G \bar f_{2,T} \nonumber\\
    \leq & \max_{f_2 \in [N]} \bar f_{1,T}^\t G e_{f_2} \stackrel{\eqref{eq:value-gap}}{\leq} \min_{f_1 \in [N]} e_{f_1}^\t G \bar f_{2,T} + \ep \leq (f_1^\st)^\t G \bar f_{2,T} + \ep   \leq v^\st + \ep \nonumber.
  \end{align}
  Thus we have $| \hat v - v^\st| \leq \ep = 1/10$, meaning that the above procedure determines the game $G$'s value up to error $1/10$.

  We next analyze the time complexity of the above procedure, which involves showing how to implement it efficiently using the oracles $\br^1, \br^2, \val$:
  \begin{itemize}
  \item Each time $\MA_1$ makes an ERM oracle call of the form $\argmin_{f_1 \in [N]} \sum_{i=1}^m w_i \cdot \ell_i(G_{f_1, f_{2,i}}, y_i)$, we do the following: we may assume without loss of generality that all $f_{2,i}$ are distinct. Now write $\ell_{i,1} := \ell_i(1, y_i),\ \ell_{i,-1} := \ell_i(-1,y_i)$. This ERM call may be simulated by the oracle call $\br^1(q)$, where
$ 
q_{f_{2,i}} = w_i \cdot \frac{\ell_{i,1} - \ell_{i,-1}}{2}\nonumber
$
for all $i \in [m]$, and $q_{f_2} = 0$ for all other $f_2$.
\item Each time $\MA_2$ makes an ERM oracle call of the form $\argmin_{f_2 \in [N]} \sum_{i=1}^m w_i \cdot \ell_i(G_{f_{1,i}, f_2})$, we define $\ell_{i,1},\ell_{i,-1}$ as above and simulate it using the oracle call $\br^2(p)$ where $p_{f_{1,i}} = w_i \cdot \frac{\ell_{i,-1} - \ell_{i,1}}{2}$ for all $i \in [m]$ and $p_{f_1} = 0$ for all other $f_2$. 
\item It only remains to show how the estimation of $\hat v$ in step \ref{it:compute-vhat} can be implemented efficiently: for any $f_1, f_2 \in [N]$, the value $G_{f_1,f_2}$ can be queried with a single oracle call as $\val(f_1, f_2)$, so $\hat v$ may trivially be computed in time $O(T^2)$. We may in fact obtain a stronger bound as follows: fix $\delta > 0$ and a sufficiently large constant $C > 0$, and for $1 \leq j \leq C \log(1/\delta)$ sample i.i.d.~pairs $(i_j^1, i_j^2)$ uniformly from $[T] \times [T]$, and output $\hat v' := \frac{1}{C \log 1/\delta} \sum_{j=1}^{C \log 1/\delta} G_{f_{1,i_j^1}, f_{2, i_j^2}}$ tuples, for a total of $O(C \log 1/\delta)$ time (including the oracle calls to $\val$). By the Chernoff bound, we have that $| \hat v' - \hat v| \leq 1/100$ with probability $1-\delta$, as long as $C$ is sufficiently large. 
\end{itemize}
As long as $\delta$ in the third bullet above satisfies $\delta \leq 4/5 - 2/3$, we have established that there is an algorithm that with probability $2/3$ estimates the value $v^\st$ of $G$ up to accuracy of $1/9$. Furthermore, it is straightforward to see that implementing the oracle calls to $\br^1, \br^2, \val$ as described above only lead to a constant factor blowup in the total time. It is also evident that since the space of contexts for both $\MA_1, \MA_2$ is $[N]$, arbitrary adaptive adversaries (in particular, the adversaries faced by $\MA_1, \MA_2$ above) are $1/N$-smooth with respect to the uniform distribution on $[N]$. Thus, by the assumed time complexity upper bound of $\MA$, we have that the algorithm to estimate $v^\st$ runs in time $c' \cdot \frac{\sqrt{N}}{\log^3(N)}$ for some constant $c'$, which can be made arbitrarily small by choosing $c$ to be arbitrarily small. This contradicts Lemma  \ref{lem:nash-oracle-lb}. 

\end{proof}

Now we prove Corollary \ref{cor:stat-comp-gap} (restated below with precise logarithmic factors), which is a straightforward consequence of Theorem \ref{thm:oracle-lb}:
\begin{customcor}{\ref{cor:stat-comp-gap}}[Restated, precise]
  Fix any $\alpha \geq 1$, $\ep < 1/200, \sigma \in (0,1]$, and $d \geq \log 1/\sigma$. 
  Any algorithm whose total time in the ERM oracle model over $T$ rounds is bounded as $T^\alpha$ requires that $T \geq {\Omega}\left(\max \left\{ \frac{d}{\ep^2}, \frac{\sigma^{-{1}/({2\alpha})}}{\log^3 1/\sigma} \right\} \right)$ to achieve regret $\ep T$ for classes $\MF$ of VC dimension at most $d$ against a $\sigma$-smooth adversary.

  Furthermore, any algorithm which achieves regret $\ep T$ for classes of VC dimension at most $d$ against a $\sigma$-smooth adversary must have computation time at least $\Omega \left( \max \left\{ \frac{d}{\ep^2}, \frac{\sigma^{-{1}/2}}{\log^3 1/\sigma} \right\} \right)$. 
\end{customcor}
\begin{proof}[Corollary \ref{cor:stat-comp-gap}]
  Fix any $\ep < 1/200, \sigma \in (0,1]$, and $d \geq \log 1/\sigma$, as in the statement of the corollary. We begin by proving the first statement of the lemma. We consider the following cases:

  \textbf{Case 1.} $d / \ep^2 > \sigma^{- \frac{1}{2\alpha}}/\log^31/\sigma$. For any fixed distribution $Q$ on $\MX \times \{-1,1\}$, consider the i.i.d.~adversary which chooses $(x_t, y_t)$ according to $Q$ for each $t$. An online-to-batch reduction \citep{cesa2004generalization,shalev2011online} establishes that if an online algorithm can achieve expected regret at most $\ep T$, then there is an offline algorithm that achieves expected error at most $\ep$ given $T$ samples from $Q$. But \cite{vapnik1974theory} shows that for any binary function class $\MF$ with $\vc(\MF) = d$, no algorithm using only $c \cdot d/\ep^2$ samples (for a sufficiently small constant $c$) can achieve expected error at most $\ep$ for all distributions $Q$ whose $\MX$-marginal is uniform on a shattered set of $\MF$ of size $d$. Taking $\mu$ to be such a uniform marginal, we see that there is no online algorithm (regardless of oracle efficiency) that achieves regret $\leq \ep T$ against any 1-smooth adversary with respect to $\mu$ if $T < c \cdot d/\ep^2$.

  \textbf{Case 2.} $d / \ep^2 \leq \sigma^{-\frac{1}{2\alpha}}/\log^31/\sigma$.
  Consider any algorithm in the ERM oracle model, $\MA$, whose total computation time over $T$ rounds is bounded above by $T^\alpha$, and suppose that for some value of $T$, $\MA$ achieves regret at most $\ep T$ against a $\sigma$-smooth adversary for any class of VC dimension at most $d$. Since any class $\MF$ with $|\MF| \leq 1/\sigma$ must have $\vc(\MF) \leq \log 1/\sigma \leq d$, and since $\ep < 1/200$, by Theorem \ref{thm:oracle-lb}, we must have that $T^\alpha \geq \Omega \left( \frac{1/\sqrt{\sigma}}{\log^3 1/\sigma} \right)$. Thus $T \geq \Omega \left( \frac{\sigma^{-1/(2\alpha)}}{\log^3 1/\sigma} \right)$, as desired.

  The second statement of the corollary follows from the above casework by noting that, in Case 1, the computation time is at least the number of rounds $T \geq \Omega(d/\ep^2)$, and in Case 2, we get immediately from Theorem \ref{thm:oracle-lb} that the computation time is $\Omega(\sigma^{-1/2}/\log^3 1/\sigma)$. 
\end{proof}

\subsection{Lower bound on oracle calls for approximate ERM oracle}
One limitation of the lower bounds of Theorem \ref{thm:oracle-lb} and Corollary \ref{cor:stat-comp-gap} is that they only lower bound the total computation time in the ERM oracle model and thus, for instance, do not rule out an algorithm which makes a single ERM oracle call with a large number of points $(x_i, y_i)$. In this section we amend this issue, showing a lower bound on the number of ERM oracle calls that any proper online learning algorithm obtaining sublinear regret must make. To obtain this result, we have to slightly weaken the oracle, namely by working with the \emph{approximate ERM oracle model} (i.e., where we have $\erma > 0$ in Definition \ref{def:oracle}).

First, we need a slight variant of Lemma \ref{lem:nash-oracle-lb}, which establishes a lower bound on the number of oracle calls (which in general is less than computation time), but under the additional assumption that all oracle calls to $\br^1, \br^2$ are made with small-support vectors.
\begin{lemma}
  \label{lem:nash-apx-oracle-lb}
There is a constant $c_0 \in (0,1)$ so that the following holds. Fix $N, S \in \BN$. Any randomized algorithm $\MA$ for approximating the equilibrium of $N \times N$ $\{-1,1\}$-valued zero-sum games with the oracles $\val, \br^1, \br^2$ cannot guarantee with probability greater than $2/3$ that $\MA$'s output value is at most $1/4$ from the game's true value with fewer than $\frac{1}{S} \cdot c_0 \cdot \sqrt{N}/ \log^3 N$ oracle calls, assuming that each oracle call to $\br^1, \br^2$ is made on a vector of support at most $S$.  
\end{lemma}
\begin{proof}
  The proof exactly mirrors that of Lemma \ref{lem:nash-oracle-lb}, with the exception of the analysis of how the oracles $\br^1, \br^2$ are simulated using oracle calls to the min-max consistent function $\phi : [N] \to \BZ$. In particular, for any $q \in \BR^N$, $\br^1(q)$ and $\br^2(p)$ may each be simulated using at most $S \cdot (\log(N)+ 1)$ oracle calls to $\phi$ assuming that $q, p$ have at most $S$ nonzero values. 

  Thus, if $\gamma \leq \frac{1}{S} c_0 \sqrt{N} /\log^3 N$ denotes the total number of oracle calls to $\val, \br^1, \br^2$, then the total number of oracle calls to $\phi$ is at most
  \begin{align}
S \cdot (\log(N)+1) \cdot \gamma \leq S \cdot (\log(N) +1) \cdot \frac{1}{S} \cdot c_0 \sqrt{N} /\log^3 N < c' \cdot \sqrt{N} / \log^2(N),\nonumber
  \end{align}
  where $c'$ is the constant of Corollary \ref{cor:aa} (as long as $c_0$ is chosen sufficiently small). This gives the desired contradiction to Corollary \ref{cor:aa}. 
\end{proof}

Given Lemma \ref{lem:nash-apx-oracle-lb} we may prove in a manner analogously to Theorem \ref{thm:oracle-lb} a lower bound on the number of oracle calls for any no-regret algorithm in the approximate ERM oracle model:
\begin{theorem}
  \label{thm:oracle-lb-apx}
  For some constant $c > 0$ we have the following: fix any $T \in \BN$, $\sigma,\erma \in (0,1]$. In the $\erma$-approximate ERM oracle model, any randomized algorithm cannot guarantee expected regret smaller than $\frac{T}{200}$ against a $\sigma$-smooth online adversary and any $\MF$ with $|\MF| \leq 1/\sigma$ over $T$ rounds in using fewer than $c\erma^2  \cdot \frac{1/\sqrt\sigma}{\log^4 1/\sigma}$ oracle calls; further, this result holds even for binary-valued classes.
\end{theorem}
\begin{proof}
  We use the notation from the proof of Theorem \ref{thm:oracle-lb}. 
  The proof exactly follows that of Theorem \ref{thm:oracle-lb}, except for how the ERM oracle calls are simulated. To describe this difference, recall the definition of $N = 1/\sigma$ to denote the size of the given game $G$, set $\delta = \frac{1}{100 N^2}$, and write $S := \frac{C \log 1/\delta}{\erma^2}$. Then the $\erma$-approximate ERM oracles are simulated as follows:
  \begin{itemize}
  \item To make an ERM oracle call of the form
    \begin{align}
      \label{eq:f1-oracle-ex}
      \argmin_{f_1 \in [N]} \sum_{i=1}^m w_i \cdot \ell_i(G_{f_1,f_{2,i}}, y_i),
    \end{align}
    we do the following:
    \begin{enumerate}
    \item Draw $S$ i.i.d.~samples $i_1, \ldots, i_S$ from the distribution over $[m]$ whose mass at $i$ is proporitional to $|w_i|$.
    \item Use the procedure as in the proof of Theorem \ref{thm:oracle-lb} to make the ERM oracle call
      \begin{align}\label{eq:erm-sampled}
\argmin_{f_1 \in [N]} \sum_{j=1}^S \mathrm{sign}(w_{i_j}) \cdot \ell_{i_j}(G_{f_1, f_{2,i_j}}, y_{i_j}).\end{align}
      Notice that this will lead to an oracle call $\br^1(q)$ for some distribution $q$ of support size at most $S$.
    \end{enumerate}
  \item We perform the same sampling procedure for an ERM oracle call of the form $\argmin_{f_2 \in [N]} \sum_{i=1}^m w_i \cdot \ell_i(G_{f_{1,i}, f_2}, y_i)$, which leads to an oracle call $\br^2(p)$ for some distribution $p$ of support size at most $S$. 
  \end{itemize}
  We claim that each such oracle call of the above form, with probability at least $1-N \cdot \delta$, satisfies the requirement of $\erma$-approximate ERM oracle. To establish this, we simply note that by the Chernoff bound and union bound, with probability $1-N \cdot \delta$, we have, for each oracle call of the form \eqref{eq:f1-oracle-ex}, for $W := \sum_{i=1}^m |w_i|$,
  \begin{align}
     & \sup_{f_1 \in [N]} \left| \sum_{i=1}^m \frac{w_i}{W} \cdot \ell_i(G_{f_1,f_{2,i}}, y_i) - \frac 1S \sum_{j=1}^S  \mathrm{sign}(w_{i_j}) \cdot \ell_{i_j}(G_{f_1,f_{2,i_j}}, y_{i_j}) \right| \leq \erma,\nonumber
  \end{align}
  which implies that the result of \eqref{eq:erm-sampled} returns some $\hat f_1$ which is ${\erma}{W}$ within the minimum of \eqref{eq:f1-oracle-ex}. A similar argument applies to the ERM oracle calls taking a minimum over $f_2 \in [N]$.

  Since the total number of oracle calls of each of the algorithms $\MA_1, \MA_2$ in the proof of Theorem \ref{thm:oracle-lb} is at most $c \cdot \frac{1/\sigma}{\log^41/\sigma} \leq N$, we have that with probability $1-2N^2 \delta$, all oracle calls simulated as above are actually $\erma$-approximate ERM oracle calls. By the assumption of the theorem statement, it follows that with probability at least $2/3$ we can approximate the value of $G$ up to accuracy of $1/4$. Further, the number of oracle calls made to $\val, \br^1, \br^2$ is at most
  \begin{align}
c\erma^2  \cdot \frac{1/\sqrt \sigma}{\log^41/\sigma} \leq \frac{1}{S} \cdot c_0 \frac{\sqrt{N}}{\log^3 N},
  \end{align}
  where $c_0$ is the constant of Lemma \ref{lem:nash-apx-oracle-lb} (as long as $c$ is sufficiently small), and each to $\br^1, \br^2$ is with a vector that has support size at most $S$. But this contradicts the statement of Lemma \ref{lem:nash-apx-oracle-lb}, completing the proof. 
\end{proof}

Finally, as a corollary of Theorem \ref{thm:oracle-lb-apx}, we have the following analogue of Corollary \ref{cor:stat-comp-gap}, which shows a regret lower bound for any algorithm which makes polynomially many oracle queries to an oracle whose accuracy is an inverse polynomial. Notice that the upper bound of Theorem \ref{thm:generalftpl} obtains a regret bound under a $\erma$-approximate oracle that is the same as that under an exact oracle (up to a constant factor), as long as $\erma < o \left( \frac{1}{T^2 \log T} \right)$; thus the assumption of $1/T^\alpha$-approximate oracle (for $\alpha$ constant) in the below corollary seems very reasonable. 
\begin{corollary}
  \label{cor:sc-gap-apx}
  Fix any $\alpha \geq 1, \ep < 1/200, \sigma \in (0,1]$, and $d \geq \log 1/\sigma$. Any algorithm making at most $T^\alpha$ oracle calls over $T$ rounds to a $1/T^\alpha$-approximate oracle requires that $T \geq \widetilde{\Omega}\left( \max \left\{ \frac{d}{\ep^2}, \sigma^{- \frac{1}{6\alpha}} \right\} \right)$ to achieve regret $\ep T$ for classes $\MF$ of VC dimension $d$ against a $\sigma$-smooth adversary.
\end{corollary}
\begin{proof}
  The proof is identical to that of Corollary \ref{cor:stat-comp-gap} for $\frac{d}{\ep^2} > \frac{\sigma^{-1/(6\alpha)}}{\log^4 1/\sigma}$. 

  For $\frac{d}{\ep^2} \leq \frac{\sigma^{-1/(6\alpha)}}{\log^4 1/\sigma}$, we note that any algorithm making $T^\alpha$ oracle calls to a $1/T^\alpha$-approximate ERM oracle over $T$ rounds, which achieves regret at most $\ep T$ against a $\sigma$-smooth adversary must, by Theorem \ref{thm:oracle-lb-apx}, have $T^\alpha \geq \Omega \left( T^{-2\alpha} \cdot \frac{1/\sqrt \sigma}{\log^4 1/\sigma} \right)$, i.e., $T \geq \Omega \left( \frac{\sigma^{-\frac{1}{6\alpha}}}{\log^4 1/\sigma} \right)$. 
\end{proof}


\section{Proof of Theorem \ref{thm:cbregret}}\label{app:cbregret}
We first prove a basic lemma about how smoothness behaves with product distributions.
\begin{lemma}\label{lem:smoothproduct}
    Suppose that $p$ is $\sigma$-smooth with respect to $\mu$ on $\cX$ and for any $x \in \cX$, $p_x' = p'(\cdot | x)$ is $\sigma'$-smooth with respect to $\mu'$ on $\cX'$.  Then $q(x, a) = p(x) p'(a | x)$ is $\sigma \sigma'$-smooth with respect to $\mu \otimes \mu'$.
\end{lemma}
\begin{proof}
    Let $A \subset \cX$ and $A' \subset\cX'$ be measurable.  Then
    \begin{align}
        q(A \times A') = \ee_{x \sim p}\left[p'(A' | x)\chi_{x \in A}\right] \leq \left(\frac{1}{\sigma'} \mu'(A')\right)p(A) \leq \frac{1}{\sigma \sigma'} \mu(A) \otimes \mu(A')
    \end{align}
    The result follows.
\end{proof}
Note that any distribution on $[K]$ is $\frac 1K$-smooth with respect to $\text{Unif}([K])$.  Thus, by Lemma \ref{lem:smoothproduct}, independent of how $a_t$ is chosen, we may assume that $(x_t, a_t)$ is sampled from a distribution that is $\frac{\sigma}{K}$-smooth with respect to $\mu \otimes \text{Unif}([K])$.  Define the random quantity
\begin{equation}
    \reg_{Sq}(T) = \sum_{t = 1}^T (\yhat_t - \ell_t(a_t))^2 - \inf_{f \in \F} \sum_{t = 1}^T (f(x_t, a_t) - \ell_t(a_t))^2
\end{equation}
Then, by \cite[Theorem 1]{foster2020beyond}, with probability at least $1 - \delta$ over the randomization over actions, if we run \textsf{SquareCB} with parameter $\gamma$, we have
\begin{equation}
    \reg_{CB}(T) \leq \frac{\gamma}{2} \reg_{Sq}(T) + 4 \gamma \log\left(\frac 2\delta\right) + \frac{2 KT}{\gamma} + \sqrt{2T \log\left(\frac{2}{\delta}\right)}
\end{equation}
Setting $\delta = \frac 1T$ and noting that the regret is always at most $T$, we have
\begin{equation}\label{eq:conbanditregret}
    \ee\left[\reg_{CB}(T)\right] \leq \frac{\gamma}{2} \ee\left[\reg_{Sq}(T)\right] + 4 \gamma \log (2 T) + \frac{2 KT}{\gamma} + \sqrt{2 T \log \left(2 T\right)} + 1
\end{equation}
If we set $\yhat_t$ to be the prediction given by the relaxation-based algorithm from \eqref{eq:fastrelaxation}, setting $k = \frac{3K}{\sigma} \log T$ then we know that
\begin{equation}
    \ee\left[\reg_{Sq}(T)\right] \leq \frac{7 L K \log T}{\sigma}  \R_T(\F)
\end{equation}
can be achieved with $O\left(T^{\frac 32} \log T \right)$ calls to the ERM oracle.  Letting
\begin{equation}
    \gamma = 12 \log (T) \sqrt{\frac{T \sigma }{L \R_T(\F)}}
\end{equation}
concludes the proof after noting that we may take $L = 2$ for the square loss in the range $[0,1]$.

If we instead use the FTPL algorithm of Theorem \ref{thm:generalftpl}, then $\vc(\F, \alpha) \lesssim \alpha^{- p}$ implies
\begin{equation}
    \ee\left[\reg_{Sq}(T)\right] \leq \widetilde{O}\left( \left(\frac{T\sqrt{K}}{\sqrt{\sigma}}\right)^{\max\left(1 - \frac{1}{3(p - 1)}, \frac{2}{3}\right)} \right)
\end{equation}
Plugging into \eqref{eq:conbanditregret} and minimizing over $\gamma$ yields a regret of
\begin{equation}
    \ee\left[\reg_{CB}(T)\right] \leq \widetilde{O}\left(T^{\max \left(1 - \frac{1}{6(p-1)}, \frac 56\right)} K^{\max\left(\frac 34 - \frac{1}{12 (p - 1)}  , \frac 23\right)} \sigma^{- \frac 14} \right)
\end{equation}

Note that for any $p < \infty$, this is $o(T)$ and so the result holds.


\end{document}